%% file: main.tex
\documentclass[a4paper,fleqn]{cas-sc}

\usepackage[numbers,sort&compress]{natbib}

\usepackage[utf8]{inputenc}
\usepackage[T1]{fontenc}
\usepackage{hyperref}
\usepackage{url}
\usepackage{booktabs}
\usepackage{amsfonts}
\usepackage{amssymb,amsmath,amsthm,mathtools}
\usepackage{graphicx}
\usepackage{subcaption}

\usepackage{algorithm}
\usepackage{algorithmicx}
\usepackage{algpseudocode}
\usepackage{float}
\usepackage{needspace}
\newtheorem{proposition}{Proposition}
\newtheorem{theorem}[proposition]{Theorem}
\usepackage{multirow}
\usepackage{array}
\usepackage{tabularx}
\usepackage{xcolor}
\usepackage{tikz}
\usetikzlibrary{decorations.pathreplacing}
\usepackage{pgfplots}
\pgfplotsset{compat=1.18}
\usepackage{cleveref}

\def\A{\mathcal{A}}
\def\V{\mathcal{V}}
\def\G{\mathcal{G}}
\def\Gtilde{\tilde{\mathcal{G}}}
\def\ftilde{\tilde{f}}
\def\ba{\mathbf{a}}
\def\bv{\mathbf{v}}
\def\bb{\mathbf{b}}
\def\bc{\mathbf{c}}
\def\bu{\mathbf{u}}
\def\bA{\mathbf{A}}
\def\bP{\mathbf{P}}
\def\bPhi{\mathbf{\Phi}}

\algnewcommand\algorithmiclet{\textbf{Let:}}
\algnewcommand\Let{\item[\algorithmiclet]}
\makeatletter
\algnewcommand{\AlgLabel}[1]{%
  \item[\textbf{#1:}]%
}
\makeatother

\input{plot_sizes}
\usepackage{placeins}

\makeatletter
\renewcommand\subsection{\@startsection{subsection}{2}{\z@}
  {10pt \@plus 2pt \@minus 1pt}{3pt}{\ssectionfont\raggedright}}
\renewcommand\subsubsection{\@startsection{subsubsection}{3}{\z@}
  {8pt \@plus 1pt \@minus 1pt}{3pt}{\sssectionfont\raggedright}}
\makeatother
\usepackage{casfix}  
\makeatletter\@secpenalty=0\makeatother  
\begin{document}
\shorttitle{Local and Sparse Kernel Operator Learning}
\shortauthors{Sharma, Lowery, Owhadi, and Shankar}

\title[mode = title, align = \raggedright]{Local and Sparse Kernel Operator Learning for Incompressible Flows: An Operator Lagrange Formulation}

\author[1]{Ramansh Sharma}
\ead{ramansh@cs.utah.edu}
\author[1]{Matthew Lowery}
\ead{mlowery@cs.utah.edu}
\author[2]{Houman Owhadi}
\ead{owhadi@caltech.edu}
\author[1]{Varun Shankar}
\cormark[1]
\ead{shankar@cs.utah.edu}

\cortext[1]{Corresponding author}

\affiliation[1]{organization={Kahlert School of Computing, University of Utah},state={UT}, country={USA}}
\affiliation[2]{organization={Department of Computing and Mathematical Sciences, California Institute of Technology},state={CA}, country={USA}}

\input{sections/abstract}

\begin{keywords}
operator learning \sep incompressible flow \sep divergence-free kernels \sep compactly supported kernels \sep RBF-FD \sep sparse Cholesky \sep GPU computing

\MSC[2020] 65D12 \sep 65D15 \sep 46E22 \sep 65F50 \sep 76D05
\end{keywords}

\maketitle
\hypersetup{pdfauthor={Ramansh Sharma, Matthew Lowery, Houman Owhadi, Varun Shankar},pdfsubject={Property-preserving operator learning for incompressible flows},pdfkeywords={operator learning, incompressible flow, kernel methods}}

\input{sections/introduction}
\input{sections/methods}
\input{sections/operator}
\input{sections/lagrange}
\input{sections/results}
\input{sections/conclusion}

\section*{Credit authorship contribution statement}
{\noindent
\textbf{Ramansh Sharma:} conceptualization, methodology, software, investigation, validation, visualization, writing -- original draft, review, editing.\\
\textbf{Matthew Lowery:} methodology, software\\
\textbf{Houman Owhadi:} conceptualization, methodology, supervision.\\
\textbf{Varun Shankar:} conceptualization, funding acquisition, methodology, project administration, resources, software, supervision, validation, visualization, writing.
}

\section*{Declaration of competing interest}
The authors declare that they have no known competing financial interests or personal relationships that could have appeared to influence the work reported in this paper.

\section*{Data availability}
Our implementation, including the scripts that build the tables and figures of this paper, is available in KernelPack-JAX at \url{https://github.com/VarShankar/kernelpack-jax}. The ten benchmark data sets are those of~\citet{sharma2026fluids} and are publicly available on Zenodo~\citep{sharmashankar2026zenodo}.

\input{sections/ack}

\section*{Declaration of generative AI and AI-assisted technologies in the writing process}
During the preparation of this work, the authors used both Codex (OpenAI) and Claude (Anthropic) to help clean up the manuscript and implement parts of the software. After using this tool, the authors reviewed and edited the content as needed and take full responsibility for the content of the publication.
\input{sections/appendix}
\input{sections/gpu_results_appendix}

\FloatBarrier
\bibliographystyle{cas-model2-names}
\bibliography{references}

\end{document}

%% file: plot_sizes.tex
\newcommand{\paperplot}[1]{\includegraphics[width=\csname plotwidth@#1\endcsname pt]{\csname plotfile@#1\endcsname}}
\expandafter\def\csname plotwidth@plots/backward_facing_step/backward_facing_step__lag_0__full__operator_train_time.png\endcsname{216.5289}
\expandafter\def\csname plotwidth@plots/backward_facing_step/backward_facing_step__lag_1__full__spatial_train_time.png\endcsname{200.1668}
\expandafter\def\csname plotwidth@plots/buoyancy_cavity_flow/buoyancy_cavity_flow__lag_0__full__operator_train_time.png\endcsname{214.7374}
\expandafter\def\csname plotwidth@plots/buoyancy_cavity_flow/buoyancy_cavity_flow__lag_1__full__spatial_train_time.png\endcsname{209.1242}
\expandafter\def\csname plotwidth@plots/flow_cylinder_laminar/flow_cylinder_laminar__lag_0__full__operator_train_time.png\endcsname{214.7374}
\expandafter\def\csname plotwidth@plots/flow_cylinder_laminar/flow_cylinder_laminar__lag_1__full__spatial_train_time.png\endcsname{206.7355}
\expandafter\def\csname plotwidth@plots/flow_cylinder_shedding/flow_cylinder_shedding__lag_0__full__operator_train_time.png\endcsname{214.7374}
\expandafter\def\csname plotwidth@plots/flow_cylinder_shedding/flow_cylinder_shedding__lag_1__full__spatial_train_time.png\endcsname{206.7355}
\expandafter\def\csname plotwidth@plots/forced_turb/forced_turb__lag_0__full__operator_train_time.png\endcsname{216.5289}
\expandafter\def\csname plotwidth@plots/forced_turb/forced_turb__lag_1__full__spatial_train_time.png\endcsname{209.1242}
\expandafter\def\csname plotwidth@plots/lid_cavity_flow/lid_cavity_flow__lag_0__full__operator_train_time.png\endcsname{214.7374}
\expandafter\def\csname plotwidth@plots/lid_cavity_flow/lid_cavity_flow__lag_1__full__spatial_train_time.png\endcsname{206.7355}
\expandafter\def\csname plotwidth@plots/merge_vortices_easier/merge_vortices_easier__lag_0__full__operator_train_time.png\endcsname{216.5289}
\expandafter\def\csname plotwidth@plots/merge_vortices_easier/merge_vortices_easier__lag_1__full__spatial_train_time.png\endcsname{209.1242}
\expandafter\def\csname plotwidth@plots/species_transport/species_transport__lag_0__full__operator_train_time.png\endcsname{214.7374}
\expandafter\def\csname plotwidth@plots/species_transport/species_transport__lag_1__full__operator_train_time.png\endcsname{205.9046}
\expandafter\def\csname plotwidth@plots/species_transport/species_transport__lag_1__full__spatial_train_time.png\endcsname{209.1242}
\expandafter\def\csname plotwidth@plots/species_transport/species_transport__lag_1__hnsw__operator_train_time.png\endcsname{205.9046}
\expandafter\def\csname plotwidth@plots/taylor_green/taylor_green__lag_0__full__operator_train_time.png\endcsname{214.7374}
\expandafter\def\csname plotwidth@plots/taylor_green_time/taylor_green_time__lag_0__full__operator_train_time.png\endcsname{214.7374}
\expandafter\def\csname plotwidth@plots/validation_accuracy/backward_facing_step__speed_error.png\endcsname{209.0997}
\expandafter\def\csname plotwidth@plots/validation_accuracy/backward_facing_step__vector_error.png\endcsname{239.8648648648649}
\expandafter\def\csname plotwidth@plots/validation_accuracy/buoyancy_cavity_flow__speed_error.png\endcsname{217.5227}
\expandafter\def\csname plotwidth@plots/validation_accuracy/buoyancy_cavity_flow__vector_error.png\endcsname{239.8648648648649}
\expandafter\def\csname plotwidth@plots/validation_accuracy/flow_cylinder_laminar__speed_error.png\endcsname{209.0997}
\expandafter\def\csname plotwidth@plots/validation_accuracy/flow_cylinder_laminar__vector_error.png\endcsname{239.8648648648649}
\expandafter\def\csname plotwidth@plots/validation_accuracy/flow_cylinder_shedding__speed_error.png\endcsname{217.5227}
\expandafter\def\csname plotwidth@plots/validation_accuracy/flow_cylinder_shedding__vector_error.png\endcsname{239.8648648648649}
\expandafter\def\csname plotwidth@plots/validation_accuracy/forced_turb__speed_error.png\endcsname{209.0997}
\expandafter\def\csname plotwidth@plots/validation_accuracy/forced_turb__vector_error.png\endcsname{262.6126126126126}
\expandafter\def\csname plotwidth@plots/validation_accuracy/lid_cavity_flow__speed_error.png\endcsname{209.0997}
\expandafter\def\csname plotwidth@plots/validation_accuracy/lid_cavity_flow__vector_error.png\endcsname{239.8648648648649}
\expandafter\def\csname plotwidth@plots/validation_accuracy/merge_vortices_easier__speed_error.png\endcsname{209.0997}
\expandafter\def\csname plotwidth@plots/validation_accuracy/merge_vortices_easier__vector_error.png\endcsname{239.8648648648649}
\expandafter\def\csname plotwidth@plots/validation_accuracy/species_transport__speed_error.png\endcsname{209.0997}
\expandafter\def\csname plotwidth@plots/validation_accuracy/species_transport__vector_error.png\endcsname{239.8648648648649}
\expandafter\def\csname plotwidth@plots/validation_accuracy/taylor_green_time__speed_error.png\endcsname{209.0997}
\expandafter\def\csname plotwidth@plots/validation_accuracy/taylor_green_time__vector_error.png\endcsname{237.1621621621622}
\expandafter\def\csname plotwidth@plots/validation_accuracy/taylor_green__speed_error.png\endcsname{209.0997}
\expandafter\def\csname plotwidth@plots/validation_accuracy/taylor_green__vector_error.png\endcsname{239.8648648648649}
\expandafter\def\csname plotwidth@plots/plots_power_law/backward_facing_step/backward_facing_step__lag_0__full__test_inference_time.png\endcsname{229.2531}
\expandafter\def\csname plotwidth@plots/plots_power_law/backward_facing_step/backward_facing_step__lag_1__full__test_inference_time.png\endcsname{229.2531}
\expandafter\def\csname plotwidth@plots/plots_power_law/buoyancy_cavity_flow/buoyancy_cavity_flow__lag_0__full__test_inference_time.png\endcsname{227.4690}
\expandafter\def\csname plotwidth@plots/plots_power_law/buoyancy_cavity_flow/buoyancy_cavity_flow__lag_1__full__test_inference_time.png\endcsname{227.4690}
\expandafter\def\csname plotwidth@plots/plots_power_law/flow_cylinder_laminar/flow_cylinder_laminar__lag_0__full__test_inference_time.png\endcsname{229.2531}
\expandafter\def\csname plotwidth@plots/plots_power_law/flow_cylinder_laminar/flow_cylinder_laminar__lag_1__full__test_inference_time.png\endcsname{229.2531}
\expandafter\def\csname plotwidth@plots/plots_power_law/flow_cylinder_shedding/flow_cylinder_shedding__lag_0__full__test_inference_time.png\endcsname{229.2531}
\expandafter\def\csname plotwidth@plots/plots_power_law/flow_cylinder_shedding/flow_cylinder_shedding__lag_1__full__test_inference_time.png\endcsname{229.2531}
\expandafter\def\csname plotwidth@plots/plots_power_law/forced_turb/forced_turb__lag_0__full__test_inference_time.png\endcsname{227.4690}
\expandafter\def\csname plotwidth@plots/plots_power_law/forced_turb/forced_turb__lag_1__full__test_inference_time.png\endcsname{227.4690}
\expandafter\def\csname plotwidth@plots/plots_power_law/lid_cavity_flow/lid_cavity_flow__lag_0__full__test_inference_time.png\endcsname{229.2531}
\expandafter\def\csname plotwidth@plots/plots_power_law/lid_cavity_flow/lid_cavity_flow__lag_1__full__test_inference_time.png\endcsname{229.2531}
\expandafter\def\csname plotwidth@plots/plots_power_law/merge_vortices_easier/merge_vortices_easier__lag_0__full__test_inference_time.png\endcsname{229.2531}
\expandafter\def\csname plotwidth@plots/plots_power_law/merge_vortices_easier/merge_vortices_easier__lag_1__full__test_inference_time.png\endcsname{229.2531}
\expandafter\def\csname plotwidth@plots/plots_power_law/species_transport/species_transport__lag_0__full__test_inference_time.png\endcsname{227.4690}
\expandafter\def\csname plotwidth@plots/plots_power_law/species_transport/species_transport__lag_1__full__test_inference_time.png\endcsname{227.4690}
\expandafter\def\csname plotwidth@plots/plots_power_law/taylor_green/taylor_green__lag_0__full__test_inference_time.png\endcsname{229.2531}
\expandafter\def\csname plotwidth@plots/plots_power_law/taylor_green/taylor_green__lag_1__full__test_inference_time.png\endcsname{229.2531}
\expandafter\def\csname plotwidth@plots/plots_power_law/taylor_green_time/taylor_green_time__lag_0__full__test_inference_time.png\endcsname{229.2531}
\expandafter\def\csname plotwidth@plots/plots_power_law/taylor_green_time/taylor_green_time__lag_1__full__test_inference_time.png\endcsname{229.2531}
\expandafter\def\csname plotfile@plots/flow_cylinder_laminar/flow_cylinder_laminar__lag_1__full__spatial_train_time.png\endcsname{paper_plots/flow_cylinder_laminar__lag_1__full__spatial_train_time.png}
\expandafter\def\csname plotfile@plots/plots_power_law/lid_cavity_flow/lid_cavity_flow__lag_0__full__test_inference_time.png\endcsname{paper_plots/lid_cavity_flow__lag_0__full__test_inference_time.png}
\expandafter\def\csname plotfile@plots/plots_power_law/species_transport/species_transport__lag_1__full__test_inference_time.png\endcsname{paper_plots/species_transport__lag_1__full__test_inference_time.png}
\expandafter\def\csname plotfile@plots/plots_power_law/forced_turb/forced_turb__lag_1__full__test_inference_time.png\endcsname{paper_plots/forced_turb__lag_1__full__test_inference_time.png}
\expandafter\def\csname plotfile@plots/flow_cylinder_laminar/flow_cylinder_laminar__lag_0__full__operator_train_time.png\endcsname{paper_plots/flow_cylinder_laminar__lag_0__full__operator_train_time.png}
\expandafter\def\csname plotfile@plots/validation_accuracy/merge_vortices_easier__vector_error.png\endcsname{paper_plots/merge_vortices_easier__vector_error.png}
\expandafter\def\csname plotfile@plots/plots_power_law/flow_cylinder_shedding/flow_cylinder_shedding__lag_1__full__test_inference_time.png\endcsname{paper_plots/flow_cylinder_shedding__lag_1__full__test_inference_time.png}
\expandafter\def\csname plotfile@plots/taylor_green_time/taylor_green_time__lag_0__full__operator_train_time.png\endcsname{paper_plots/taylor_green_time__lag_0__full__operator_train_time.png}
\expandafter\def\csname plotfile@plots/lid_cavity_flow/lid_cavity_flow__lag_1__full__spatial_train_time.png\endcsname{paper_plots/lid_cavity_flow__lag_1__full__spatial_train_time.png}
\expandafter\def\csname plotfile@plots/validation_accuracy/flow_cylinder_shedding__vector_error.png\endcsname{paper_plots/flow_cylinder_shedding__vector_error.png}
\expandafter\def\csname plotfile@plots/plots_power_law/merge_vortices_easier/merge_vortices_easier__lag_0__full__test_inference_time.png\endcsname{paper_plots/merge_vortices_easier__lag_0__full__test_inference_time.png}
\expandafter\def\csname plotfile@plots/validation_accuracy/backward_facing_step__vector_error.png\endcsname{paper_plots/backward_facing_step__vector_error.png}
\expandafter\def\csname plotfile@plots/buoyancy_cavity_flow/buoyancy_cavity_flow__lag_0__full__operator_train_time.png\endcsname{paper_plots/buoyancy_cavity_flow__lag_0__full__operator_train_time.png}
\expandafter\def\csname plotfile@plots/backward_facing_step/backward_facing_step__lag_0__full__operator_train_time.png\endcsname{paper_plots/backward_facing_step__lag_0__full__operator_train_time.png}
\expandafter\def\csname plotfile@plots/validation_accuracy/buoyancy_cavity_flow__vector_error.png\endcsname{paper_plots/buoyancy_cavity_flow__vector_error.png}
\expandafter\def\csname plotfile@plots/flow_cylinder_shedding/flow_cylinder_shedding__lag_1__full__spatial_train_time.png\endcsname{paper_plots/flow_cylinder_shedding__lag_1__full__spatial_train_time.png}
\expandafter\def\csname plotfile@plots/plots_power_law/backward_facing_step/backward_facing_step__lag_0__full__test_inference_time.png\endcsname{paper_plots/backward_facing_step__lag_0__full__test_inference_time.png}
\expandafter\def\csname plotfile@plots/forced_turb/forced_turb__lag_1__full__spatial_train_time.png\endcsname{paper_plots/forced_turb__lag_1__full__spatial_train_time.png}
\expandafter\def\csname plotfile@plots/forced_turb/forced_turb__lag_0__full__operator_train_time.png\endcsname{paper_plots/forced_turb__lag_0__full__operator_train_time.png}
\expandafter\def\csname plotfile@plots/plots_power_law/flow_cylinder_laminar/flow_cylinder_laminar__lag_0__full__test_inference_time.png\endcsname{paper_plots/flow_cylinder_laminar__lag_0__full__test_inference_time.png}
\expandafter\def\csname plotfile@plots/taylor_green/taylor_green__lag_0__full__operator_train_time.png\endcsname{paper_plots/taylor_green__lag_0__full__operator_train_time.png}
\expandafter\def\csname plotfile@plots/validation_accuracy/taylor_green__vector_error.png\endcsname{paper_plots/taylor_green__vector_error.png}
\expandafter\def\csname plotfile@plots/validation_accuracy/flow_cylinder_laminar__vector_error.png\endcsname{paper_plots/flow_cylinder_laminar__vector_error.png}
\expandafter\def\csname plotfile@plots/species_transport/species_transport__lag_1__full__spatial_train_time.png\endcsname{paper_plots/species_transport__lag_1__full__spatial_train_time.png}
\expandafter\def\csname plotfile@plots/flow_cylinder_shedding/flow_cylinder_shedding__lag_0__full__operator_train_time.png\endcsname{paper_plots/flow_cylinder_shedding__lag_0__full__operator_train_time.png}
\expandafter\def\csname plotfile@plots/plots_power_law/buoyancy_cavity_flow/buoyancy_cavity_flow__lag_0__full__test_inference_time.png\endcsname{paper_plots/buoyancy_cavity_flow__lag_0__full__test_inference_time.png}
\expandafter\def\csname plotfile@plots/species_transport/species_transport__lag_0__full__operator_train_time.png\endcsname{paper_plots/species_transport__lag_0__full__operator_train_time.png}
\expandafter\def\csname plotfile@plots/plots_power_law/taylor_green_time/taylor_green_time__lag_1__full__test_inference_time.png\endcsname{paper_plots/taylor_green_time__lag_1__full__test_inference_time.png}
\expandafter\def\csname plotfile@plots/plots_power_law/backward_facing_step/backward_facing_step__lag_1__full__test_inference_time.png\endcsname{paper_plots/backward_facing_step__lag_1__full__test_inference_time.png}
\expandafter\def\csname plotfile@plots/buoyancy_cavity_flow/buoyancy_cavity_flow__lag_1__full__spatial_train_time.png\endcsname{paper_plots/buoyancy_cavity_flow__lag_1__full__spatial_train_time.png}
\expandafter\def\csname plotfile@plots/plots_power_law/taylor_green_time/taylor_green_time__lag_0__full__test_inference_time.png\endcsname{paper_plots/taylor_green_time__lag_0__full__test_inference_time.png}
\expandafter\def\csname plotfile@plots/plots_power_law/species_transport/species_transport__lag_0__full__test_inference_time.png\endcsname{paper_plots/species_transport__lag_0__full__test_inference_time.png}
\expandafter\def\csname plotfile@plots/validation_accuracy/forced_turb__vector_error.png\endcsname{paper_plots/forced_turb__vector_error.png}
\expandafter\def\csname plotfile@plots/validation_accuracy/taylor_green_time__vector_error.png\endcsname{paper_plots/taylor_green_time__vector_error.png}
\expandafter\def\csname plotfile@plots/plots_power_law/flow_cylinder_laminar/flow_cylinder_laminar__lag_1__full__test_inference_time.png\endcsname{paper_plots/flow_cylinder_laminar__lag_1__full__test_inference_time.png}
\expandafter\def\csname plotfile@plots/validation_accuracy/lid_cavity_flow__vector_error.png\endcsname{paper_plots/lid_cavity_flow__vector_error.png}
\expandafter\def\csname plotfile@plots/backward_facing_step/backward_facing_step__lag_1__full__spatial_train_time.png\endcsname{paper_plots/backward_facing_step__lag_1__full__spatial_train_time.png}
\expandafter\def\csname plotfile@plots/merge_vortices_easier/merge_vortices_easier__lag_1__full__spatial_train_time.png\endcsname{paper_plots/merge_vortices_easier__lag_1__full__spatial_train_time.png}
\expandafter\def\csname plotfile@plots/plots_power_law/merge_vortices_easier/merge_vortices_easier__lag_1__full__test_inference_time.png\endcsname{paper_plots/merge_vortices_easier__lag_1__full__test_inference_time.png}
\expandafter\def\csname plotfile@plots/plots_power_law/lid_cavity_flow/lid_cavity_flow__lag_1__full__test_inference_time.png\endcsname{paper_plots/lid_cavity_flow__lag_1__full__test_inference_time.png}
\expandafter\def\csname plotfile@plots/plots_power_law/buoyancy_cavity_flow/buoyancy_cavity_flow__lag_1__full__test_inference_time.png\endcsname{paper_plots/buoyancy_cavity_flow__lag_1__full__test_inference_time.png}
\expandafter\def\csname plotfile@plots/validation_accuracy/species_transport__vector_error.png\endcsname{paper_plots/species_transport__vector_error.png}
\expandafter\def\csname plotfile@plots/lid_cavity_flow/lid_cavity_flow__lag_0__full__operator_train_time.png\endcsname{paper_plots/lid_cavity_flow__lag_0__full__operator_train_time.png}
\expandafter\def\csname plotfile@plots/plots_power_law/forced_turb/forced_turb__lag_0__full__test_inference_time.png\endcsname{paper_plots/forced_turb__lag_0__full__test_inference_time.png}
\expandafter\def\csname plotfile@plots/plots_power_law/flow_cylinder_shedding/flow_cylinder_shedding__lag_0__full__test_inference_time.png\endcsname{paper_plots/flow_cylinder_shedding__lag_0__full__test_inference_time.png}
\expandafter\def\csname plotfile@plots/plots_power_law/taylor_green/taylor_green__lag_0__full__test_inference_time.png\endcsname{paper_plots/taylor_green__lag_0__full__test_inference_time.png}
\expandafter\def\csname plotfile@plots/merge_vortices_easier/merge_vortices_easier__lag_0__full__operator_train_time.png\endcsname{paper_plots/merge_vortices_easier__lag_0__full__operator_train_time.png}
\expandafter\def\csname plotfile@plots/plots_power_law/taylor_green/taylor_green__lag_1__full__test_inference_time.png\endcsname{paper_plots/taylor_green__lag_1__full__test_inference_time.png}
\expandafter\def\csname plotfile@plots/species_transport/species_transport__lag_1__full__operator_train_time.png\endcsname{plots/species_transport/species_transport__lag_1__full__operator_train_time.png}
\expandafter\def\csname plotfile@plots/species_transport/species_transport__lag_1__hnsw__operator_train_time.png\endcsname{plots/species_transport/species_transport__lag_1__hnsw__operator_train_time.png}
\expandafter\def\csname plotfile@plots/validation_accuracy/backward_facing_step__speed_error.png\endcsname{plots/validation_accuracy/backward_facing_step__speed_error.png}
\expandafter\def\csname plotfile@plots/validation_accuracy/buoyancy_cavity_flow__speed_error.png\endcsname{plots/validation_accuracy/buoyancy_cavity_flow__speed_error.png}
\expandafter\def\csname plotfile@plots/validation_accuracy/flow_cylinder_laminar__speed_error.png\endcsname{plots/validation_accuracy/flow_cylinder_laminar__speed_error.png}
\expandafter\def\csname plotfile@plots/validation_accuracy/flow_cylinder_shedding__speed_error.png\endcsname{plots/validation_accuracy/flow_cylinder_shedding__speed_error.png}
\expandafter\def\csname plotfile@plots/validation_accuracy/forced_turb__speed_error.png\endcsname{plots/validation_accuracy/forced_turb__speed_error.png}
\expandafter\def\csname plotfile@plots/validation_accuracy/lid_cavity_flow__speed_error.png\endcsname{plots/validation_accuracy/lid_cavity_flow__speed_error.png}
\expandafter\def\csname plotfile@plots/validation_accuracy/merge_vortices_easier__speed_error.png\endcsname{plots/validation_accuracy/merge_vortices_easier__speed_error.png}
\expandafter\def\csname plotfile@plots/validation_accuracy/species_transport__speed_error.png\endcsname{plots/validation_accuracy/species_transport__speed_error.png}
\expandafter\def\csname plotfile@plots/validation_accuracy/taylor_green_time__speed_error.png\endcsname{plots/validation_accuracy/taylor_green_time__speed_error.png}
\expandafter\def\csname plotfile@plots/validation_accuracy/taylor_green__speed_error.png\endcsname{plots/validation_accuracy/taylor_green__speed_error.png}

%% file: sections/abstract.tex
\begin{abstract}
We present a local and sparse kernel method for learning the solution operators of incompressible flows. It builds on our earlier property-preserving kernel method, whose global output reconstruction and dense operator regression scale poorly. Our method replaces the global output reconstruction by local divergence-free interpolants with divergence-free polynomials of any degree, solved in parallel on the GPU. These interpolants are divergence free, periodic on periodic domains, and, with a multiscale kernel, approach Kolmogorov statistics under refinement. A compactly supported kernel, positive semidefinite in any input dimension, makes the operator regression sparse. Its support radius, and hence the sparsity, is fitted to data, and optional affine polynomials in the input capture global trends. Operator Lagrange weights join the two. They cost one sparse solve per new input and act on the training fields before the output reconstruction, so their local coefficients are never formed. We show that, unlike direct evaluation, Lagrange evaluation does not amplify solve errors by the cancellation among spatial coefficients. On ten fluid benchmarks, validation selects sparse matrices nearly everywhere, with densities that mostly decrease as the training set grows, and affine augmentation lowers the error by more than an order of magnitude for forced isotropic turbulence. These sparse models are at least as accurate as our earlier dense method and more accurate than neural operators. Lagrange evaluation is as accurate as direct evaluation with local output reconstructions, and up to nearly three orders of magnitude more accurate with global ones. We also present GPU timings and storage.
\end{abstract}

%% file: sections/introduction.tex
\section{Introduction}
\label{sec:introduction}

Uncertainty quantification, design, and inverse problems for incompressible flows require the velocity field for many initial conditions, boundary conditions, or forcing functions, and the repeated solution of the Navier--Stokes equations then dominates the cost~\citep{sudret2017surrogate,queipo2005surrogate,stuart2010inverse,tezzele2020enhancing}. One remedy is to learn the solution operator, the map from the problem data to the velocity, from a collection of simulations~\citep{boulle2024mathematical}. Reduced-order models have long served this purpose~\citep{quarteroni2015reduced,benner2015survey}, and neural operators such as DeepONets, Fourier and graph neural operators, transformer-based models, and kernel neural operators now learn such maps from large data sets~\citep{lu2021deeponet,li2020fourier,li2020neural_operator,li2023fourier,wu2024transolver,lowery2024kernel}. We seek an operator learner for incompressible flows that satisfies four criteria. First, every prediction should be divergence free, and periodic on periodic domains, however accurately the operator itself is learned. Second, the cost of the output representation should grow linearly with the number of output points, so that fine three-dimensional outputs remain affordable. Third, the regression should handle ten thousand or more training functions without factoring a dense matrix. Fourth, every stage should run in parallel on a graphics processing unit (GPU).

Kernel methods offer an attractive alternative to neural operators~\citep{batlle2024kernel,mora2025operator,kadri2016operator}. Batlle et al.~\citep{batlle2024kernel} learn an operator from point samples in two steps. An operator regression, a kernel regression on the space of sampled inputs, maps the samples of the input function to the samples of the output function. An output reconstruction, a kernel interpolant of these samples in space (or spacetime), then recovers the output function. Since both steps are kernel interpolation or regression problems, the learned operator has a closed form, few parameters, and an error analysis. For flows, the output representation is the natural place to impose physics~\citep{richter2022neural,khorrami2024physics}. In our previous work~\citep{sharma2026fluids}, we therefore regressed the coefficients of a divergence-free kernel interpolant of the output instead of its point samples, so that every prediction is exactly divergence free regardless of the accuracy of the operator regression. We called this the property-preserving kernel method (PPKM). PPKM satisfies the first criterion, but both of its kernel problems are global. Its operator regression factors one dense kernel matrix on all $N$ training inputs, at a cost of $O(N^3)$ operations and $O(N^2)$ storage, and its output reconstruction solves one dense system on all output points. Both limits bind for three-dimensional turbulence, where a training set has ten thousand or more functions and an output has thousands of points. Meeting the other criteria poses three tasks. The output reconstruction must become local without losing the properties of the global one, the operator regression must become sparse, and the two must be joined without forming the local coefficients of every training field.

The first task draws on radial basis function-generated finite differences (RBF-FD). RBF-FD methods replace a global radial basis function (RBF) interpolant, whose dense system couples all nodes, by small interpolants on the stencil of each node, so that every linear system has a fixed size and is independent of the others~\citep{tolstykh2003using,shu2003local,wright2006scattered,fornberg2015solving,fornberg2015primer}. Over the last two decades, RBF-FD methods have been stabilized for convection-dominated problems~\citep{fornberg2011stabilization,shankar2018hyperviscosity}, applied to the Navier--Stokes equations~\citep{shu2003local,flyer2016enhancing}, extended to overlapped stencils, surfaces, and moving domains~\citep{shankar2017overlapped,shankar2018rbfloi,shankar2021efficient}, and formulated in the least-squares sense~\citep{tominec2021least}. They also have error bounds~\citep{davydov2016error}, node generators for irregular domains~\citep{shankar2018robust}, and GPU implementations~\citep{bollig2012solution}. In particular, polyharmonic splines (PHS) augmented with polynomials converge at the order of the polynomial degree and have no shape parameter~\citep{flyer2016role,bayona2017role}. For flows, the local interpolants must also be divergence free. Divergence-free kernels, obtained by applying $-\Delta I+\nabla\nabla^{\mathsf T}$ to a scalar kernel~\citep{narcowich1994generalized,dodu2004irrotational,narcowich2007divergence,fuselier2008sobolev,wendland2009divergence}, have this property by construction and have been used both globally and in partition-of-unity form~\citep{fuselier2016high,drake2021partition}, and Yelnyk~\citep{yelnyk2023rbf} combined divergence-free PHS with polynomials in local Hermite interpolation for Stokes problems. However, the high order of PHS-based RBF-FD comes from the polynomials, so a local divergence-free interpolant of high order needs divergence-free polynomials of any degree, in two and three dimensions, that can be evaluated anywhere on the stencil.

Our approach is to build these polynomials by differentiating products of Legendre polynomials, select an independent subset on each stencil by a column-pivoted QR factorization, and combine them with a divergence-free PHS on the stencils of RBF-FD, used for interpolation instead of differentiation. Every local interpolant is then exactly divergence free, and building the stencils from periodic images of the nodes makes the output reconstruction periodic. For turbulence, we replace the scale-free PHS by the multiscale Mat\'ern kernel of~\citet{sharma2026fluids}, whose weights encode the Kolmogorov scaling of the velocity increments, so that the output reconstruction approaches prescribed Kolmogorov statistics as the nodes are refined. Mat\'ern kernels and PHS are Green's functions of closely related operators, and a Mat\'ern kernel tends to a PHS as its length grows, so this change adds length scales without changing the construction. Since the local systems are independent and of fixed size, we assemble, factor, and evaluate them in batches on the GPU, and our output reconstruction is embarrassingly parallel.

For the second task, the dense kernel matrix of the operator regression can be approximated by low-rank, Nystr\"om, and random-feature methods, inducing points, structured kernel interpolation, or sparse approximate factors of its inverse~\citep{drineas2005approximating,gittens2013revisiting,snelson2006sparse,rahimi2007random,yu2016orthogonal,fine2001efficient,bach2005predictive,schafer2004sparse,si2014memory,wilson2015kernel}. In contrast, a compactly supported kernel gives a sparse matrix that can be factored exactly by sparse Cholesky. Its obstacle is the dimension of the input. A sampled input function has thousands of coordinates, but no compactly supported radial kernel is positive definite in every dimension~\citep{Wendland2004}. Our kernel therefore determines its support from a few principal coordinates of the input, in which a compactly supported function is positive definite, and treats the remaining coordinates with a Mat\'ern factor, which is positive definite in every dimension. The product is positive semidefinite in any input dimension. Its support radius sets the sparsity of the kernel matrix, and we choose it from the data instead of fixing it in advance. A restricted likelihood criterion, with a smooth count of the interacting training pairs as an optional penalty, fits the shape of the kernel and the regularization, and the error on validation fields then selects the radius, so that nearly dense matrices remain among the candidates when accuracy requires them.

However, a compactly supported kernel returns to zero away from the training inputs, so it represents trends in the input poorly, although many operators depend nearly linearly on part of their input. We therefore augment the operator regression with affine polynomials in the leading principal coordinates of the input, which it reproduces exactly, as in polynomially augmented kernel interpolation. The resulting saddle-point system is solved by a sparse Cholesky factorization of the kernel matrix followed by a QR factorization of the whitened polynomial columns, which avoids forming the Schur complement~\citep{belianovich2026unified}. JAX has no sparse direct solver that keeps a factorization between calls, so we factor the kernel matrix on the GPU with cuDSS, through a primitive of our KernelPack-JAX framework that exposes cuDSS to compiled JAX programs~\citep{cattaneo2026kernelpack}.

The third task leads to operator Lagrange weights. Since the same kernel regresses every output coefficient, a prediction is a linear combination of the training outputs, with weights that depend only on the new input. These weights cost one sparse solve per new input, regardless of the size of the output, and they give the coefficients on every stencil at once (Figure~\ref{fig:pipeline}). Because the output reconstruction is linear, the weights can also be applied to the sampled training fields before the output reconstruction, so the local coefficients of the training fields, which would take tens of times the storage of the fields themselves, are never formed. Weights and coefficients give the same prediction in exact arithmetic but round differently. We show that direct evaluation through coefficients amplifies solve errors by a factor that measures the cancellation among the spatial coefficients, while Lagrange evaluation does not.

We test our method in two steps. On scattered nodes, our local output reconstruction converges at close to orders $2$, $4$, and $6$, keeps the divergence and the periodic mismatch at rounding level, and approaches prescribed Kolmogorov statistics under refinement at the separations that the nodes resolve. Beyond a few thousand output points, it is also cheaper to set up than a global output reconstruction, and its memory grows only linearly with their number. We then compare kernels with a fixed and a fitted shape, each with and without affine augmentation, on ten fluid benchmarks with up to ten thousand training functions. Validation selects sparse kernel matrices on all benchmarks except forced isotropic turbulence and the unaugmented Taylor--Green models, and on most of them the selected density decreases as the training set grows. With these sparse matrices, the test errors decrease steadily with the number of training functions. At the training sizes of our earlier study~\citep{sharma2026fluids}, our models are at least as accurate as dense PPKM and the kernel method of~\citet{batlle2024kernel} on every data set, and more accurate than the neural operators reported there. Affine augmentation reduces the test error for forced isotropic turbulence by more than an order of magnitude, and together with the fitted shape it does the same for turbulent species transport. In contrast, the five flows whose training inputs vary along a single direction are insensitive to both choices. Finally, Lagrange evaluation is as accurate as direct evaluation with our local output reconstruction, whose amplification factors are at most about $10^4$, and up to nearly three orders of magnitude more accurate with global output reconstructions, whose amplification factors reach $6\times10^9$.

The remainder of this paper is organized as follows. Section~\ref{sec:methods} constructs the local property-preserving output reconstruction, Section~\ref{sec:operator} the sparse operator regression and its affine augmentation, and Section~\ref{sec:lagrange-section} the operator Lagrange weights, their rounding errors, and the GPU implementation. Section~\ref{sec:results} presents the numerical experiments, and Section~\ref{sec:conclusion} our conclusions.

\subsection{Problem formulation}
\label{sec:overview}
\label{sec:sappkm_algorithm}

We consider an incompressible fluid of constant density, governed by
\begin{align}
\label{eq:ins}
\frac{\partial \bu}{\partial t} + (\bu \cdot \nabla) \bu
&= \nu \nabla^2 \bu - \frac{1}{\rho} \nabla p + \frac{1}{\rho} \mathbf{f},
\quad \text{in } \Omega \times (0,T], \\
\nabla \cdot \bu &= 0, \quad \text{in } \Omega \times (0,T], \\
\bu(\mathbf{x},0) &= \bu_0(\mathbf{x}), \quad \text{in } \Omega, \\
\mathcal{B} \bu &= \mathbf{g}, \quad \text{on } \partial \Omega \times (0,T],
\end{align}
where $\bu$ is the velocity, $p$ the pressure, $\nu$ the kinematic viscosity, $\rho$ the density, $\mathbf{f}$ the body force, and $\mathcal{B}$ the boundary operator. Our operators map initial data (a velocity $\bu_0$, a vorticity, or, for a buoyancy-driven flow, a temperature), the boundary data $\mathbf{g}$, or the forcing $\mathbf{f}$ to the velocity. The training data are simulations of~\eqref{eq:ins} from~\citet{sharma2026fluids} (Appendix~\ref{app:benchmark-configurations}).

Let $\A$ and $\V$ be separable Banach spaces of input and output functions, and let $\G: \A \rightarrow \V$ be the operator we wish to learn. An input $a\in\A$ and its output $v\in\V$ are defined on domains $\Omega_a$ and $\Omega_v$, which may be spatial, spacetime, or parametric, with
\begin{align}
a: \Omega_a \rightarrow \mathbb{R}^{d_a},\quad
v: \Omega_v \rightarrow \mathbb{R}^{d}.
\end{align}
Here $d_a$ and $d$ are the numbers of input and output components, and $D_a$ and $D_v$ are the dimensions of the domains. We sample inputs at $X=\{x_1,\ldots,x_n\}\subset\Omega_a\subset\mathbb{R}^{D_a}$, which gives vectors $\ba\in\mathbb{R}^{n d_a}$. For spatial outputs, $D_v=d$, and we sample at $Y=\{y_j\}_{j=1}^{m}\subset\Omega_v\subset\mathbb{R}^d$, which gives vectors $\bv\in\mathbb{R}^{md}$ stacked in component blocks, each holding one velocity component at all points of $Y$, as in~\citet{sharma2026fluids}. For spacetime outputs, $D_v=d+1$, and sampling the $m$ spatial sites at $N_t$ times gives vectors in $\mathbb{R}^{dmN_t}$ (Section~\ref{sec:spacetime}).

Given $N$ training pairs $\{(a_i,v_i)\}_{i=1}^{N}$ with sampled vectors $\{(\ba_i,\bv_i)\}_{i=1}^{N}$, we write the learned operator, as in~\citet{sharma2026fluids}, as
\begin{align}
v \approx \Gtilde(a) = \chi \circ \ftilde \circ \varphi(a),
\label{eq:sappkm_decomposition}
\end{align}
where $\varphi$ samples the input, the operator regression $\ftilde$ maps the samples to the coefficients of the output, and the output reconstruction $\chi$ maps these coefficients to a function in space (or spacetime). This structure follows the formulation of~\citet{batlle2024kernel}, in which $\ftilde$ maps input samples to output samples, with the difference that $\ftilde$ predicts the coefficients of a divergence-free interpolant. In PPKM~\citep{sharma2026fluids}, $\chi$ is a global divergence-free interpolant and $\ftilde$ a dense operator regression. Here, $\chi$ is a collection of local divergence-free interpolants (Section~\ref{sec:methods}), and $\ftilde$ is a sparse operator regression (Section~\ref{sec:operator}). Figure~\ref{fig:pipeline} shows how operator Lagrange weights join the two (Section~\ref{sec:lagrange-section}).
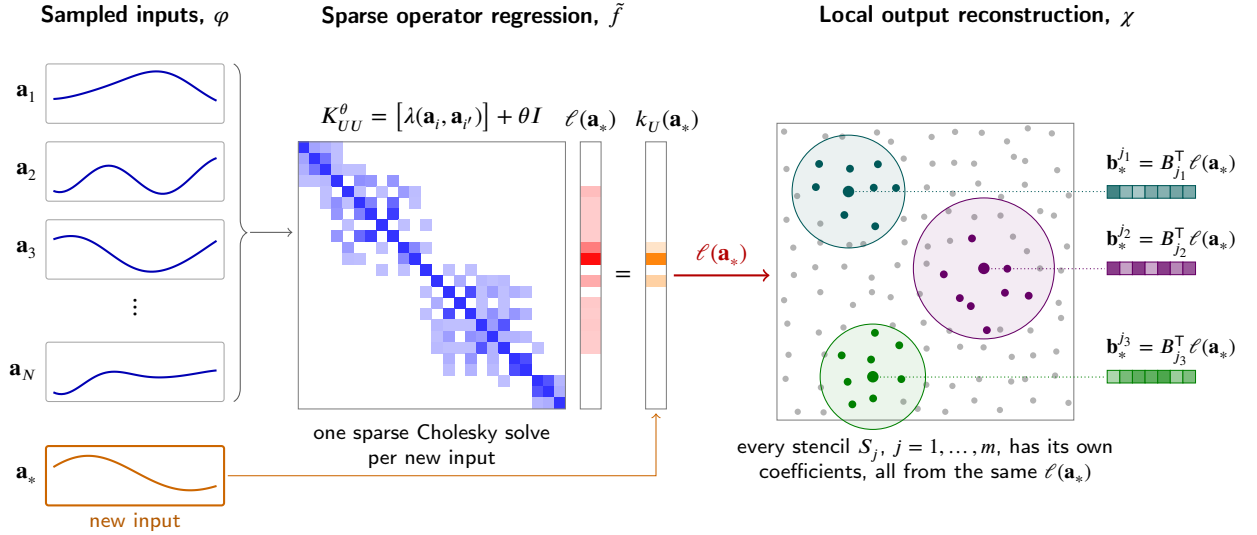
\begin{figure}[tbp]
\centering
\resizebox{\textwidth}{!}{\input{sections/fig_pipeline}}
\caption{Overview of our method for a new input $\ba_*$. The compactly supported kernel $\lambda$ of Section~\ref{sec:operator} gives a sparse matrix $K_{UU}^{\theta}$ on the sampled training inputs (left and center). One sparse solve with the kernel values $k_U(\ba_*)$ gives the operator Lagrange weights $\ell(\ba_*)$ of Section~\ref{sec:lagrange-section}, shown here without their affine correction. The same weights give the coefficients $\bb_*^j=B_j^{\mathsf T}\ell(\ba_*)$ of the local divergence-free interpolant of Section~\ref{sec:methods} on every stencil $S_j$ of the output (right), where the rows of $B_j$ are the coefficients of the training fields on $S_j$. In practice, we apply the weights to the sampled training fields before the output reconstruction, which gives the same prediction (Section~\ref{sec:commuted-prediction}).}
\label{fig:pipeline}
\end{figure}

We use the index $i$ for training functions, $j$ for stencils, $k$ for components and coordinates, $q$ for new (query) inputs and query points, and $t$ for time. Each stencil has $n_s$ nodes. A subscript $*$ marks a quantity of a new input, a superscript $\star$ a quantity at query points, and a hat a computed quantity in the error analysis. The superscript $\mathsf T$ denotes the transpose, $I$ an identity matrix, and $\|\cdot\|_F$ the Frobenius norm.

%% file: sections/fig_pipeline.tex
\begin{tikzpicture}[x=1cm,y=1cm,font=\small]
\node[font=\small\bfseries] at (1.35,6.05) {Sampled inputs, $\varphi$};
\draw[gray!70,rounded corners=1pt] (0.15,4.60) rectangle (2.55,5.40);
\draw[blue!70!black,thick] plot[domain=0:1,samples=60] ({0.25+2.2*\x},{5.00+(0.183)*sin(180*\x+0.00)+(-0.114)*sin(360*\x+40.11)+(0.012)*sin(540*\x+80.21)+(-0.013)*sin(720*\x+120.32)});
\node[left] at (0.15,5.00) {$\ba_1$};
\draw[gray!70,rounded corners=1pt] (0.15,3.55) rectangle (2.55,4.35);
\draw[blue!70!black,thick] plot[domain=0:1,samples=60] ({0.25+2.2*\x},{3.95+(-0.151)*sin(180*\x+0.00)+(-0.036)*sin(360*\x+40.11)+(-0.225)*sin(540*\x+80.21)+(-0.019)*sin(720*\x+120.32)});
\node[left] at (0.15,3.95) {$\ba_2$};
\draw[gray!70,rounded corners=1pt] (0.15,2.50) rectangle (2.55,3.30);
\draw[blue!70!black,thick] plot[domain=0:1,samples=60] ({0.25+2.2*\x},{2.90+(-0.110)*sin(180*\x+0.00)+(0.211)*sin(360*\x+40.11)+(0.010)*sin(540*\x+80.21)+(-0.011)*sin(720*\x+120.32)});
\node[left] at (0.15,2.90) {$\ba_3$};
\draw[gray!70,rounded corners=1pt] (0.15,0.85) rectangle (2.55,1.65);
\draw[blue!70!black,thick] plot[domain=0:1,samples=60] ({0.25+2.2*\x},{1.25+(-0.120)*sin(180*\x+0.00)+(-0.143)*sin(360*\x+40.11)+(-0.150)*sin(540*\x+80.21)+(-0.042)*sin(720*\x+120.32)});
\node[left] at (0.15,1.25) {$\ba_N$};
\node at (1.35,2.12) {$\vdots$};
\draw[decorate,decoration={brace,amplitude=5pt},gray!70!black] (2.68,5.4) -- (2.68,0.85);
\draw[orange!80!black,thick,rounded corners=1pt] (0.15,-0.55) rectangle (2.55,0.25);
\draw[orange!80!black,thick] plot[domain=0:1,samples=60] ({0.25+2.2*\x},{-0.15+0.3*sin(288*\x+22.92)*exp(-0.5*\x)});
\node[left] at (0.15,-0.15) {$\ba_*$};
\node[font=\footnotesize,orange!70!black] at (1.35,-0.77) {new input};
\node[font=\small\bfseries] at (5.90,6.05) {Sparse operator regression, $\ftilde$};
\foreach \i/\j/\t in {0/0/80,0/1/36,0/2/31,0/3/25,1/0/36,1/1/80,1/2/55,1/3/25,2/0/31,2/1/55,2/2/80,2/3/28,
  2/4/25,2/5/25,3/0/25,3/1/25,3/2/28,3/3/80,3/5/44,3/6/41,3/8/41,4/2/25,4/4/80,4/5/25,
  4/7/25,4/9/39,4/10/26,5/2/25,5/3/44,5/4/25,5/5/80,5/6/50,5/8/59,5/10/25,6/3/41,6/5/50,
  6/6/80,6/8/76,7/4/25,7/7/80,7/9/25,7/11/25,8/3/41,8/5/59,8/6/76,8/8/80,9/4/39,9/7/25,
  9/9/80,9/10/47,9/11/25,10/4/26,10/5/25,10/9/47,10/10/80,11/7/25,11/9/25,11/11/80,11/13/25,11/15/25,
  12/12/80,12/14/32,12/16/28,12/17/25,12/18/27,13/11/25,13/13/80,13/15/71,13/16/25,13/17/38,13/18/25,13/19/31,
  13/20/26,14/12/32,14/14/80,14/16/25,14/18/25,15/11/25,15/13/71,15/15/80,15/16/25,15/17/43,15/18/25,15/19/41,
  15/20/30,16/12/28,16/13/25,16/14/25,16/15/25,16/16/80,16/17/39,16/18/67,17/12/25,17/13/38,17/15/43,17/16/39,
  17/17/80,17/18/30,17/19/30,17/20/27,18/12/27,18/13/25,18/14/25,18/15/25,18/16/67,18/17/30,18/18/80,18/21/25,
  19/13/31,19/15/41,19/17/30,19/19/80,19/20/70,20/13/26,20/15/30,20/17/27,20/19/70,20/20/80,21/18/25,21/21/80,
  21/22/67,21/23/29,22/21/67,22/22/80,22/23/32,23/21/29,23/22/32,23/23/80}
  \fill[blue!\t!white] ({3.55+\j*0.150},{4.35-(\i+1)*0.150}) rectangle +(0.150,0.150);
\draw[gray] (3.55,0.75) rectangle (7.15,4.35);
\node[above] at (5.35,4.37) {$K_{UU}^{\theta}=\bigl[\lambda(\ba_i,\ba_{i'})\bigr]+\theta I$};
\foreach \i/\t in {4/26,5/20,6/20,7/20,8/20,9/51,10/95,12/33,14/21,15/20,16/21,17/20,
  18/20}
  \fill[red!\t!white] (7.35,{4.35-(\i+1)*0.150}) rectangle +(0.28,0.150);
\draw[gray] (7.35,0.75) rectangle (7.63,4.35);
\node[above] at (7.49,4.37) {$\ell(\ba_*)$};
\node at (7.93,2.55) {$=$};
\foreach \i/\t in {9/21,10/95,12/36}
  \fill[orange!\t!white] (8.23,{4.35-(\i+1)*0.150}) rectangle +(0.28,0.150);
\draw[gray] (8.23,0.75) rectangle (8.51,4.35);
\node[above] at (8.52,4.37) {$k_U(\ba_*)$};
\node[font=\footnotesize,align=center] at (5.35,0.25) {one sparse Cholesky solve\\per new input};
\draw[->,gray!70!black] (2.9,3.13) -- (3.47,3.13);
\draw[->,orange!80!black] (2.6,-0.15) -- (8.37,-0.15) -- (8.37,0.69);
\node[font=\small\bfseries] at (12.70,6.05) {Local output reconstruction, $\chi$};
\draw[gray] (10.00,0.60) rectangle (14.00,4.60);
\begin{scope}[shift={(10.00,0.60)},x=4.00cm,y=4.00cm]
\foreach \x/\y in {0.070/0.028,0.157/0.026,0.426/0.062,0.474/0.038,0.567/0.028,0.665/0.066,0.751/0.034,0.863/0.066,0.982/0.058,0.029/0.139,
  0.148/0.110,0.503/0.119,0.591/0.115,0.681/0.138,0.789/0.130,0.836/0.138,0.937/0.150,0.039/0.220,0.160/0.221,0.504/0.241,
  0.617/0.247,0.696/0.245,0.791/0.213,0.863/0.218,0.941/0.231,0.035/0.325,0.141/0.295,0.224/0.312,0.390/0.319,0.494/0.328,
  0.581/0.301,0.752/0.305,0.854/0.302,0.957/0.344,0.024/0.404,0.146/0.431,0.256/0.388,0.320/0.431,0.383/0.414,0.481/0.425,
  0.949/0.395,0.057/0.489,0.129/0.515,0.227/0.517,0.319/0.480,0.406/0.522,0.504/0.527,0.886/0.511,0.957/0.504,0.055/0.581,
  0.146/0.598,0.244/0.590,0.291/0.593,0.430/0.599,0.508/0.585,0.583/0.604,0.795/0.618,0.842/0.569,0.977/0.585,0.032/0.675,
  0.146/0.662,0.427/0.678,0.496/0.685,0.578/0.704,0.665/0.684,0.753/0.689,0.887/0.704,0.933/0.689,0.068/0.755,0.505/0.784,
  0.592/0.763,0.678/0.797,0.755/0.752,0.835/0.793,0.983/0.757,0.032/0.846,0.427/0.881,0.524/0.864,0.597/0.867,0.664/0.864,
  0.798/0.887,0.836/0.873,0.944/0.842,0.024/0.972,0.163/0.956,0.220/0.955,0.337/0.969,0.384/0.938,0.521/0.954,0.600/0.971,
  0.690/0.968,0.775/0.939,0.845/0.983,0.926/0.966}
  \fill[gray!60] (\x,\y) circle (1.2pt);
\draw[teal!70!black,fill=teal!70!black,fill opacity=0.08] (0.241,0.770) circle (0.76cm);
\foreach \x/\y in {0.326/0.783,0.246/0.847,0.227/0.691,0.340/0.862,0.401/0.782,0.131/0.785,0.331/0.653,0.142/0.863} \fill[teal!70!black] (\x,\y) circle (1.5pt);
\fill[teal!70!black] (0.241,0.770) circle (2.2pt);
\draw[violet!80!black,fill=violet!80!black,fill opacity=0.08] (0.697,0.511) circle (0.95cm);
\foreach \x/\y in {0.768/0.429,0.776/0.510,0.653/0.384,0.616/0.411,0.562/0.491,0.657/0.612,0.848/0.419,0.707/0.303} \fill[violet!80!black] (\x,\y) circle (1.5pt);
\fill[violet!80!black] (0.697,0.511) circle (2.2pt);
\draw[green!50!black,fill=green!50!black,fill opacity=0.08] (0.323,0.146) circle (0.71cm);
\foreach \x/\y in {0.317/0.204,0.324/0.073,0.419/0.138,0.238/0.129,0.412/0.252,0.218/0.218,0.254/0.053,0.330/0.294} \fill[green!50!black] (\x,\y) circle (1.5pt);
\fill[green!50!black] (0.323,0.146) circle (2.2pt);
\end{scope}
\foreach \t/\s in {0/74,1/43,2/34,3/51,4/47,5/54,6/52}
  \fill[teal!70!black!\s!white,draw=teal!70!black] ({14.45+\t*0.17},3.59) rectangle +(0.17,0.17);
\draw[teal!70!black,densely dotted] (10.96,3.68) -- (14.42,3.68);
\node[above right,font=\footnotesize,inner sep=1pt] at (14.40,3.79) {$\bb_*^{j_1}=B_{j_1}^{\mathsf T}\ell(\ba_*)$};
\foreach \t/\s in {0/77,1/37,2/77,3/34,4/77,5/47,6/66}
  \fill[violet!80!black!\s!white,draw=violet!80!black] ({14.45+\t*0.17},2.56) rectangle +(0.17,0.17);
\draw[violet!80!black,densely dotted] (12.79,2.64) -- (14.42,2.64);
\node[above right,font=\footnotesize,inner sep=1pt] at (14.40,2.76) {$\bb_*^{j_2}=B_{j_2}^{\mathsf T}\ell(\ba_*)$};
\foreach \t/\s in {0/32,1/53,2/64,3/62,4/70,5/39,6/51}
  \fill[green!50!black!\s!white,draw=green!50!black] ({14.45+\t*0.17},1.10) rectangle +(0.17,0.17);
\draw[green!50!black,densely dotted] (11.29,1.18) -- (14.42,1.18);
\node[above right,font=\footnotesize,inner sep=1pt] at (14.40,1.30) {$\bb_*^{j_3}=B_{j_3}^{\mathsf T}\ell(\ba_*)$};
\node[font=\footnotesize,align=center] at (12.00,0.05) {every stencil $S_j$, $j=1,\ldots,m$, has its own\\coefficients, all from the same $\ell(\ba_*)$};
\draw[->,red!70!black,thick] (8.61,2.55) -- (9.90,2.55) node[midway,above] {$\ell(\ba_*)$};
\end{tikzpicture}

%% file: sections/methods.tex
\section{Local property-preserving output reconstruction}
\label{sec:methods}
We now construct the output reconstruction map $\chi$ of~\eqref{eq:sappkm_decomposition}. In PPKM~\citep{sharma2026fluids}, $\chi$ is a single divergence-free kernel interpolant on all $m$ output points, so its linear system grows with the output resolution. We replace it by local interpolants, one on the stencil of each output point. Each local interpolant combines a divergence-free kernel with divergence-free polynomials, so that it is divergence free for any coefficient vector and reproduces low-degree divergence-free fields. Its system has a fixed size as $m$ grows. We first describe the kernel, the polynomials, and the local system, and then the extensions to turbulent fields, periodic domains, and spacetime outputs. Nothing in this section depends on the operator regression of Section~\ref{sec:operator}.

\subsection{Divergence-free interpolation on stencils}
\label{sec:local_dfphs}
\subsubsection{The divergence-free kernel}
The output points $Y$ serve as the nodes of the output reconstruction. For each node $y_j\in Y$, let the stencil $S_j=\{\tilde y_{j,1},\ldots,\tilde y_{j,n_s}\}$ consist of its $n_s$ nearest nodes. In the tests of Section~\ref{sec:spatial-convergence} on bounded domains, we instead use the boundary-aware stencils of KernelPack-JAX~\citep{cattaneo2026kernelpack}, which keep only nodes visible from the center inside the domain. In periodic domains, some of these are translated images of nodes (Section~\ref{sec:periodic_stencils}). For smooth outputs, we use the cubic PHS $\phi(x,y) = \phi_3(\|x - y\|_2)$, with radial function
\begin{align}
\phi_3(r) = r^3,
\label{eq:phs_scalar}
\end{align}
which has no shape parameter to choose (Section~\ref{sec:local_multiscale} relates it to the Mat\'ern kernels that we use when a length scale is wanted). For the sixth-order tests of Section~\ref{sec:spatial-convergence}, we use $\phi(x,y)=\|x-y\|_2^5$ in the same way. Applying the operator $-\Delta I+\nabla\nabla^{\mathsf T}$ to a scalar kernel gives a matrix-valued kernel whose columns are divergence free~\citep{narcowich1994generalized,fuselier2016high},
\begin{align}
\Phi_{\rm DF}(x,y)=\left(-\Delta_x I+\nabla_x\nabla_x^{\mathsf T}\right)\phi(x,y). \label{eq:dfphs_kernel}
\end{align}
For the cubic PHS, with $z=x-y$ and $r=\|z\|_2>0$,
\begin{equation}
\Phi_{\rm DF}(x,y)=-3dr I+3zz^{\mathsf T}/r,\qquad \Phi_{\rm DF}(y,y)=0.
\label{eq:cubic-df-explicit}
\end{equation}
Each column of $\Phi_{\rm DF}$ has zero divergence in the distributional sense, and zero classical divergence away from its center~\citep{fuselier2016high}. At the center the kernel is continuous but not differentiable. For example, along $z=t\mathbf e_k$, with $\mathbf e_k$ the $k$th coordinate vector, its $k$th diagonal entry is $-3(d-1)|t|$.

\subsubsection{Divergence-free polynomials of any degree}
A PHS interpolant converges at the order set by the degree of its polynomial augmentation~\citep{flyer2016role,bayona2017role}. A divergence-free local interpolant of high order therefore needs divergence-free vector polynomials of any degree $\varpi$, in 2D and 3D, that can be evaluated at any point and not only at the stencil nodes. Our approach is to differentiate scalar potentials. On $S_j$, let $\xi=(y-\bar y_j)/\rho_j$ be centered and scaled coordinates, where $\bar y_j=n_s^{-1}\sum_{l}\tilde y_{j,l}$ is the centroid of the stencil and $\rho_j=\max_l\|\tilde y_{j,l}-\bar y_j\|_2$ its radius. From the products $J_\alpha(\xi)=\prod_kP_{\alpha_k}(\xi_k)$ of normalized Legendre polynomials $P_n$, of total degree $|\alpha|\leq\varpi+1$, we form the rotated gradients in 2D and the curls of the potentials times the coordinate vectors in 3D,
\begin{align}
 g_\alpha(y)&=\nabla_\xi^\perp J_\alpha(\xi), &&d=2,\\
 g_{\alpha,k}(y)&=\nabla_\xi\times\bigl(J_\alpha(\xi)\mathbf e_k\bigr), &&d=3,\quad k=1,2,3,
\label{eq:dfpoly_candidates}
\end{align}
where $\alpha$ is a multi-index, $\nabla_\xi^\perp=(\partial_{\xi_2},-\partial_{\xi_1})^{\mathsf T}$. These candidates have degree at most $\varpi$ and are divergence free in $\xi$, and hence in $y$, since the scaling is uniform. Together, they span the space of divergence-free vector polynomials of degree at most $\varpi$, which has dimension $d\binom{\varpi+d}{d}-\binom{\varpi+d-1}{d}$ (Appendix~\ref{app:polynomial-span}).

The candidates are not independent. Constant potentials give zero fields, different potentials can give the same field, and a particular stencil can add further dependencies. Let $g_1,\ldots,g_{n_{{\rm cand},j}}$ be the candidates, let $G_j(y)=[g_1(y)\ \cdots\ g_{n_{{\rm cand},j}}(y)]\in\mathbb R^{d\times n_{{\rm cand},j}}$ hold their values at $y$, and let $G_j\in\mathbb R^{dn_s\times n_{{\rm cand},j}}$ hold their values at the stencil nodes in component block order. We select an independent subset by column-pivoted QR,
\begin{align}
 G_j\Pi_j=Q_jR_j,
\label{eq:dfpoly_cpqr}
\end{align}
and retain the first $n_{P,j}$ pivoted columns, where $n_{P,j}$ counts the diagonal entries of $R_j$ larger than $10^{-12}$ times the first. Let $I_j$ hold the indices of the retained candidates, $G_{j,I}$ the corresponding columns of $G_j$, $Q_{j,1}$ the first $n_{P,j}$ columns of $Q_j$, and $R_{j,11}\in\mathbb R^{n_{P,j}\times n_{P,j}}$ the leading block of $R_j$. Then
\begin{align}
 G_{j,I}=Q_{j,1}R_{j,11},\qquad
 G_{j,I}R_{j,11}^{-1}=Q_{j,1}.
\label{eq:dfpoly_stencil_basis}
\end{align}
The orthonormal columns of $Q_{j,1}$ are only values at the nodes. To obtain functions, we apply the same triangular transformation to the retained candidate functions $G_{j,I}(y)=[g_c(y)]_{c\in I_j}$,
\begin{align}
 \bP_j(y)=G_{j,I}(y)R_{j,11}^{-1}.
\label{eq:dfpoly_change_basis}
\end{align}
By~\eqref{eq:dfpoly_stencil_basis}, $\bP_j=Q_{j,1}$ at the stencil nodes, so the polynomial block of the local system is well conditioned. Every column of $\bP_j(y)$ is a fixed combination of divergence-free fields, and is thus divergence free everywhere. At a new point, we evaluate the retained candidates and apply $R_{j,11}^{-1}$ by a triangular solve. If the stencil geometry forces the rank below the full dimension, the interpolant reproduces only the retained subspace.

We batch this construction over the stencils. On stencils that are certainly of full rank, pivoting is unnecessary, and we use an unpivoted QR factorization instead (Appendix~\ref{app:local-solves}).

\subsubsection{The local system}
On $S_j$, we determine the kernel and polynomial coefficients together, from the interpolation conditions and the polynomial moment conditions. The saddle-point matrix
\begin{align}
\bA_j
&=
\begin{bmatrix}
\bPhi_j & \bP_j\\
\bP_j^{\mathsf T} & 0
\end{bmatrix}\in\mathbb{R}^{r_j\times r_j},
\label{eq:local_dfphs_matrix}
\end{align}
has size $r_j=dn_s+n_{P,j}$, where $\bPhi_j\in\mathbb{R}^{dn_s\times dn_s}$ is the Gram matrix of the divergence-free kernel on $S_j$. For training field $i$, the coefficient vector
\begin{align*}
    \bb_i^j =
    \begin{bmatrix}
        \bb_{i,\Phi}^j\\
        \bb_{i,P}^j
    \end{bmatrix}
    \in \mathbb{R}^{r_j}
\end{align*}
holds kernel coefficients $\bb_{i,\Phi}^j \in \mathbb{R}^{dn_s}$ and polynomial coefficients $\bb_{i,P}^j \in \mathbb{R}^{n_{P,j}}$, and defines the local interpolant
\begin{align}
\chi_i^j(y)
=\sum_{l=1}^{n_s}\Phi_{\rm DF}(y,\tilde y_{j,l})\bb_{i,\Phi}^{j,l}
+\bP_j(y)\bb_{i,P}^j,
\label{eq:local_dfphs_interpolant}
\end{align}
where $\bb_{i,\Phi}^{j,l}\in\mathbb{R}^{d}$ holds the kernel coefficients of node $l$. Interpolation of the training field at the stencil nodes (with the value of the original node at a periodic image) and orthogonality of the kernel coefficients to the polynomials give
\begin{align}
\bA_j\bb_i^j
=
\begin{bmatrix}
\left.\bv_i\right|_{S_j}\\0
\end{bmatrix},
\label{eq:local_dfphs_system}
\end{align}
where $\left.\bv_i\right|_{S_j}\in\mathbb{R}^{dn_s}$ holds the sampled velocities on the stencil. Since $\bP_j$ has full column rank, $\bA_j$ is nonsingular whenever $\bPhi_j$ is positive definite on the null space of $\bP_j^{\mathsf T}$. Conditional positive definiteness of the scalar PHS does not guarantee this for the matrix kernel on a particular stencil, so we check the solves numerically (Appendix~\ref{app:local-solves}). We factor all local systems by batched LU on the GPU.

\paragraph{Evaluation}
To reconstruct a prediction at query points $y_q^\star\in Y^\star$, $q=1,\ldots,m^\star$, we assign each query point to a stencil $S_{j(q)}$ and evaluate that stencil's basis there. Let $\bPhi_{j(q),q,\mathrm{comp}}^\star\in\mathbb R^{d\times dn_s}$ hold the kernel values in the component block order of the coefficients (Appendix~\ref{app:coefficient-ordering}) and $\bP_{j(q),q}^\star \in \mathbb{R}^{d \times n_{P,j(q)}}$ the polynomial values, and let
\begin{align}
\bA_q^\star =
\begin{bmatrix}
    \bPhi_{j(q),q,\mathrm{comp}}^\star & \bP_{j(q),q}^\star
\end{bmatrix}
\in \mathbb{R}^{d \times r_{j(q)}}.
\label{eq:local_dfphs_evaluation}
\end{align}
For coefficients $\bb_*^{j}$ predicted by the operator regression (Section~\ref{sec:lagrange}), the reconstructed velocity is then
\begin{align}
\tilde{\bv}_q^\star = \bA_q^\star \bb_*^{j(q)}
\in \mathbb{R}^{d},
\label{eq:local_dfphs_prediction}
\end{align}
and $\tilde{\bv}^\star$ stacks the $\tilde{\bv}_q^\star$ in component blocks. In the operator benchmarks, the query points are the nodes themselves, each with its own stencil. Appendix~\ref{app:stencil-assignment} gives the assignment of other points.

The output reconstruction is divergence free on each stencil, for any coefficients. However, the local interpolants are not coupled, so the reconstructed field can jump where the assigned stencil changes, and its divergence is controlled only within each stencil.

On the other hand, every system has the fixed size $r_j$, independent of the number of output points $m$ and of the number of training functions $N$. Refining the output adds systems but does not enlarge them. The systems are also independent of each other, so their assembly, factorization, and evaluation are embarrassingly parallel, and we process them in batches on the GPU with no communication between stencils. In contrast, a global output reconstruction couples all output points in one dense system with $dm$ unknowns, whose factorization costs $O((dm)^3)$ operations and does not decompose into independent parts. Section~\ref{sec:spatial-cost} measures both costs. Section~\ref{sec:commuted-prediction} shows how the same systems serve every prediction.

\subsection{Multiscale Mat\'ern output reconstruction for turbulence}\label{sec:local_multiscale}
Turbulent velocity fields vary over many length scales, which a kernel can represent only if it has length scales of its own. Polyharmonic splines have none, since they are homogeneous (dilating the argument of $r^3$ or $r^5$ multiplies it by a constant, and so does the derived matrix kernel), and sums of dilated copies add nothing. Mat\'ern kernels have one length each, and they are close relatives of the PHS. The Mat\'ern kernels $M_{3/2}(r/\sigma)$ and $M_{5/2}(r/\sigma)$ of length $\sigma$, with $M_{3/2}(r)=(1+r)e^{-r}$ and $M_{5/2}$ given in~\eqref{eq:local_matern_scalar} below, have Fourier transforms proportional to $(\sigma^{-2}+\|\mathbf k\|_2^2)^{-(3+d)/2}$ and $(\sigma^{-2}+\|\mathbf k\|_2^2)^{-(5+d)/2}$ in the wave vector $\mathbf k$. They are therefore, up to constants, the Green's functions of $(\sigma^{-2}-\Delta)^{(3+d)/2}$ and $(\sigma^{-2}-\Delta)^{(5+d)/2}$ in $\mathbb R^d$~\citep{whittle1954stationary,stein1999interpolation,lindgren2011explicit,fasshauer2011reproducing}. Letting $\sigma\to\infty$ removes the length and leaves $(-\Delta)^{(3+d)/2}$ and $(-\Delta)^{(5+d)/2}$, whose fundamental solutions are multiples of $r^3$ and $r^5$~\citep{duchon1977splines,Wendland2004}. The two Mat\'ern kernels are thus the counterparts, with a length, of the PHS $r^3$ and $r^5$. For example,
\begin{align}
M_{5/2}\!\left(\frac{r}{\sigma}\right)=1-\frac{r^2}{6\sigma^2}+\frac{r^4}{24\sigma^4}-\frac{r^5}{45\sigma^5}+O\!\left(\frac{r^6}{\sigma^6}\right).
\label{eq:matern_phs_limit}
\end{align}
The even terms are polynomials in $x-y$, and after the operator $-\Delta I+\nabla\nabla^{\mathsf T}$ their contribution to an interpolant is a divergence-free polynomial of degree at most two, which the polynomial part of the local interpolant absorbs when $\varpi\ge 2$. Since rescaling the kernel does not change the interpolant, the local interpolant with $M_{5/2}(r/\sigma)$ thus tends to the one with $r^5$ as $\sigma\to\infty$. We use the PHS when no length scale is wanted, for smooth outputs whose accuracy is set by the polynomial degree, and Mat\'ern kernels when lengths are wanted. A mixture of Mat\'ern kernels can have many lengths. This motivates localizing the multiscale Mat\'ern construction of our earlier work~\citep{sharma2026fluids}, which follows~\citet{owhadi2023gaussian}. On $S_j$, we sum five Mat\'ern kernels with dyadic lengths and weights
\begin{align}
\sigma_{j,\kappa}=\frac{\sigma_{j,0}}{2^\kappa},\qquad
\sigma_{j,0}=c_\sigma\rho_j,\qquad
\vartheta_{j,\kappa}=\left(\frac{\sigma_{j,\kappa}}{\sigma_{j,0}}\right)^{8/3},\qquad \kappa=1,\ldots,5,
\label{eq:local_matern_scales}
\end{align}
where $\rho_j$ is the stencil radius and $c_\sigma$ a fixed factor. Each scale uses the scalar Mat\'ern function
\begin{align}
M_{5/2}(r)=\left(1+r+\frac{r^2}{3}\right)e^{-r}
\label{eq:local_matern_scalar}
\end{align}
at the argument $r/\sigma_{j,\kappa}$.
The weights are applied to the scalar kernels, before the divergence-free operator, whose two derivatives account for the $2$ in the exponent $8/3=2+2/3$. The remaining $2/3$ is the Kolmogorov exponent of the velocity increments in three dimensions. The local multiscale kernel is
\begin{align}
\Phi_j^{\rm M}(x,y)=\frac{1}{c_j}\left(-\Delta_x I+\nabla_x\nabla_x^{\mathsf T}\right)
\sum_{\kappa=1}^{5}3\vartheta_{j,\kappa}M_{5/2}\!\left(\frac{\|x-y\|_2}{\sigma_{j,\kappa}}\right),\qquad
c_j=(d-1)\sum_{\kappa=1}^{5}\vartheta_{j,\kappa}\sigma_{j,\kappa}^{-2},
\label{eq:local_multiscale_matern}
\end{align}
normalized so that $\Phi_j^{\rm M}(y,y)=I$. The normalization rescales the coefficients and does not change the interpolant. We replace $\Phi_{\rm DF}$ by $\Phi_j^{\rm M}$ in~\eqref{eq:local_dfphs_matrix} and~\eqref{eq:local_dfphs_evaluation} and keep the polynomials, so each stencil still solves one system with $r_j$ unknowns for the whole mixture. Unlike the cubic kernel, the mixture is smooth at its centers, so the reconstructed fields are classically divergence free everywhere on the stencil.

The weights build the prescribed increment statistics into the kernel~\citep{owhadi2023gaussian}. Whether an interpolant on finite stencils reproduces those statistics also depends on the stencil extent and the node spacing, and we test this under refinement in Section~\ref{sec:spatial-convergence} and for the turbulent benchmarks in Appendix~\ref{app:multiscale-spatial}.

\subsection{Periodic stencils}
\label{sec:periodic_stencils}
In a periodic domain, the nearest neighbors of a node near one face of the box can lie near the opposite face. Our stencils are therefore built from translated images of the nodes. If coordinate $k$ has period $\Lambda_k$, a node $y_l$ has images $y_l+\sum_k\nu_k\Lambda_k\mathbf e_k$, with integers $\nu_k$ over the periodic coordinates, which carry its index and its sampled velocity,
\begin{align}
 v_i\Bigl(y_l+\sum_k\nu_k\Lambda_k\mathbf e_k\Bigr)=v_i(y_l).
\label{eq:periodic_image_values}
\end{align}
For a stencil center $y_j$, we take the displacements $\delta_l=y_l-y_j$ and shift every periodic component to its nearest representative,
\begin{equation}
\delta_{l,k}\leftarrow\delta_{l,k}-\Lambda_k\operatorname{round}(\delta_{l,k}/\Lambda_k).
\label{eq:periodic-displacement}
\end{equation}
Since all periodic coordinates are shifted together, this also finds the right images near the edges and corners of the box, where a node's neighbors are shifted in two or three coordinates (Figure~\ref{fig:periodic}).
\begin{figure}[tbp]
\centering
\resizebox{0.5\textwidth}{!}{\input{sections/fig_periodic}}
\caption{Periodic stencils in a box with periods $\Lambda_1$ and $\Lambda_2$. Gray dots are nodes (dark) and their translated images (light). A stencil near a face (blue) and a stencil near a corner (red) are built from the nearest images given by~\eqref{eq:periodic-displacement}. Dashed arrows link each image in a stencil to its node (circled), whose sampled velocity it carries.}
\label{fig:periodic}
\end{figure}
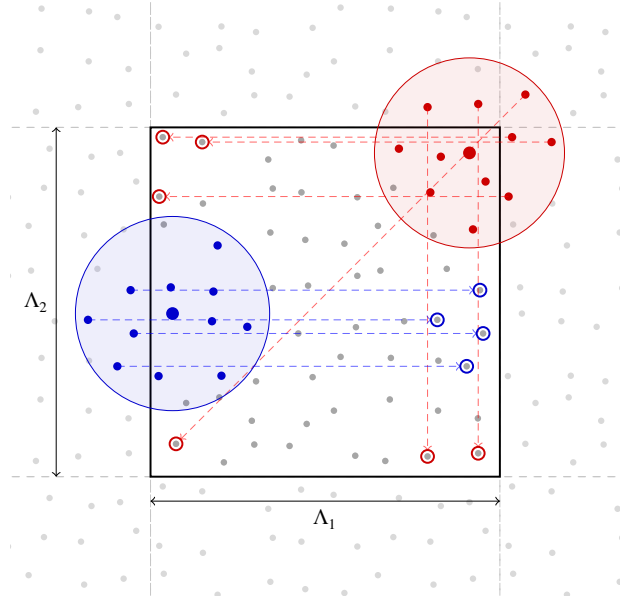 The stencil consists of the $n_s$ nearest of the points $y_j+\delta_l$, with one representative of each node, and the kernel and polynomials are evaluated in these translated coordinates. Translation does not change the divergence. Appendix~\ref{app:periodic-correction} gives the periods used for each benchmark.

\subsection{Spacetime outputs}\label{sec:spacetime}
For outputs on $\Omega_v=\Omega\times\Gamma$, with $\Gamma$ a time interval, PPKM~\citep{sharma2026fluids} used a global product kernel
\begin{align}
\Phi_{\rm ST}((y,t),(y',t'))=\Phi_{\rm DF}^{\rm Mat}(y,y')\Phi_t(t,t'),
\label{eq:spacetime_product_kernel}
\end{align}
with a divergence-free Mat\'ern kernel in space and a scalar Mat\'ern kernel $\Phi_t$ in time, and $dmN_t$ coefficients in one system. We instead keep each stencil and its spatial basis fixed over the sampled times $\mathcal T=\{t_s\}_{s=1}^{N_t}$ and interpolate in time. Let $\bb_i^{j,s}\in\mathbb R^{r_j}$ be the coefficients of training field $i$ on $S_j$ at time $t_s$. Every snapshot solves~\eqref{eq:local_dfphs_system} with the same matrix $\bA_j$, and the target of the operator regression is the stacked vector
\begin{align}
\bb_{i,\mathcal T}^j=[(\bb_i^{j,1})^{\mathsf T},\ldots,(\bb_i^{j,N_t})^{\mathsf T}]^{\mathsf T}\in\mathbb R^{N_t r_j}.
\label{eq:spacetime_coefficient_snapshots}
\end{align}
Between the sampled times, we interpolate with the Lagrange polynomials
\begin{align}
\omega_s(t)=\prod_{s'\ne s}\frac{t-t_{s'}}{t_s-t_{s'}},\qquad s=1,\ldots,N_t,
\label{eq:temporal_polynomial_weights}
\end{align}
which need no temporal kernel parameter and no further linear system. With $\bA_j^\star(y)$ the evaluation of the spatial basis on $S_j$ as in~\eqref{eq:local_dfphs_evaluation}, the spacetime output reconstruction of a prediction is
\begin{align}
\chi_*^j(y,t)=\sum_{s=1}^{N_t}\omega_s(t)\bA_j^\star(y)\bb_*^{j,s}.
\label{eq:local_spacetime_interpolant}
\end{align}
Its polynomial part is the tensor product of the retained spatial polynomials with polynomials of degree at most $N_t-1$ in time. Since the weights $\omega_s$ do not depend on $y$, the spacetime output reconstruction is divergence free wherever the spatial one is. The local representation has $N_t\sum_j r_j$ coefficients, but each system keeps the size $r_j$. Another sampled time adds only a right-hand side.

%% file: sections/fig_periodic.tex
\begin{tikzpicture}[x=5cm,y=5cm,font=\small]
\clip (-0.40,-0.34) rectangle (1.36,1.36);
\foreach \a/\b in {-1/-1,-1/0,-1/1,0/-1,0/1,1/-1,1/0,1/1}
  \draw[gray!45,dashed] (\a,\b) rectangle +(1,1);
\draw[thick] (0,0) rectangle (1,1);
\def\nodes{0.073/0.094,0.210/0.041,0.297/0.092,0.398/0.088,0.586/0.060,0.672/0.045,0.793/0.058,0.938/0.067,
  0.102/0.211,0.197/0.227,0.290/0.160,0.447/0.151,0.525/0.189,0.685/0.221,0.823/0.189,0.937/0.167,
  0.023/0.288,0.203/0.289,0.302/0.273,0.464/0.285,0.544/0.343,0.688/0.340,0.824/0.332,0.905/0.316,
  0.063/0.467,0.176/0.445,0.277/0.429,0.423/0.410,0.588/0.428,0.726/0.445,0.821/0.449,0.952/0.410,
  0.058/0.542,0.180/0.530,0.350/0.540,0.451/0.547,0.592/0.575,0.658/0.590,0.848/0.595,0.943/0.534,
  0.038/0.722,0.192/0.662,0.343/0.699,0.443/0.678,0.555/0.667,0.651/0.718,0.810/0.691,0.923/0.708,
  0.025/0.802,0.150/0.782,0.350/0.825,0.432/0.814,0.592/0.800,0.695/0.827,0.801/0.814,0.959/0.845,
  0.035/0.972,0.148/0.958,0.337/0.908,0.431/0.963,0.524/0.948,0.711/0.939,0.831/0.916,0.913/0.927}
\foreach \a in {-1,0,1} \foreach \b in {-1,0,1} \foreach \x/\y in \nodes
  \fill[gray!30] ({\x+\a},{\y+\b}) circle (1.3pt);
\foreach \x/\y in \nodes \fill[gray!70] (\x,\y) circle (1.3pt);
\draw[red!80!black,fill=red!80!black,fill opacity=0.07] (0.913,0.927) circle (1.36cm);
\foreach \x/\y in {0.913/0.927,0.831/0.916,0.959/0.845,1.035/0.972,0.938/1.067,0.801/0.814,1.025/0.802,0.793/1.058,0.711/0.939,0.923/0.708,1.073/1.094,1.148/0.958} \fill[red!80!black] (\x,\y) circle (1.7pt);
\foreach \x/\y/\u/\v in {1.035/0.972/0.035/0.972,0.938/1.067/0.938/0.067,1.025/0.802/0.025/0.802,0.793/1.058/0.793/0.058,
  1.073/1.094/0.073/0.094,1.148/0.958/0.148/0.958} {
  \draw[red,densely dashed,->,opacity=0.6,shorten >=2.5pt] (\x,\y) -- (\u,\v);
  \draw[red!80!black,thick] (\u,\v) circle (2.6pt);}
\fill[red!80!black] (0.913,0.927) circle (2.6pt);
\draw[blue!80!black,fill=blue!80!black,fill opacity=0.07] (0.063,0.467) circle (1.39cm);
\foreach \x/\y in {0.063/0.467,0.058/0.542,0.176/0.445,-0.048/0.410,0.180/0.530,-0.057/0.534,0.023/0.288,0.277/0.429,-0.095/0.316,0.203/0.289,0.192/0.662,-0.179/0.449} \fill[blue!80!black] (\x,\y) circle (1.7pt);
\foreach \x/\y/\u/\v in {-0.048/0.410/0.952/0.410,-0.057/0.534/0.943/0.534,-0.095/0.316/0.905/0.316,-0.179/0.449/0.821/0.449} {
  \draw[blue,densely dashed,->,opacity=0.6,shorten >=2.5pt] (\x,\y) -- (\u,\v);
  \draw[blue!80!black,thick] (\u,\v) circle (2.6pt);}
\fill[blue!80!black] (0.063,0.467) circle (2.6pt);
\draw[<->] (0,-0.07) -- (1,-0.07) node[midway,below] {$\Lambda_1$};
\draw[<->] (-0.27,0) -- (-0.27,1) node[midway,left] {$\Lambda_2$};
\end{tikzpicture}

%% file: sections/operator.tex
\section{Sparse operator regression}\label{sec:operator}
Having fixed the output representation, we now learn its coefficients as functions of the input. Each output coefficient is a function on the space of sampled inputs, and we approximate all of them by kernel regression with the same kernel. PPKM~\citep{sharma2026fluids} used a kernel with global support, whose $N\times N$ kernel matrix is dense, so that its Cholesky factorization costs $O(N^3)$ operations and $O(N^2)$ storage for $N$ training functions. We now construct a compactly supported kernel that is positive semidefinite on inputs of any dimension, so that the kernel matrix is sparse. The rest of the section defines the operator interpolant built from this kernel and its affine augmentation, derives their linear systems, factors them, and chooses the kernel parameters.

\subsection{A compactly supported kernel in any input dimension}
\label{sec:sparse_wendland}
To make the kernel matrix sparse, the kernel must vanish for inputs that are far apart, and it must remain positive definite on the space of sampled inputs, whose dimension $nd_a$ can be in the thousands. These two requirements conflict for radial kernels. By Schoenberg's theorem, a radial function that is positive definite in every dimension is a scale mixture of Gaussians and therefore never compactly supported~\citep{Wendland2004}. Every compactly supported radial kernel~\citep{wendland1995piecewise,gneiting2002compactly} is thus positive definite only up to some dimension $D_{\max}$. For the Wendland functions, the polynomial degree grows linearly with $D_{\max}$ at fixed smoothness~\citep{wendland1995piecewise}, so a function valid in thousands of dimensions would have a degree in the thousands. Tapering~\citep{furrer2006covariance}, which multiplies a globally supported kernel by a compactly supported one, inherits the same restriction from its taper. Low numerical rank of the training inputs does not remove the restriction either, since a new input need not lie in their span.

Our approach is to separate the two requirements. We let a low-dimensional projection of the input determine the support, and use a compactly supported radial function only there, in a dimension in which it is positive definite. The remaining coordinates enter through a Mat\'ern factor, which is positive definite in every dimension but not compactly supported. The product is positive semidefinite in any input dimension and has the support of its first factor. In this paper, we use the $C^4$ Wendland function
\begin{equation}
w(t)=(1-t)_+^6\left(1+6t+\frac{35}{3}t^2\right),
\label{eq:wendland}
\end{equation}
which is positive definite in up to three dimensions~\citep{wendland1995piecewise}, so we project onto at most three coordinates. Another member of the Wendland family would allow a projection of higher dimension.

Let $\bar\ba$ be the sampled input after centering and scaling each coordinate by its training mean and standard deviation (coordinates with zero training variance are set to zero). Let the columns of $V_z$ be the first $n_z$ principal directions of the standardized training inputs, where $n_z$ is the smaller of three and their numerical rank (Appendix~\ref{app:residual-gate}), and split
\begin{equation}
z(\ba)=V_z^{\mathsf T}\bar\ba,\qquad
\zeta(\ba)=(I-V_zV_z^{\mathsf T})\bar\ba.
\label{eq:operator-projection}
\end{equation}
For a support radius $\varrho>0$, a residual scale $\alpha_\perp>0$, and a positive diagonal matrix $D_\parallel$, our kernel is
\begin{align}
\lambda(\ba,\ba')={}&\eta(\ba)\eta(\ba')\,
w\!\left(\frac{\|D_\parallel(z(\ba)-z(\ba'))\|_2}{\varrho}\right)
M_{5/2}\!\left(\frac{\alpha_\perp\|\zeta(\ba)-\zeta(\ba')\|_2}{\delta}\right).
\label{eq:dimension-safe-kernel}
\end{align}
Here $M_{5/2}$ is the Mat\'ern function~\eqref{eq:local_matern_scalar}, $\delta>0$ is a residual length, and $\eta$ is a gate. The Wendland factor sets the support, since two inputs interact only if their principal coordinates are closer than $\varrho$ in the metric $D_\parallel$. The Mat\'ern factor separates inputs with the same principal coordinates. Note that the kernel is not radial in the full input.

The gate switches the kernel off for inputs whose residual $\zeta(\ba)$ is much larger than those of the training inputs. With $\zeta_i=\zeta(\ba_i)$, we set
\begin{equation}
\delta=\max\left\{0.01\max_i\|\bar\ba_i\|_2,\ 2.2\max_i\|\zeta_i\|_2,\ 10^{-8}\right\},
\label{eq:residual-length}
\end{equation}
and, with the polynomial $H(t)=126t^5-420t^6+540t^7-315t^8+70t^9$, which increases from $H(0)=0$ to $H(1)=1$ since $H'(t)=630t^4(1-t)^4$,
\begin{equation}
\eta(\ba)=\hat\eta(\|\zeta(\ba)\|_2/\delta),\qquad
\hat\eta(r)=\begin{cases}
1,&0\leq r\leq1/2,\\
H(2(1-r)),&1/2<r<1,\\
0,&r\geq1.
\end{cases}
\label{eq:residual-gate}
\end{equation}
The gate is $C^4$ and positive for $\|\zeta(\ba)\|_2<\delta$. Since $\|\zeta_i\|_2\leq\delta/2.2$, it equals one on every training input. A new input need not lie in the span of the principal directions $V_z$, and its gate decreases to zero as its residual grows.

\begin{proposition}[Positive definiteness in any input dimension]
\label{prop:dimension-safe-pd}
For any finite sampled-input dimension, the kernel~\eqref{eq:dimension-safe-kernel} is positive semidefinite. Its Gram matrix is strictly positive definite on distinct standardized inputs with positive gates. Consequently, adding a positive multiple of the identity gives a positive-definite kernel matrix, including when sampled inputs coincide or gates vanish.
\end{proposition}
\begin{proof}[Proof sketch]
Both factors are positive definite functions composed with linear maps, and products and gate multiplications preserve positive semidefiniteness. Distinct inputs differ in their principal or their residual coordinates. Appendix~\ref{app:dimension-safe-proof} gives the proof.
\end{proof}
The proposition holds for any orthogonal projection, so it needs no estimate of an intrinsic dimension. The principal components only choose the directions along which the support is cut off.

\paragraph{Fixed and fitted shapes}
We compare two choices of the shape parameters. The fixed-shape kernel uses $D_\parallel=I$ and $\alpha_\perp=1$. The fitted-shape kernel also fits $\alpha_\perp$ and $D_\parallel=\operatorname{diag}(e^{\upsilon_1},\ldots,e^{\upsilon_{n_z}})$, with $\sum_{k=1}^{n_z} \upsilon_k=0$ so that $\varrho$ alone sets the overall support and the $\upsilon_k$ only redistribute it among the principal directions. Note that when $n_z=1$ the constraint forces $D_\parallel=[1]$, and when the residual coordinates are negligible we also keep $\alpha_\perp=1$. The fitted-shape kernel then has the shape of the fixed-shape one and differs from it only through its radius and regularization (Section~\ref{sec:protocol} lists these cases).

\subsection{The operator interpolant}
\label{sec:operator-interpolant}
We now use the kernel $\lambda$ to approximate the map from a sampled input to the local coefficients of its output. Let $\bb_i\in\mathbb R^{m_{\rm rhs}}$ stack the local coefficients $\bb_i^j$ of training field $i$ on all stencils (and at all sampled times), where $m_{\rm rhs}=\sum_jr_j$ (times $N_t$ for spacetime outputs). Every one of these coefficients is a function of the same input, so we approximate all of them with the same scalar kernel, i.e., with the matrix-valued kernel $\lambda(\ba,\ba')I$. The operator regression $\ftilde$ of~\eqref{eq:sappkm_decomposition} is the operator interpolant
\begin{equation}
\ftilde(\ba)=\sum_{i=1}^N\lambda(\ba,\ba_i)\bc_i,
\label{eq:operator-interpolant-plain}
\end{equation}
with coefficients $\bc_i\in\mathbb R^{m_{\rm rhs}}$ determined by the regularized interpolation conditions
\begin{equation}
\ftilde(\ba_i)+\theta\bc_i=\bb_i,\qquad i=1,\ldots,N,
\label{eq:operator-interpolant-plain-conditions}
\end{equation}
where $\theta>0$ is a regularization parameter. In the limit $\theta\to0$, $\ftilde$ interpolates the training coefficients, and for $\theta>0$ it is their kernel ridge regression. Since $\chi$ is linear, the learned operator~\eqref{eq:sappkm_decomposition} is
\begin{equation}
\Gtilde(a)=\sum_{i=1}^N\lambda(\ba,\ba_i)\chi(\bc_i),
\label{eq:operator-continuous-plain}
\end{equation}
a combination of $N$ fixed fields with weights that depend on the input. Every one of these fields is a local divergence-free output reconstruction, so $\Gtilde(a)$ is divergence free on each stencil for every input, regardless of the accuracy of the operator regression.

\subsection{Affine augmentation}
\label{sec:operator-affine}
Operators often depend nearly linearly on part of their input. A compactly supported kernel regression represents such a trend poorly, since it returns to zero away from the training inputs. As in polynomially augmented interpolation~\citep{belianovich2026unified}, we augment the operator regression with affine functions of up to $15$ principal coordinates of the input (more than the at most three that set the support), which it reproduces exactly, and let the kernel represent the remainder. These polynomials in the input are unrelated to the divergence-free polynomials in space of Section~\ref{sec:methods}, which are part of the output representation in every model.

Let the columns of $V_p$ be up to $15$ principal directions of the training inputs, let $\varsigma_k$ be the training standard deviation along direction $k$, and set
\begin{equation}
p(\ba)=\left(1,\frac{(V_p^{\mathsf T}\bar\ba)_1}{\varsigma_1},\ldots,
\frac{(V_p^{\mathsf T}\bar\ba)_{n_{\rm aff}-1}}{\varsigma_{n_{\rm aff}-1}}\right)^{\mathsf T},
\label{eq:operator-mean}
\end{equation}
with $n_{\rm aff}-1\leq15$ limited by the rank of the training inputs. The augmented operator interpolant is
\begin{equation}
\ftilde(\ba)=\sum_{i=1}^N\lambda(\ba,\ba_i)\bc_i+\boldsymbol\beta^{\mathsf T}p(\ba),
\label{eq:operator-interpolant}
\end{equation}
with kernel coefficients $\bc_i\in\mathbb R^{m_{\rm rhs}}$ and affine coefficients $\boldsymbol\beta\in\mathbb R^{n_{\rm aff}\times m_{\rm rhs}}$. Its coefficients are determined by the regularized interpolation conditions of~\eqref{eq:operator-interpolant-plain-conditions}, now for the augmented $\ftilde$, together with moment conditions on the kernel coefficients,
\begin{equation}
\ftilde(\ba_i)+\theta\bc_i=\bb_i,\quad i=1,\ldots,N,\qquad
\sum_{i=1}^Np(\ba_i)\bc_i^{\mathsf T}=0.
\label{eq:operator-interpolant-conditions}
\end{equation}
The moment conditions make~\eqref{eq:operator-interpolant} reproduce every affine function of the principal coordinates exactly, for any $\theta$. If $\bb_i=T^{\mathsf T}p(\ba_i)$ for some matrix $T$, then $\bc_i=0$ and $\boldsymbol\beta=T$ satisfy~\eqref{eq:operator-interpolant-conditions}. As in~\eqref{eq:operator-continuous-plain}, the augmented learned operator is
\begin{equation}
\Gtilde(a)=\sum_{i=1}^N\lambda(\ba,\ba_i)\chi(\bc_i)+\sum_{k=1}^{n_{\rm aff}}p_k(\ba)\chi(\boldsymbol\beta_{k,:}),
\label{eq:operator-continuous}
\end{equation}
a combination of $N+n_{\rm aff}$ fixed divergence-free fields. Without augmentation, $n_{\rm aff}=0$, the affine terms and the moment conditions disappear, and~\eqref{eq:operator-interpolant}--\eqref{eq:operator-continuous} reduce to~\eqref{eq:operator-interpolant-plain}--\eqref{eq:operator-continuous-plain}.

Both interpolants have a variational characterization. The function $\ftilde=f_{\mathcal H}+\boldsymbol\beta^{\mathsf T}p$ minimizes
\begin{equation}
\sum_{i=1}^N\bigl\|f_{\mathcal H}(\ba_i)+\boldsymbol\beta^{\mathsf T}p(\ba_i)-\bb_i\bigr\|_2^2+\theta\|f_{\mathcal H}\|_{\mathcal H}^2
\label{eq:operator-variational}
\end{equation}
over the affine coefficients $\boldsymbol\beta$ and the functions $f_{\mathcal H}$ in the reproducing kernel Hilbert space $\mathcal H$ of $\lambda(\ba,\ba')I$, since by the representer theorem the minimizing $f_{\mathcal H}$ has the form $\sum_i\lambda(\cdot,\ba_i)\bc_i$, and~\eqref{eq:operator-interpolant-conditions} makes the gradient of the convex objective vanish (Appendix~\ref{app:operator-variational}). Without augmentation, $\boldsymbol\beta$ is absent and~\eqref{eq:operator-variational} is the objective of kernel ridge regression.

\subsection{Linear systems}
\label{sec:operator-linear}
We now write the conditions~\eqref{eq:operator-interpolant-plain-conditions} and~\eqref{eq:operator-interpolant-conditions} as linear systems. For training inputs $U=\{\ba_i\}_{i=1}^N$ and $N_*$ query inputs $U_*=\{\ba_*^{(q)}\}_{q=1}^{N_*}$, we define the kernel matrices
\begin{equation}
(K_{UU})_{ii'}=\lambda(\ba_i,\ba_{i'}),\qquad
(K_{U_*U})_{qi}=\lambda(\ba_*^{(q)},\ba_i),\qquad
K_{UU}^{\theta}=K_{UU}+\theta I,
\label{eq:operator_kernel_matrices}
\end{equation}
with $\theta$ the regularization of~\eqref{eq:operator-interpolant-plain-conditions}, and the affine matrices $P_U$ and $P_{U_*}$ with rows $(P_U)_{i,:}=p(\ba_i)^{\mathsf T}$ and $(P_{U_*})_{q,:}=p(\ba_*^{(q)})^{\mathsf T}$. We keep every nonzero entry of $K_{UU}$, with no neighbors discarded by an approximate search. Its density $\operatorname{nnz}(K_{UU})/N^2$, where $\operatorname{nnz}$ counts the nonzero entries, is thus set by $\varrho$ and the data, and the storage of its Cholesky factor also depends on the fill.

Since the kernel $\lambda(\ba,\ba')I$ does not couple different output coefficients, the conditions decouple into one system per stencil. On stencil $j$, let $B_j\in\mathbb R^{N\times r_j}$ hold the $\bb_i^j$ as rows, and let $C_j\in\mathbb R^{N\times r_j}$ and $\beta_j\in\mathbb R^{n_{\rm aff}\times r_j}$ hold the corresponding entries of the $\bc_i$ and the columns of $\boldsymbol\beta$. Evaluating~\eqref{eq:operator-interpolant} at the training inputs turns~\eqref{eq:operator-interpolant-conditions} into the first equation of
\begin{equation}
\begin{bmatrix}K_{UU}^{\theta}&P_U\\P_U^{\mathsf T}&0\end{bmatrix}
\begin{bmatrix}C_j\\\beta_j\end{bmatrix}
=\begin{bmatrix}B_j\\0\end{bmatrix},\qquad
B_{*,j}=K_{U_*U}C_j+P_{U_*}\beta_j,
\label{eq:operator-augmented-primal}
\end{equation}
and evaluating it at the query inputs gives the second, where row $q$ of $B_{*,j}$ holds the coefficients that $\ftilde$ predicts on stencil $j$ for query input $q$. The block $P_U^{\mathsf T}C_j=0$ is the moment condition. Without augmentation, the same steps applied to~\eqref{eq:operator-interpolant-plain} and~\eqref{eq:operator-interpolant-plain-conditions} give
\begin{equation}
K_{UU}^{\theta}C_j=B_j,\qquad
B_{*,j}=K_{U_*U}C_j.
\label{eq:operator_coefficient_solve}
\end{equation}

\subsection{Sparse Cholesky and whitened QR}
\label{sec:operator-factors}
The system~\eqref{eq:operator-augmented-primal} is an indefinite saddle-point system. Forming the Schur complement $P_U^{\mathsf T}(K_{UU}^{\theta})^{-1}P_U$ squares the condition number of the whitened polynomial columns, as in the normal equations of least squares. We instead use the Cholesky--QR elimination of our unified kernel and polynomial interpolation~\citep{belianovich2026unified}. Since $K_{UU}^{\theta}$ is sparse and positive definite, we factor it by sparse Cholesky, whiten the polynomial columns, and factor them by thin QR. Suppressing the fill-reducing permutation, we obtain
\begin{equation}
K_{UU}^{\theta}=L_KL_K^{\mathsf T},\qquad
L_K^{-1}P_U=Q_PR_P,\qquad
\mathcal E_U=R_P^{-1}Q_P^{\mathsf T}L_K^{-1}.
\label{eq:operator-whitened-qr}
\end{equation}
Here $Q_P$ has orthonormal columns and $R_P$ is a nonsingular $n_{\rm aff}\times n_{\rm aff}$ triangular matrix, and we apply $\mathcal E_U$ by triangular solves without forming $L_K^{-1}$. Block elimination of~\eqref{eq:operator-augmented-primal} then gives
\begin{equation}
\beta_j=\mathcal E_UB_j,\qquad
K_{UU}^{\theta}C_j=B_j-P_U\beta_j.
\label{eq:operator-primal-elimination}
\end{equation}
Without augmentation, $Q_P$, $R_P$, and $\mathcal E_U$ are empty and the correction disappears. The same factors serve the parameter selection below and the prediction of Section~\ref{sec:lagrange}, which never forms the $C_j$.

The size of the Cholesky factor, and the speed of its triangular solves, depend on the ordering of the training inputs. Since two inputs interact only if their principal coordinates are close, the sparsity pattern of $K_{UU}$ resembles that of a discretization on a mesh in at most three dimensions, and we order it by nested dissection. For inputs of numerical rank one ($n_z=1$), sorting the inputs by $z$ even makes $K_{UU}$ a band matrix (the matrix of an interval graph), whose Cholesky factor has no entries outside the band. However, the elimination of a band matrix is a single chain, which leaves the triangular solves of a prediction little parallelism on the GPU, while nested dissection gives a balanced elimination tree. Thus, nested dissection is our default, and the band ordering is preferable when memory is more limited than prediction time (Section~\ref{sec:gpu-sparsity}). The ordering changes neither the operator regression nor its predictions beyond rounding.

\subsection{Choosing the kernel parameters}
\label{sec:kernel-training}
The support radius $\varrho$, the regularization $\theta$, and, for the fitted-shape kernel, the shape parameters remain to be chosen. The support radius is the most important of these, since it trades accuracy against sparsity. We use two criteria in turn. For the fitted-shape kernel, a restricted likelihood criterion fits the shape and the regularization from the training data alone. The error on validation fields then selects the radius and refines the regularization.

Let $F\in\mathbb R^{N\times n_v}$ hold all sampled training outputs, with $n_v=dm$ values per field, or $dmN_t$ for spacetime outputs, after the amplitude normalization~\eqref{eq:tg-response-transform} where it applies, let $Q_0$ be an orthonormal basis of the range of $P_U$, $E_0=(I-Q_0Q_0^{\mathsf T})F$, and $E=\sqrt N E_0/\|E_0\|_F$, and let
\begin{equation}
\Pi_U=(K_{UU}^{\theta})^{-1}
-(K_{UU}^{\theta})^{-1}P_U
\left[P_U^{\mathsf T}(K_{UU}^{\theta})^{-1}P_U\right]^{-1}
P_U^{\mathsf T}(K_{UU}^{\theta})^{-1},
\label{eq:residual-precision}
\end{equation}
which we apply by sparse solves and never store. With $n_{\rm res}=N-n_{\rm aff}$, we minimize
\begin{align}
\mathcal J_{c_d}={}&\frac12\left[
\frac{\log\det K_{UU}^{\theta}+2\log|\det R_P|}{n_{\rm res}}
+\log\!\left(\frac{\operatorname{tr}(E^{\mathsf T}\Pi_U E)}{n_{\rm res}}\right)\right]
+c_d\widetilde d,\label{eq:reml-objective}\\
\widetilde d={}&\frac1{N(N-1)}\sum_{i\ne i'}
\operatorname{sigmoid}\!\left(
\frac{1-\|D_\parallel(z_i-z_{i'})\|_2/\varrho}{0.03}\right).
\label{eq:soft-density}
\end{align}
The bracket can be interpreted in two ways. In kernel terms, $\operatorname{tr}(E^{\mathsf T}\Pi_UE)$ is the squared norm, in the native space of the regularized kernel, of the augmented interpolants of the training outputs after their affine part is removed, and the log-determinants penalize kernels whose matrices have a large determinant, which are the more flexible ones. The bracket thus favors kernels that represent the outputs with a small norm, while penalizing overly flexible kernels. In probabilistic terms, if the outputs are modeled as samples of a Gaussian process with covariance $\lambda$, independent noise of variance $\theta$, and an affine mean with unknown coefficients, the bracket is, up to constants, the negative log restricted likelihood with the variance profiled out (with $\det R_P=1$ without augmentation). We borrow the criterion from this setting, but use it only to select kernel parameters, not for uncertainty quantification.

The term $\widetilde d$, with $z_i=z(\ba_i)$ and the logistic function $\operatorname{sigmoid}(x)=1/(1+e^{-x})$, is a smooth count of the interacting pairs, and its weight $c_d\in\{0,1\}$ trades accuracy against sparsity. We fit one candidate with each weight. All terms of $\mathcal J_{c_d}$ come from the sparse factors (Appendix~\ref{app:cost-details}).

The criterion measures how well the kernel represents the training outputs, but our goal is the error of the predicted velocity for new inputs. Thus, we discard the radius that minimizes the likelihood criterion and choose the radius by validation alone. We first evaluate the five radii that give the densities $0.10$, $0.25$, $0.50$, $0.75$, and $0.90$, with the regularization from the likelihood. We then take gradient steps in the radius and the regularization on the mean relative velocity error on validation fields, from four starts (half and a quarter of the largest of these radii, each with $\theta=10^{-6}$ and $10^{-3}$), while the fitted shape stays unchanged. Among all candidates with density at most $0.90$, we keep the one with the smallest validation error. The fixed-shape kernel goes through the same search without the likelihood step, with $\theta=10^{-6}$ at the five radii. All choices are made before the test fields are used (Appendix~\ref{app:gpu-selection}).

We thus compare four models, the fixed-shape and the fitted-shape kernel, each with and without affine augmentation. A comparison at the same augmentation isolates the effect of fitting the shape. Since all parameters are refit when the augmentation changes, a comparison with and without it measures the whole procedure.

%% file: sections/lagrange.tex
\section{Operator Lagrange prediction}
\label{sec:training-costs}
\label{sec:lagrange-section}
We now combine the operator regression of Section~\ref{sec:operator} with the output reconstruction of Section~\ref{sec:methods}. The direct way to predict is to solve~\eqref{eq:operator-augmented-primal} for the coefficients $C_j$ and $\beta_j$ of every stencil. This needs one operator right-hand side for every spatial coefficient, and the local coefficients of every training field. However, an operator regression with the same kernel for all outputs is a weighted combination of the training outputs, and the weights depend only on the new input. It suffices to solve once per new input for these weights and apply them to the whole output at once. Since the output reconstruction is linear, we can apply them even before the output reconstruction, to the sampled training fields.

\subsection{Operator Lagrange weights}
\label{sec:lagrange}
We first define the weights as functions on the input space. For an input $\ba$, let $k_U(\ba)=[\lambda(\ba_i,\ba)]_{i=1}^N$. We define its operator Lagrange weights $\ell(\ba)\in\mathbb R^N$ by
\begin{equation}
\begin{bmatrix}K_{UU}^{\theta}&P_U\\P_U^{\mathsf T}&0\end{bmatrix}
\begin{bmatrix}\ell(\ba)\\\mu(\ba)\end{bmatrix}
=\begin{bmatrix}k_U(\ba)\\p(\ba)\end{bmatrix},
\label{eq:augmented-lagrange-system}
\end{equation}
where the multiplier $\mu(\ba)$ enforces the reproduction of affine functions. Block elimination with the factors~\eqref{eq:operator-whitened-qr} gives
\begin{equation}
\ell(\ba)=\bigl(I-\mathcal E_U^{\mathsf T}P_U^{\mathsf T}\bigr)(K_{UU}^{\theta})^{-1}k_U(\ba)+\mathcal E_U^{\mathsf T}p(\ba),
\label{eq:lagrange-continuous}
\end{equation}
so every weight $\ell_i$ is a combination of the kernel functions $\lambda(\cdot,\ba_{i'})$ and the affine functions, of the same form as~\eqref{eq:operator-interpolant}. Since the matrix of~\eqref{eq:augmented-lagrange-system} is the symmetric matrix of~\eqref{eq:operator-augmented-primal}, the operator interpolant~\eqref{eq:operator-interpolant} and the learned operator~\eqref{eq:operator-continuous} take the Lagrange form
\begin{equation}
\ftilde(\ba)=\sum_{i=1}^N\ell_i(\ba)\bb_i,\qquad
\Gtilde(a)=\sum_{i=1}^N\ell_i(\ba)\chi(\bb_i),
\label{eq:lagrange-form}
\end{equation}
a combination of the reconstructed training fields with weights that depend only on the input (Appendix~\ref{app:augmented-lagrange-proof}). Without augmentation, $\ell(\ba)=(K_{UU}^{\theta})^{-1}k_U(\ba)$, and~\eqref{eq:lagrange-form} is the Lagrange form of~\eqref{eq:operator-interpolant-plain} and~\eqref{eq:operator-continuous-plain}. The weights are regularized Lagrange (cardinal) functions of the training inputs~\citep{santin2018greedy,Wendland2004,fasshauer2007meshfree}. For $\theta=0$, they are cardinal at the training inputs. For $\theta>0$, they are not, but they still reproduce the affine functions exactly.

\begin{proposition}[Reproduction and preservation under shared weights]
\label{prop:operator-reproduction}
Assume $K_{UU}^{\theta}$ is positive definite and $P_U$ has full column rank. Then~\eqref{eq:lagrange-continuous} is the unique weight vector in~\eqref{eq:augmented-lagrange-system}, and $P_U^{\mathsf T}\ell(\ba)=p(\ba)$. In particular, the augmented model reproduces constants and the scaled principal coordinates in $p(\ba)$. Let $\chi$ be any common linear output reconstruction of the training coefficient vectors, and let $\mathcal D$ be a linear constraint with $\mathcal D\chi(\bb_i)=0$ for every training field. Then
\begin{equation}
\chi\!\left(\sum_i\ell_i(\ba)\bb_i\right)
=\sum_i\ell_i(\ba)\chi(\bb_i),\qquad
\mathcal D\chi\!\left(\sum_i\ell_i(\ba)\bb_i\right)=0.
\label{eq:operator-property-preservation}
\end{equation}
\end{proposition}
\begin{proof}[Proof sketch]
The factors~\eqref{eq:operator-whitened-qr} give $\mathcal E_UP_U=I$, from which reproduction follows by substitution, and uniqueness follows from the block structure. The identities~\eqref{eq:operator-property-preservation} are linearity. Appendix~\ref{app:augmented-lagrange-proof} gives the details.
\end{proof}
With $\mathcal D$ the divergence on a stencil, the proposition shows that every prediction is divergence free on each stencil, for both kernels and with or without augmentation. Since the temporal weights~\eqref{eq:temporal_polynomial_weights} do not depend on the input, the operator weights also commute with interpolation in time.

We now turn to the discrete form. For a new input $\ba_*$, restricting the first identity of~\eqref{eq:lagrange-form} to stencil $j$ gives the predicted coefficients
\begin{equation}
\bb_*^j=B_j^{\mathsf T}\ell(\ba_*),
\label{eq:operator_coefficient_prediction}
\end{equation}
and, with the weights of $N_*$ new inputs as the rows of $L\in\mathbb R^{N_*\times N}$, the second equation of~\eqref{eq:operator-augmented-primal} becomes
\begin{equation}
B_{*,j}=LB_j.
\label{eq:lagrange_map}
\end{equation}
Figure~\ref{fig:pipeline} illustrates~\eqref{eq:operator_coefficient_prediction}. One sparse solve gives the weights of a new input, and these weights produce the coefficients on every stencil. The number of operator solves now equals the number of new inputs and no longer depends on the size of the output representation. To compute~\eqref{eq:lagrange-continuous}, we first solve $K_{UU}^{\theta}\ell_0(\ba)=k_U(\ba)$ and then correct for the affine augmentation,
\begin{equation}
\ell(\ba)=\ell_0(\ba)+\mathcal E_U^{\mathsf T}
\left[p(\ba)-P_U^{\mathsf T}\ell_0(\ba)\right].
\label{eq:augmented-lagrange}
\end{equation}
Without augmentation, $L=K_{U_*U}(K_{UU}^{\theta})^{-1}$, so that
\begin{equation}
K_{UU}^{\theta}L^{\mathsf T}=K_{U_*U}^{\mathsf T}.
\label{eq:lagrange_solve}
\end{equation}

\subsection{Applying the weights before the output reconstruction}
\label{sec:commuted-prediction}
With $\bb_i$ as in Section~\ref{sec:operator-interpolant}, let $\mathcal R$ be the linear map from a sampled output $\bv_i$ to its local output reconstruction $\chi(\bb_i)$. Since the same stencils, polynomial bases, and temporal weights serve every training field, linearity gives
\begin{equation}
\begin{aligned}
\Gtilde(a)
&=\chi\!\left(\sum_{i=1}^N\ell_i(\ba)\bb_i\right)
=\sum_{i=1}^N\ell_i(\ba)\chi(\bb_i)\\
&=\mathcal R\!\left(\sum_{i=1}^N\ell_i(\ba)\bv_i\right).
\label{eq:reconstruction-commutation}
\end{aligned}
\end{equation}
Our GPU implementation uses the last form. It combines the sampled training fields with the weights and then applies $\mathcal R$ once per new input. At setup, adjoint solves with the local matrices $\bA_j$ turn $\mathcal R$ into compact evaluation weights on each stencil. Prediction then stores the sampled training fields, the evaluation weights, and the stencil indices, but none of the $N$ sets of local coefficients. For the Taylor--Green benchmarks, whose responses we divide by the input amplitude before fitting (Section~\ref{sec:protocol}), the same identity holds for the normalized responses.

\subsection{Rounding errors of direct and Lagrange evaluation}
\label{sec:lagrange-theorem}
The forms in~\eqref{eq:reconstruction-commutation} agree in exact arithmetic, but they round differently. The difference is important when a moderate velocity is the result of cancellation among large spatial coefficients, which can happen for a global interpolant. A solve error then enters the velocity amplified by the output reconstruction. We now quantify this for the model without augmentation and for evaluation through the coefficients, which is the setting of our comparison in Section~\ref{sec:lagrange_accuracy}. The analysis does not cover the affine correction~\eqref{eq:augmented-lagrange} or the order of operations of Section~\ref{sec:commuted-prediction}.

Assume $K_{UU}^{\theta}$ is symmetric and nonsingular. We fix one stencil and drop its index, writing $r=r_j$, $B=B_j\in\mathbb R^{N\times r}$, and $C=C_j\in\mathbb R^{N\times r}$, so that each row of $B$ holds the spatial coefficients of one training field. Let $\bA^\star\in\mathbb R^{n_{\rm eval}\times r}$ stack the evaluation matrices~\eqref{eq:local_dfphs_evaluation} of the $n_{\rm eval}$ velocity values reconstructed from this stencil, so that $B(\bA^\star)^{\mathsf T}$ holds the reconstructed training velocities. The direct coefficients $C$ and the Lagrange weights $L$ satisfy
\begin{equation}
K_{UU}^{\theta}C=B,\qquad LK_{UU}^{\theta}=K_{U_*U},
\end{equation}
and both give the prediction $V_*=K_{U_*U}C(\bA^\star)^{\mathsf T}=LB(\bA^\star)^{\mathsf T}$. The same notation describes a global output reconstruction, with a single stencil.

\Needspace{10\baselineskip}
\begin{theorem}[Solve errors and multiplication rounding]
\label{prop:lagrange-error}
For computed approximations $\widehat C$ and $\widehat L$, define the solve residuals by evaluating their defining equations in exact arithmetic,
\begin{equation}
R_{\rm dir}=B-K_{UU}^{\theta}\widehat C,\qquad
R_{\rm Lag}=K_{U_*U}-\widehat L K_{UU}^{\theta}.
\label{eq:lagrange-residuals}
\end{equation}
Before subsequent multiplication rounding, the prediction errors satisfy
\begin{align}
K_{U_*U}\widehat C (\bA^\star)^{\mathsf T}-V_*&=-L R_{\rm dir} (\bA^\star)^{\mathsf T},\label{eq:direct-error-identity}\\
\widehat L B(\bA^\star)^{\mathsf T}-V_*&=-R_{\rm Lag}(K_{UU}^{\theta})^{-1}B(\bA^\star)^{\mathsf T}.\label{eq:lagrange-error-identity}
\end{align}
Define, with $\operatorname{fl}(\cdot)$ denoting evaluation in floating point, $\widehat V_{*,\rm dir}=\operatorname{fl}(\operatorname{fl}(K_{U_*U}\widehat C)(\bA^\star)^{\mathsf T})$ and $\widehat V_{*,\rm Lag}=\operatorname{fl}(\operatorname{fl}(\widehat L B)(\bA^\star)^{\mathsf T})$.

Suppose the coefficient products and the output reconstruction use standard floating-point dot products, possibly split into $d$ component blocks followed by summation.

Assume no overflow or underflow, and set $\gamma_k=ku_{\rm mach}/(1-ku_{\rm mach})$, where $u_{\rm mach}$ is unit roundoff and $\max\{N,r+d\}u_{\rm mach}<1$. With $\gamma_{\rm mult}=\gamma_N+\gamma_{r+d}+\gamma_N\gamma_{r+d}$, we have
\begin{align}
\|\widehat V_{*,\rm dir}-V_*\|_F
&\leq \|L\|_F\|R_{\rm dir}\|_F\|(\bA^\star)^{\mathsf T}\|_F
 +\gamma_{\rm mult}\big\|\,|K_{U_*U}|\,|\widehat C|\,|(\bA^\star)^{\mathsf T}|\,\big\|_F,
\label{eq:direct-full-error-bound}\\
\|\widehat V_{*,\rm Lag}-V_*\|_F
&\leq \|R_{\rm Lag}\|_F\|(K_{UU}^{\theta})^{-1}B(\bA^\star)^{\mathsf T}\|_F
 +\gamma_{\rm mult}\big\|\,|\widehat L|\,|B|\,|(\bA^\star)^{\mathsf T}|\,\big\|_F.
\label{eq:lagrange-full-error-bound}
\end{align}
Here absolute values are taken entrywise. The residual identities require symmetry and nonsingularity, but not positive definiteness.
\end{theorem}

\begin{proof}[Proof sketch]
Substituting the residuals into the two predictions gives~\eqref{eq:direct-error-identity}--\eqref{eq:lagrange-error-identity}. The standard bound for floating-point dot products then bounds the multiplications. Appendix~\ref{app:lagrange-proof} gives the proof.
\end{proof}

The bounds become clearer with the amplification factor
\begin{equation}
\alpha_{\rm spatial}=\frac{\|B\|_F\|(\bA^\star)^{\mathsf T}\|_F}{\|B(\bA^\star)^{\mathsf T}\|_F}\geq1
\label{eq:spatial-amplification}
\end{equation}
which measures how much the coefficients cancel in the output reconstruction of the training velocities. If $\|R_{\rm dir}\|_F\leq\eta_{\rm dir}\|B\|_F$ and $\|R_{\rm Lag}\|_F\leq\eta_{\rm Lag}\|K_{U_*U}\|_F$, the solve errors satisfy
\begin{subequations}
\label{eq:solve-amplification}
\begin{align}
\|K_{U_*U}\widehat C(\bA^\star)^{\mathsf T}-V_*\|_F
&\leq\eta_{\rm dir}\|L\|_F\alpha_{\rm spatial}\|B(\bA^\star)^{\mathsf T}\|_F,\\
\|\widehat L B(\bA^\star)^{\mathsf T}-V_*\|_F
&\leq\eta_{\rm Lag}\|K_{U_*U}\|_F\|(K_{UU}^{\theta})^{-1}\|_F\|B(\bA^\star)^{\mathsf T}\|_F.
\end{align}
\end{subequations}
In direct evaluation, the relative solve error $\eta_{\rm dir}$ is amplified by $\alpha_{\rm spatial}$. In Lagrange evaluation, the coefficients have already been combined into velocities, and $\alpha_{\rm spatial}$ does not appear. A global interpolant with a smooth kernel can have very large $\alpha_{\rm spatial}$, since its coefficients are large and cancel, while a local interpolant on a small stencil has moderate coefficients. The rounding of the multiplications, however, contains the entrywise product with $|(\bA^\star)^{\mathsf T}|$ in both bounds (of $|\widehat C|$ in direct and of $|B|$ in Lagrange evaluation), so cancellation in the output reconstruction also enters the rounding, and the theorem does not make Lagrange evaluation more accurate in general. We test these predictions in Section~\ref{sec:lagrange_accuracy}.

\subsection{GPU implementation}
\label{sec:algorithm}
Algorithm~\ref{alg:sappkm} summarizes our implementation. Every floating-point computation runs on the GPU in double precision, from the normalization of the inputs and the stencil selection to the factorizations and the output reconstruction, using the routines of KernelPack-JAX~\citep{cattaneo2026kernelpack}. Only the integer bookkeeping of the cell-based stencil search for large non-periodic node sets (Section~\ref{sec:spatial-cost}) runs on the host. The kernel is assembled pair by pair over its support, in passes of bounded size, so that no dense kernel or distance matrix is ever formed.

However, JAX has no sparse direct solver that keeps a factorization between calls. Our implementation therefore factors $K_{UU}^{\theta}$ with NVIDIA's cuDSS through the primitive \texttt{kernelpack.sparse.cudss} of KernelPack-JAX, which exposes cuDSS to XLA, the compiler behind JAX. We use cuDSS in its symmetric positive-definite mode, with its nested-dissection ordering (Section~\ref{sec:operator-factors}). The factors stay on the GPU, and the right-hand sides of a batch of new inputs are solved as the columns of one call. A change of the support (i.e., of $\varrho$) requires a new ordering and symbolic analysis, while a change of the values only requires a numerical refactorization. Appendix~\ref{app:query-batches} describes the batching and the checks of the timed runs.

\begin{algorithm}[!htbp]
\caption{Local and sparse kernel operator learning on the GPU.}
\label{alg:sappkm}
\begin{algorithmic}[1]
\Require training and validation pairs, new inputs, output points, spatial kernel and polynomial degree
\AlgLabel{Output reconstruction}
\State select the stencils, with periodic images where needed (Section~\ref{sec:periodic_stencils})
\State build the divergence-free polynomials and the local systems, and prepare $\mathcal R$ (Section~\ref{sec:local_dfphs})
\State for spacetime outputs, add the temporal weights~\eqref{eq:temporal_polynomial_weights}
\AlgLabel{Operator regression}
\State standardize the training inputs and compute their principal directions
\State for each choice of augmentation, keep the fixed shape or fit it by~\eqref{eq:reml-objective}
\State select the radius and regularization on the validation fields
\State assemble $K_{UU}^{\theta}$ in sparse form and factor it with cuDSS
\State with augmentation, form the whitened QR correction~\eqref{eq:operator-whitened-qr}
\AlgLabel{Prediction}
\State solve for the Lagrange weights~\eqref{eq:augmented-lagrange} of a batch of new inputs
\State combine the sampled training fields and apply $\mathcal R$ by~\eqref{eq:reconstruction-commutation}
\State check the solve residuals and the reproduction of affine functions
\end{algorithmic}
\end{algorithm}

\subsection{Work and storage}
\label{sec:complexity}
Table~\ref{tab:complexity} lists the work and storage of each step. Let $\mathcal N_K=\operatorname{nnz}(K_{UU})$ and $\mathcal N_{\rm fac}=\operatorname{nnz}(L_K)$, where $L_K$ is the Cholesky factor of $K_{UU}^{\theta}$ in~\eqref{eq:operator-whitened-qr}. With $\mathcal L_i$ the set of rows in which column $i$ of $L_K$ is nonzero, so that $\sum_{i=1}^N|\mathcal L_i|=\mathcal N_{\rm fac}$, the sparse factorization takes, up to a constant factor,
\begin{equation}
F_{\rm chol}=\sum_{i=1}^N|\mathcal L_i|^2,\qquad
\frac{\mathcal N_{\rm fac}^2}{N}\leq F_{\rm chol}\leq N\mathcal N_{\rm fac},
\label{eq:cholesky-work}
\end{equation}
floating-point operations~\citep{davis2006direct}, where the lower bound follows from the Cauchy--Schwarz inequality and the upper bound from $|\mathcal L_i|\leq N$. For a dense factor, $|\mathcal L_i|=N-i+1$ and $F_{\rm chol}\approx N^3/3$. Let $b_q$ be the number of new inputs in a batch, let $m_{\rm rhs}$ be the number of spatial coefficients, and let $n_v$ be the number of sampled output values per field ($n_v=dm$, or $dmN_t$ for spacetime outputs). We count the work of a triangular solve as proportional to the number of nonzeros of the factor.

\begin{table}[!htbp]
\centering
\caption{Work and storage. Here $\mathcal N_K$ and $\mathcal N_{\rm fac}$ are the numbers of nonzeros of the kernel matrix and of its Cholesky factor, $F_{\rm chol}$ the work~\eqref{eq:cholesky-work} of the factorization, $b_q$ the number of new inputs in a batch, $n_{\rm aff}$ the number of affine columns, $n_v$ the number of sampled output values per field, $m^\star$ the number of query points, $r_j$ the size of the local system on stencil $j$, and $m_{\rm rhs}=\sum_j r_j$. The last two rows belong to direct evaluation through coefficients, which our predictions avoid. For spacetime outputs, the counts of these two rows are multiplied by $N_t$.}
\label{tab:complexity}
\small
\begin{tabular}{p{0.27\linewidth}p{0.36\linewidth}p{0.29\linewidth}}
\toprule
Step & Work & Storage\\
\midrule
Kernel factorization & $F_{\rm chol}$ & $\mathcal N_K$, $\mathcal N_{\rm fac}$\\
Affine correction ($n_{\rm aff}$ columns) & solves with $n_{\rm aff}$ right-hand sides, $O(Nn_{\rm aff}^2)$ QR & $O(Nn_{\rm aff})$\\
Lagrange weights, $b_q$ queries & $O(b_q\mathcal N_{\rm fac})+O(b_qNn_{\rm aff})$ & $b_qN$ (dense)\\
Combining training fields & $O(b_qNn_v)$ & $Nn_v$\\
Local systems (setup) & $O(r_j^3)$ per stencil & \\
Local evaluation, $m^\star$ points & $O(b_qm^\star d^2n_s)$ per snapshot & $O(m^\star d^2n_s)$ weights, $O(m^\star n_s)$ indices\\
\midrule
Training coefficients & $O(N\sum_jr_j^2)$ & $N\sum_jr_j$\\
Applying weights to coefficients & $O(b_qNm_{\rm rhs})$ & \\
\bottomrule
\end{tabular}
\end{table}

Two observations follow from the table. First, the local output reconstruction costs $O(r_j^3)$ per stencil, with $r_j=dn_s+n_{P,j}$ fixed by the stencil size and the polynomial degree, so its cost grows linearly with the number of output points instead of cubically. Second, the cost of the operator regression is governed by the factor of $K_{UU}^{\theta}$, whose size depends on the fill as well as on the density of $K_{UU}$~\citep{davis2006direct}. At a fixed density, the kernel matrix and its factor still grow as $N^2$. For this reason, we report density, factor storage, and prediction time separately in Section~\ref{sec:results}.

Note that the Lagrange weights are dense even when $K_{UU}$ is sparse. Direct evaluation instead solves $m_{\rm rhs}$ right-hand sides once and can reuse its coefficients for any number of new inputs, so the cheaper route depends on how many predictions are needed. Appendix~\ref{app:cost-details} gives the storage accounting of our experiments.

%% file: sections/results.tex
\FloatBarrier
\section{Numerical experiments}
\label{sec:results}

We now test our method in three steps. We first test the local output reconstruction alone, with exact fields in place of predictions, for convergence and for the three properties it preserves (Section~\ref{sec:spatial-convergence}). We then compare four operator models on ten fluid benchmarks as the training set grows, in accuracy, sparsity, and storage (Section~\ref{sec:regression-results}). Finally, we compare direct and Lagrange evaluation in floating point and measure the cost of prediction (Section~\ref{sec:lagrange_accuracy}).

\input{sections/spatial_convergence}

\FloatBarrier
\subsection{Sparse operator regression}
\label{sec:regression-results}

\subsubsection{Benchmarks and models}
\label{sec:protocol}
We use the ten data sets of our previous work~\citep{sharma2026fluids}, listed in Table~\ref{tab:gpu-benchmarks}. Eight are two-dimensional laminar flows, two of them the analytical Taylor--Green vortices, and two are three-dimensional turbulent flows, a Reynolds-averaged simulation of a static mixer and a large-eddy simulation of forced isotropic turbulence. In each, the input is an initial velocity, temperature, or vorticity, an inflow, or a body force, sampled at points of the domain or of its inflow boundary, and the output is the velocity at a fixed later time (or, for Taylor--Green vortices (spacetime), at four times), sampled at $m$ output points. The output points are approximate Fekete points, i.e., well-spaced points selected from the mesh~\citep{sharma2026fluids}. Appendix~\ref{app:benchmark-configurations} gives the geometries, parameter distributions, and simulations.

Each data set has $10000$ training pairs, $100$ validation pairs, and $100$ test pairs. We fit every model on the first $N\in\{500,1000,5000,10000\}$ training pairs, so the training sets are nested, and the validation and test pairs are the same for every $N$. The validation and test pairs were also used in our previous study and during the development of the method, so our comparisons are retrospective. Freezing all choices before the final test runs prevents test-based selection in this study but does not make the test set new. Our error measure, the test error, is the relative velocity error $\|\tilde\bv^\star-\bv^\star\|_2/\|\bv^\star\|_2$ of the prediction $\tilde\bv^\star$ of a test field $\bv^\star$ at the output points, averaged over the $100$ test fields (and, for Taylor--Green vortices (spacetime), over the four times). It measures the prediction for inputs that are used neither for fitting nor for validation.

\begin{table}[pos=tbp]
\centering\small
\setlength{\tabcolsep}{4pt}
\caption{The ten benchmarks: input and output of the operator, number of sampled input coordinates, number $n_z$ of principal input coordinates (the numerical rank of the training inputs, capped at three), spatial dimension $d$, number $m$ of output points, and spatial kernel. Taylor--Green vortices (spacetime) has spacetime outputs, at the same $500$ points at four times.}
\label{tab:gpu-benchmarks}
\begin{tabular}{llrrrrl}
\toprule
Data set & Input $\to$ output & Input coord. & $n_z$ & $d$ & $m$ & Spatial kernel \\
\midrule
Forced isotropic turbulence & $\mathbf f\to\bu(\cdot,0.5)$ & 21000 & 3 & 3 & 7000 & Five-scale Mat\'ern \\
Turbulent species transport & inflow $\to\bu(\cdot,0.5)$ & 1388 & 2 & 3 & 7000 & Five-scale Mat\'ern \\
Merging vortices & $\omega(\cdot,0)\to\bu(\cdot,0.4)$ & 500 & 3 & 2 & 500 & Cubic PHS \\
\midrule
Lid-driven cavity & $\bu(\cdot,0)\to\bu(\cdot,5)$ & 2000 & 1 & 2 & 1000 & Cubic PHS \\
Cylinder flow (no shedding) & $\bu(\cdot,0)\to\bu(\cdot,10)$ & 2000 & 1 & 2 & 1000 & Cubic PHS \\
Cylinder flow (vortex shedding) & $\bu(\cdot,0)\to\bu(\cdot,10)$ & 2000 & 1 & 2 & 1000 & Cubic PHS \\
Backward-facing step & inflow $\to\bu(\cdot,5)$ & 11 & 1 & 2 & 1000 & Cubic PHS \\
Buoyancy-driven cavity & $T(\cdot,0)\to\bu(\cdot,5)$ & 5000 & 1 & 2 & 5000 & Cubic PHS \\
\midrule
Taylor--Green vortices (space) & $\bu(\cdot,0)\to\bu(\cdot,1)$ & 1000 & 1 & 2 & 500 & Cubic PHS \\
Taylor--Green vortices (spacetime) & $\bu(\cdot,0)\to\bu(\cdot,t)$, $t\in[0.7,1]$ & 1000 & 1 & 2 & 500 & Cubic PHS \\
\bottomrule
\end{tabular}
\end{table}

The benchmarks differ less in their flows than in their inputs, and this affects how we interpret the results. In three of them, forced isotropic turbulence, turbulent species transport, and merging vortices, the inputs vary in at least two independent directions, and the operator is a multivariate map. In the other seven, the training inputs have numerical rank one. These inputs are multiples of one input function (e.g., the speed of the lid or the amplitude of the inflow), so the operator restricted to the data is a curve in the output space, and the operator regression is one-dimensional. Two of these, the Taylor--Green vortices, are degenerate in a further sense. Their outputs are proportional to the input amplitude, so we divide the responses by the amplitude before fitting and multiply by it after prediction. With $\mathcal A_i>0$ the amplitude of training input $i$ relative to the median training amplitude and $\mathcal A(\ba)$ that of a new input, we fit the normalized responses $\bv_i/\mathcal A_i$ and predict
\begin{equation}
\tilde\bv(\ba)=\mathcal A(\ba)\sum_{i=1}^N\ell_i(\ba)\frac{\bv_i}{\mathcal A_i}.
\label{eq:tg-response-transform}
\end{equation}
The normalized operator is then constant, which affine augmentation reproduces, and Proposition~\ref{prop:operator-reproduction} applies to the normalized responses. All four models use this normalization for both Taylor--Green data sets, and the other eight data sets use unnormalized responses. Taylor--Green vortices (spacetime) thus tests our spacetime output reconstruction but not the operator regression.

We compare four models (Section~\ref{sec:kernel-training}). The \emph{fixed-shape} kernel uses $D_\parallel=I$ and $\alpha_\perp=1$ in~\eqref{eq:dimension-safe-kernel}, an isotropic metric on the principal coordinates and a unit scale for the residual coordinates. The \emph{fitted-shape} kernel has its shape fitted by the restricted likelihood. Each is used with and without affine augmentation, and each chooses its support radius and regularization by validation. For inputs of rank one, the shape has no free parameters, so the fitted-shape kernel differs from the fixed-shape one only through the regularization that the likelihood proposes, which enters the validation search.

All four models share one local output reconstruction, with divergence-free polynomials of degree $\varpi=4$ (dimension $20$ in 2D and $85$ in 3D) on stencils of $n_s=21$ and $59$ nodes, the cubic PHS for the smooth flows, and the multiscale Mat\'ern kernel with $c_\sigma=16$ (which sets its lengths $\sigma_{j,\kappa}=c_\sigma\rho_j/2^\kappa$ relative to the stencil radius $\rho_j$ in~\eqref{eq:local_matern_scales}) for the two turbulent flows. Section~\ref{sec:spatial-ranks} compares these choices with alternatives. All computations run in double precision on one NVIDIA RTX 5090.

\FloatBarrier
\subsubsection{Sparsity and convergence in $N$}
\label{sec:gpu-accuracy}
We first ask whether a sparse kernel matrix is compatible with accurate prediction. Validation chooses the support radius, and with it the density $d_K=\mathrm{nnz}(K_{UU})/N^2$ of the kernel matrix, among candidates with densities up to $90\%$. If accuracy required a dense matrix, validation would select the largest candidates. Figure~\ref{fig:gpu-summary} (right) shows the densities selected at $N=10000$ for all ten benchmarks and all four models of Section~\ref{sec:protocol}. Validation selects $10\%$ for the lid-driven cavity, both cylinder flows, and the buoyancy-driven cavity, about $33\%$ for the backward-facing step, and between $9\%$ and $25\%$ for turbulent species transport and merging vortices. Nearly dense matrices are selected only for forced isotropic turbulence and for Taylor--Green without augmentation. The value $10\%$ is the sparsest of the five prescribed densities that validation evaluates, although its gradient refinement can also reach sparser matrices.

Figure~\ref{fig:gpu-density-main} shows how the selected density changes with $N$. For most benchmarks, it decreases as $N$ grows, from between $20\%$ and $90\%$ at $N=500$ to $10\%$ at $N=10000$ for the lid-driven cavity and both cylinder flows, so the sparsity of the operator regression increases with the size of the training set. Forced isotropic turbulence stays at $90\%$ for every $N$, and the backward-facing step at about $33\%$.

\begin{figure}[pos=tbp]
\centering
\includegraphics{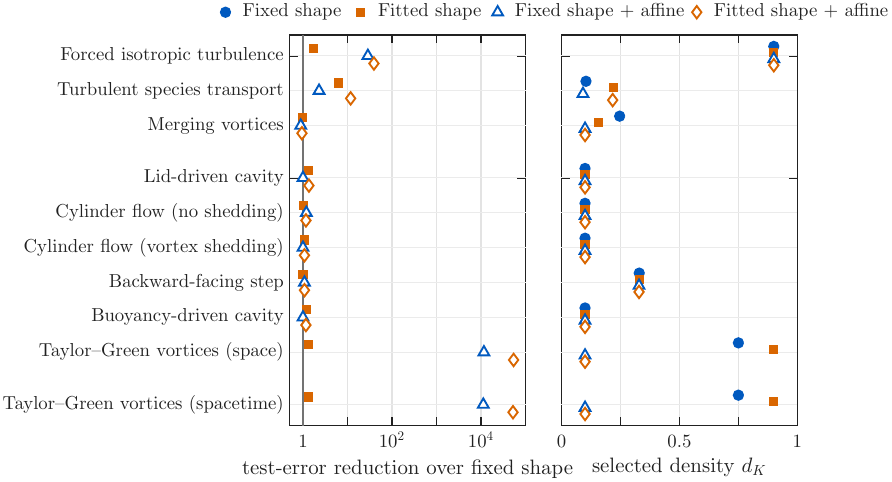}
\caption{All benchmarks at $N=10000$. Left: test error of the fixed-shape kernel ($D_\parallel=I$, $\alpha_\perp=1$) without augmentation, divided by the test error of each other model (values above one are improvements). Right: density $d_K$ of the selected kernel matrix.}
\label{fig:gpu-summary}
\end{figure}

\begin{figure}[pos=tbp]
\centering
\includegraphics{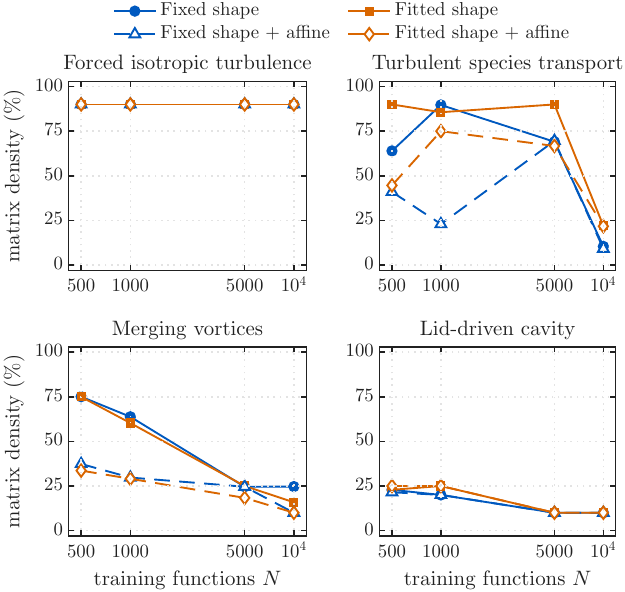}
\caption{Selected density $d_K$ of the kernel matrix against the number $N$ of training functions for the three multivariate flows and the lid-driven cavity. The affine columns are not counted. Figure~\ref{fig:gpu-density-additional} shows the other benchmarks.}
\label{fig:gpu-density-main}
\end{figure}

With these sparse matrices, the test errors decrease steadily as the training set grows. Figure~\ref{fig:gpu-error-main} shows the errors for the three multivariate flows and for the lid-driven cavity, a representative flow with rank-one inputs. Least-squares fits give errors proportional to $N^{-\tau}$ with $\tau$ between $0.4$ and $1.3$ for every benchmark and model (Table~\ref{tab:gpu-decay-rates} in Appendix~\ref{app:gpu-results}), with two exceptions. With augmentation, the backward-facing step decays only like $N^{-0.15}$. For Taylor--Green, the augmented models reproduce the constant normalized operator to rounding error, while the error of the unaugmented models grows slowly with $N$.

\begin{figure}[pos=tbp]
\centering
\includegraphics{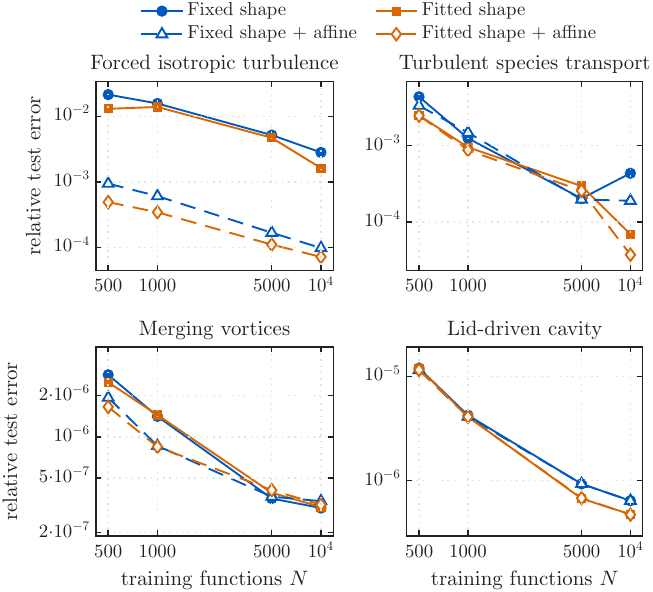}
\caption{Test error (mean relative velocity error of the predictions for the 100 test inputs) against the number $N$ of training functions for the three multivariate flows and the lid-driven cavity. Dashed: with affine augmentation. Figure~\ref{fig:gpu-error-additional} shows the other benchmarks.}
\label{fig:gpu-error-main}
\end{figure}

\subsubsection{Kernel shape and affine augmentation}
\label{sec:gpu-variants}
We next compare the four models. Figure~\ref{fig:gpu-summary} (left) shows, for each benchmark at $N=10000$, the error of the fixed-shape kernel without augmentation divided by the error of each other model. On the five rank-one flows other than Taylor--Green, these ratios lie between $0.98$ and $1.36$, so neither the fitted shape nor the augmentation changes the error by more than about a quarter. This is consistent with the operator regression being one-dimensional on these data.

On the multivariate flows, however, the ratios are larger and depend on the flow. For forced isotropic turbulence, augmentation reduces the error by a factor of $29$, and by $39$ together with the fitted shape, at the same density. For turbulent species transport, the fitted shape reduces it by a factor of $6$, and by $12$ together with augmentation, while doubling the selected density. For merging vortices, all four errors lie within $12\%$ of each other. Augmentation also lowers the selected density for merging vortices and for both Taylor--Green benchmarks. Neither choice is uniformly better, so we select both by validation.

\subsubsection{Comparison with earlier results}
\label{sec:prior-comparison}
We now compare our models with the results that~\citet{sharma2026fluids} report on the same data sets for dense PPKM, for the kernel method of~\citet{batlle2024kernel} (VKM), which maps the samples of the input to the samples of the output, and for two neural operators, Geo-FNO~\citep{li2023fourier} and Transolver~\citep{wu2024transolver}, each trained with and without a divergence penalty. Table~\ref{tab:prior-comparison} lists these errors at the training sizes used there, together with the error of our model that has the smallest validation error among the four. The comparison has two caveats. The earlier errors are averages over all $200$ held-out pairs, which include our $100$ validation pairs, and the neural operators were trained in single precision.

\begin{table}[pos=tbp]
\centering\footnotesize
\setlength{\tabcolsep}{3pt}
\caption{Test error (mean relative velocity error of the predictions for held-out inputs) of our sparse local model (the one of the four with the smallest validation error) and of the methods reported by~\citet{sharma2026fluids} at the same training size $N$: dense PPKM, the kernel method of~\citet{batlle2024kernel} on output samples (VKM), and Geo-FNO and Transolver, each with the better of its errors with and without a divergence penalty. The earlier errors are averages over all $200$ held-out pairs and ours over the $100$ test pairs. Our models of the cylinder flow without shedding at $N=100$ were fitted for this comparison, with the settings of the other training sizes. Bold: smallest test error in each row.}
\label{tab:prior-comparison}
\begin{tabular}{lrrrrrr}
\toprule
Data set & $N$ & Ours & PPKM & VKM & Geo-FNO & Transolver \\
\midrule
Forced isotropic turbulence & 10000 & $\mathbf{7.1\times10^{-5}}$ & $3.6\times10^{-4}$ & $1.6\times10^{-4}$ & $1.5\times10^{-4}$ & $1.2\times10^{-3}$ \\
Turbulent species transport & 10000 & $\mathbf{3.7\times10^{-5}}$ & $2.4\times10^{-4}$ & $2.4\times10^{-4}$ & $5.0\times10^{-4}$ & $2.7\times10^{-3}$ \\
Merging vortices & 500 & $\mathbf{1.9\times10^{-6}}$ & $1.1\times10^{-4}$ & $1.1\times10^{-4}$ & $1.1\times10^{-3}$ & $3.4\times10^{-3}$ \\
\midrule
Lid-driven cavity & 10000 & $\mathbf{4.7\times10^{-7}}$ & $1.4\times10^{-6}$ & $1.5\times10^{-6}$ & $1.5\times10^{-4}$ & $2.1\times10^{-4}$ \\
Cylinder flow (no shedding) & 100 & $\mathbf{2.8\times10^{-5}}$ & $4.0\times10^{-5}$ & $5.3\times10^{-5}$ & $1.6\times10^{-4}$ & $5.7\times10^{-3}$ \\
Cylinder flow (vortex shedding) & 10000 & $\mathbf{3.0\times10^{-7}}$ & $3.1\times10^{-7}$ & $2.3\times10^{-6}$ & $3.7\times10^{-5}$ & $1.5\times10^{-4}$ \\
Backward-facing step & 500 & $\mathbf{9.5\times10^{-8}}$ & $1.5\times10^{-6}$ & $2.9\times10^{-6}$ & $9.3\times10^{-5}$ & $7.2\times10^{-4}$ \\
Buoyancy-driven cavity & 10000 & $\mathbf{4.5\times10^{-6}}$ & $6.3\times10^{-6}$ & $6.1\times10^{-6}$ & $1.4\times10^{-4}$ & $1.9\times10^{-4}$ \\
\midrule
Taylor--Green vortices (space) & 5000 & $\mathbf{6.0\times10^{-11}}$ & $8.8\times10^{-10}$ & $1.7\times10^{-9}$ & $4.2\times10^{-4}$ & $6.1\times10^{-5}$ \\
Taylor--Green vortices (spacetime) & 5000 & $\mathbf{5.4\times10^{-11}}$ & $8.7\times10^{-10}$ & $8.7\times10^{-10}$ & $7.8\times10^{-4}$ & $2.2\times10^{-4}$ \\
\bottomrule
\end{tabular}
\end{table}

Our models are at least as accurate as dense PPKM on all ten data sets, by factors between $1.06$ (cylinder flow with vortex shedding) and about $56$ (merging vortices), and the same holds against VKM. Against the better of the two neural operators, the factor ranges from $2$ for forced isotropic turbulence to more than $10^6$ for Taylor--Green. The gains over dense PPKM are likely due to the affine augmentation, the fitted kernel shape, and the validation of the radius and regularization, and for Taylor--Green also to the amplitude normalization. Thus, at these training sizes, the sparse operator regression is no less accurate than the dense one.

In addition, the neural-operator predictions in~\citet{sharma2026fluids} have maximum pointwise divergences between about $0.1$ and $20$, while ours are divergence free on each stencil. The training times reported there, from seconds to about a minute and a half for the kernel methods with fixed shape parameters and from minutes to almost a day for the neural operators, are not comparable to our selection times, which include the search over radii and regularizations (Appendix~\ref{app:gpu-offline}).

\subsubsection{Storage}
\label{sec:gpu-sparsity}
The storage of a model consists mainly of the Cholesky factor of its kernel matrix and the sampled training fields, which the prediction~\eqref{eq:reconstruction-commutation} uses in place of their local coefficients. Figure~\ref{fig:gpu-storage-main} shows the storage against $N$. Over our range of $N$, the training fields dominate, and the storage of most models grows roughly linearly in $N$, with least-squares exponents between $0.9$ and $1.1$. The exceptions are merging vortices, the backward-facing step, and the unaugmented Taylor--Green models, with exponents between $1.2$ and $1.7$, likely because the factor is a larger share of their storage (small outputs, or densities that do not decrease). The Cholesky factor itself grows like $N^{1.4}$ to $N^{2.4}$, since even a decreasing density leaves a number of nonzeros that grows faster than $N$, and it would eventually dominate for larger training sets.

The size of the factor depends on the ordering as well as on the density. With the nested-dissection ordering of Section~\ref{sec:operator-factors}, the factors at $N=10000$ have between $1$ and $4.3$ times fewer entries (a median of $2.8$) than with the approximate-minimum-degree ordering, with identical predictions. Forced isotropic turbulence gains nothing, since its nearly dense kernel matrix has a dense factor in any ordering. The fill also makes storage not proportional to the density. For example, the fitted-shape augmented species model has $2.4$ times the density of the fixed-shape augmented one, but only $1.4$ times as many factor entries. For this reason, we report storage separately from density (Table~\ref{tab:gpu-cost-n10000}).

\begin{figure}[pos=tbp]
\centering
\includegraphics{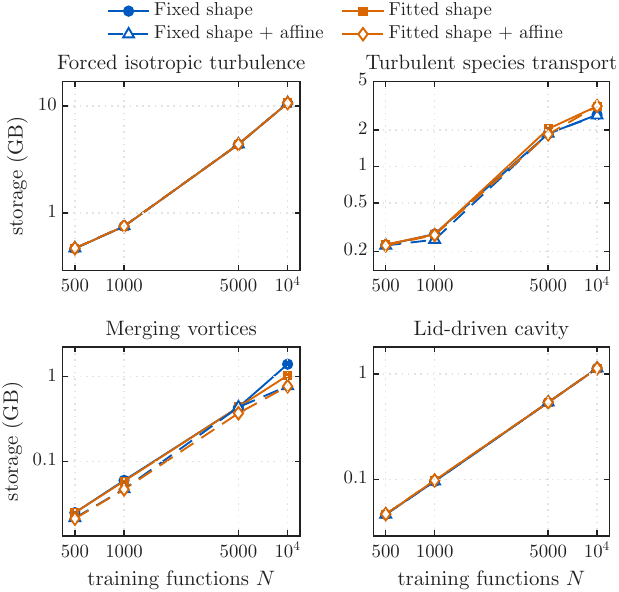}
\caption{Accounted storage of the fitted models (stored arrays and factors, not measured peak memory) against $N$ for the three multivariate flows and the lid-driven cavity. Figure~\ref{fig:gpu-storage-additional} shows the other benchmarks.}
\label{fig:gpu-storage-main}
\end{figure}

\FloatBarrier
\subsection{Operator Lagrange weights}
\label{sec:lagrange_accuracy}
We now turn to the operator Lagrange weights of Section~\ref{sec:lagrange-section}. Their main purpose is computational. They make a prediction cost one sparse solve per new input, regardless of the size of the output representation, and they let us combine the training fields before the output reconstruction. We first measure these savings. We then test the prediction of Theorem~\ref{prop:lagrange-error} that the weights are also more accurate than direct evaluation, but only when the output reconstruction amplifies solve errors.

\subsubsection{Storage of the training data}
Since the weights are applied to the sampled training fields, a prediction needs the $Nn_v$ sampled values, with $n_v=dm$, and not the $N\sum_j r_j$ local coefficients of the training fields. With one stencil per output point, the ratio of the two is $\sum_j r_j/(dm)=n_s+n_P/d$ when every stencil keeps the full polynomial space, of dimension $n_P$. For our degree-four output reconstructions, the local coefficients would thus take $31$ times the storage of the training fields in 2D and $87$ times in 3D. For forced isotropic turbulence at $N=10000$, for example, this is the difference between $1.7$ and $147$ GB in double precision.

\subsubsection{Prediction cost}
\label{sec:gpu-timing}
Figure~\ref{fig:gpu-warm-main} shows the time to predict all $100$ test fields once the model is set up, including the evaluation of the kernel, the sparse solves for the weights, and the local output reconstruction. For the lid-driven cavity, a twentyfold increase in $N$ raises this time seven- to eightfold, so prediction grows sublinearly in $N$ over this range. The turbulent flows cost more at $N=10000$, about three times the time of the lid-driven cavity for forced isotropic turbulence and $1.7$--$2.0$ times for turbulent species transport, since they select denser kernels with more fill.

The ordering also affects the prediction time. Nested dissection makes the median prediction $1.3$ times faster at $N=10000$ than approximate minimum degree, since its smaller factors also give more parallel triangular solves. Fitting the kernel shape has no consistent effect on prediction time, but it adds offline cost, since a fitted-shape model takes several times longer to select than a fixed-shape one (Appendix~\ref{app:gpu-offline}).

\begin{figure}[pos=tbp]
\centering
\includegraphics{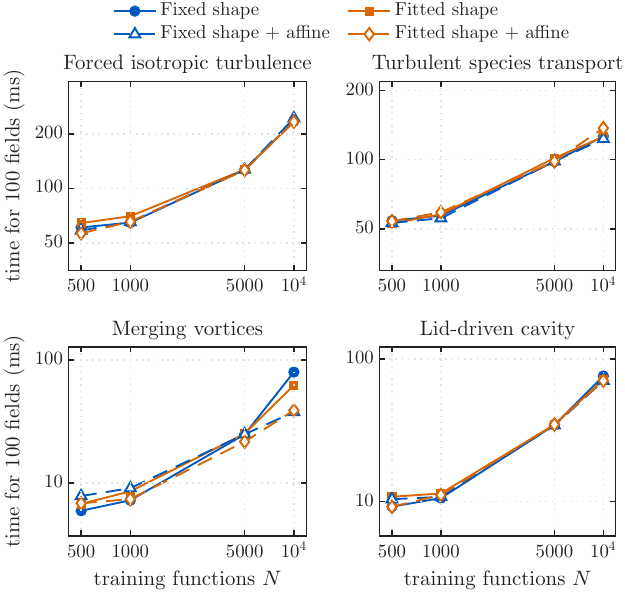}
\caption{Median time to predict all $100$ test fields after setup, including the kernel evaluation, the sparse solves, and the local output reconstruction. Fitting and setup are excluded (Appendix~\ref{app:gpu-offline}).}
\label{fig:gpu-warm-main}
\end{figure}

\subsubsection{Accuracy of direct and Lagrange evaluation}
Recall the amplification factor $\alpha_{\rm spatial}=\|B\|_F\|(\bA^\star)^{\mathsf T}\|_F/\|B(\bA^\star)^{\mathsf T}\|_F$ of~\eqref{eq:spatial-amplification}, which compares the size of the spatial coefficients $B$ of the training fields and of the evaluation matrix $\bA^\star$ with the size of the reconstructed velocities. It is large when large coefficients cancel in the output reconstruction. By~\eqref{eq:solve-amplification}, a solve error enters direct evaluation multiplied by $\alpha_{\rm spatial}$, and Lagrange evaluation without this factor. The theorem therefore predicts a gain from Lagrange evaluation where $\alpha_{\rm spatial}$ is large, and none where it is small. We test this prediction in two steps.

First, we compute $\alpha_{\rm spatial}$ for all ten benchmarks, from $1000$ training fields, for the global divergence-free Mat\'ern output reconstruction of our earlier work and for our local output reconstruction. For the local output reconstruction, which has one $B_j$ and one $\bA_j^\star$ per stencil, we combine the stencils as $(\sum_j\|B_j\|_F^2\|\bA_j^\star\|_F^2)^{1/2}/\|B(\bA^\star)^{\mathsf T}\|_F$, which bounds the amplification of the solve errors summed over all stencils and reduces to $\alpha_{\rm spatial}$ for a single global stencil. Table~\ref{tab:amplification} shows that the coefficients of the global output reconstructions are large and cancel, with $\alpha_{\rm spatial}$ above $10^7$, up to $6\times10^9$, on all benchmarks except Taylor--Green vortices (spacetime), whose kernel is less flat. Our local output reconstructions have $\alpha_{\rm spatial}$ between $6$ and $150$ in 2D, $396$ for forced isotropic turbulence, and about $10^4$ for turbulent species transport, where more than half the stencils have a truncated polynomial rank (Table~\ref{tab:spatial-ranks}).

\input{sections/amplification_table}

Second, we evaluate the same operator models both ways, with everything else held fixed. These are the unaugmented models of Section~\ref{sec:regression-results}, which the theorem covers: the fixed-shape kernel at the radii of the five prescribed densities and at its validated radius, and the fitted-shape kernel at its validated parameters, for $N$ from $500$ to $10000$ (Appendix~\ref{app:lagrange-numerics}). Table~\ref{tab:lagrange-all} lists the ratio of the direct error to the Lagrange error on all ten benchmarks.

\input{sections/lagrange_all_table}

The results agree with this prediction. With global output reconstruction, Lagrange evaluation is more accurate by up to a factor of $868$ (merging vortices), with median gains above $5$ for merging vortices, the lid-driven cavity, and the backward-facing step. With our local output reconstruction, the two evaluations agree to within a relative difference of $3\times10^{-6}$ for every model, as the small $\alpha_{\rm spatial}$ implies. A large $\alpha_{\rm spatial}$ allows a gain but does not guarantee one, since the gain also requires a solve residual that is not negligible. Forced isotropic turbulence, for example, has $\alpha_{\rm spatial}\approx4\times10^9$ with its global output reconstruction but a gain of at most $1.2$.

Over all $560$ pairs, Lagrange evaluation is never less accurate than direct evaluation by more than $0.02\%$. Appendix~\ref{app:lagrange-numerics} follows one global model of the lid-driven cavity as $N$ grows (Table~\ref{tab:lagrange-matched-control}). There, the gain grows from $1.3$ to $5.1$ together with the norm of the spatial coefficients, and solving the direct system for the sampled velocities instead of the coefficients removes it, which indicates that the large coefficients are the cause.

Lagrange evaluation thus serves our local output reconstruction through its cost and storage, and it would also protect a global or otherwise ill-conditioned output reconstruction, such as that of our earlier method, from the amplification of solve errors.

%% file: sections/spatial_convergence.tex
\subsection{Local output reconstruction}
\label{sec:spatial-convergence}

Exact velocity fields in place of predictions separate the error of the output reconstruction from that of the operator regression. We use these tests to answer three questions. Does it converge at the order set by its polynomial degree on scattered nodes? Are its divergence and its periodic mismatch at rounding level? And does the multiscale Mat\'ern kernel reproduce prescribed turbulence statistics as the nodes are refined?

\subsubsection{Convergence on scattered nodes}
We interpolate manufactured divergence-free fields in the smooth domains $\Omega_d=\{F_d>0\}$, $d=2,3$, bounded by a quartic curve and a quartic surface. Each field is the curl of a trigonometric stream function ($d=2$) or vector potential ($d=3$) multiplied by $F_d^2$, so it vanishes on the boundary (Appendix~\ref{app:manufactured-spatial} gives $F_d$ and the potentials). We use divergence-free polynomials of degree $\varpi=1$, $3$, and $5$ for target orders $2$, $4$, and $6$. Each stencil has $n_s=2\lceil n_P/d\rceil+1$ nodes, where $n_P$ is the dimension of the polynomial space, and both the PHS ($r^3$ for degrees $1$ and $3$, and $r^5$ for degree $5$) and the multiscale Mat\'ern kernel ($c_\sigma=16$, which sets the Mat\'ern lengths relative to the stencil radius in~\eqref{eq:local_matern_scales}) use the same stencils and polynomials. The node sets are four refinements generated by the node generator of KernelPack-JAX~\citep{cattaneo2026kernelpack}, with $172$ to $1907$ nodes in 2D and $197$ to $2570$ in 3D, and the stencils are boundary aware (Section~\ref{sec:local_dfphs}).

Let $v_h$ be the output reconstruction of the exact field $v$ from the nodes $Y$ of a node set with spacing $h$. We measure its relative error on a fixed set $Q$ of uniformly random interior evaluation points ($4096$ in 2D and $32768$ in 3D), against the fill distance sampled on the same set,
\begin{align}
 E_{\rm sp}
 &=\left(\frac{\sum_{y\in Q}\|v_h(y)-v(y)\|_2^2}
 {\sum_{y\in Q}\|v(y)\|_2^2}\right)^{1/2},
 & h_Q&=\max_{y\in Q}\min_{y'\in Y}\|y-y'\|_2,
 \label{eq:spatial-study-error}
\end{align}
where each evaluation point uses the stencil of its nearest node.

Figure~\ref{fig:spatial-convergence} shows the results. In 2D, both kernels converge at close to the target orders, with least-squares rates of $2.1$, $3.9$, and $5.8$. In 3D, the rates over all four node sets are lower, $1.7$--$1.9$, $3.4$--$3.7$, and $5.0$, but the sixth-order rate rises to $5.5$ over the three finest node sets. The 3D deficit is likely due to the coarsest node sets, which are not yet in the asymptotic range. The two kernels give nearly the same errors on smooth fields, so the multiscale kernel does not reduce the accuracy on smooth fields.

\begin{figure}[pos=tbp]
\centering
\includegraphics{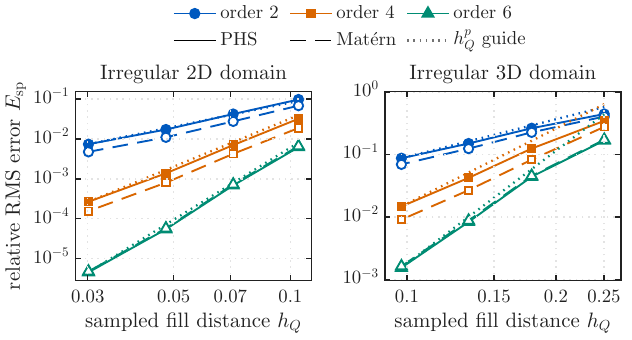}
\caption{Convergence of the local output reconstruction on scattered nodes in the domains $\Omega_2$ (left) and $\Omega_3$ (right), with PHS and multiscale Mat\'ern kernels and divergence-free polynomials. Relative error $E_{\rm sp}$ at random evaluation points that are not nodes, against the sampled fill distance $h_Q$ of~\eqref{eq:spatial-study-error}. Dotted lines: $h_Q^2$, $h_Q^4$, $h_Q^6$.}
\label{fig:spatial-convergence}
\end{figure}

\subsubsection{Divergence and periodicity}
Since every basis function is divergence free, the divergence of the output reconstruction should vanish up to rounding, however large its approximation error. Differentiating each local expansion on its own stencil, we find a normalized root-mean-square (RMS) divergence at rounding level for both kernels, all degrees, and all node sets (Appendix~\ref{app:manufactured-spatial}). We test periodic stencils on scattered nodes in $[0,2\pi)^d$, comparing the reconstructed velocities and their Jacobians at periodically equivalent points on faces, edges, and corners. The differences are again at rounding level for both kernels and all degrees, while the approximation error varies with degree and resolution as expected.

\subsubsection{Kolmogorov statistics}
The multiscale kernel builds a Kolmogorov law into its weights. We now test whether interpolation on finite stencils preserves this law. We reconstruct a three-dimensional solenoidal Fourier ensemble whose velocity increments follow the Kolmogorov law $S_2(r)\propto r^{2/3}$ over the separations $0.03\leq r\leq0.8$. The ensemble is a sum of $1514$ Fourier modes with Gaussian coefficients perpendicular to their wave vectors and a $-5/3$ energy spectrum, restricted to $\Omega_3$ (Appendix~\ref{app:k41-ensemble}). Below $r=0.03$, its finite bandwidth makes the increments smooth. For pairs of points $x_k(r)$ and $y_k(r)$ at distance $r$ along eight fixed directions $k=1,\ldots,8$, we compute the second-order structure function exactly from the covariance of the ensemble,
\begin{align}
 S_{2,k}(r)&=\mathbb E\bigl[\|v(x_k(r))-v(y_k(r))\|_2^2\bigr],
 \qquad \|x_k(r)-y_k(r)\|_2=r,\notag\\
 S_2(r)&=\frac{1}{8}\sum_{k=1}^{8}S_{2,k}(r),
 \label{eq:spatial-k41-s2}
\end{align}
both for the ensemble and, with $v_h$ in place of $v$, for its output reconstruction ($S_{2,k,h}$). We measure the worst relative error over the directions,
\begin{equation}
 E_{S_2}(r)=\max_{1\leq k\leq8}
 \frac{|S_{2,k,h}(r)-S_{2,k}(r)|}{S_{2,k}(r)}.
 \label{eq:spatial-k41-error}
\end{equation}
We use the multiscale Mat\'ern kernel with polynomial degrees $1$, $3$, and $5$ on six node sets, with node spacings $h$ from $0.26$ to $0.05$ ($128$ to $13567$ nodes).

Figure~\ref{fig:spatial-k41} shows the structure functions approaching the reference under refinement. At the largest separation, $r=0.8$, the error decays like $h^{0.64}$, $h^{0.73}$, and $h^{0.67}$ for the three degrees. These are fits over a finite range of $h$, not asymptotic rates. For degree five, the fitted slope of $S_2$ over the upper part of the range moves toward the reference slope. However, the error at small separations stays large even on the finest node set, since these separations are below the node spacing. To see whether locality is to blame, we compare local and global interpolation on one node set, with the same kernel and polynomials (Appendix~\ref{app:k41-ensemble}). Both have nearly the same error, at $r=0.8$ and over the whole range, so the small-separation error is likely due to the resolution of the nodes and not to the stencils.

\begin{figure}[pos=tbp]
\centering
\includegraphics{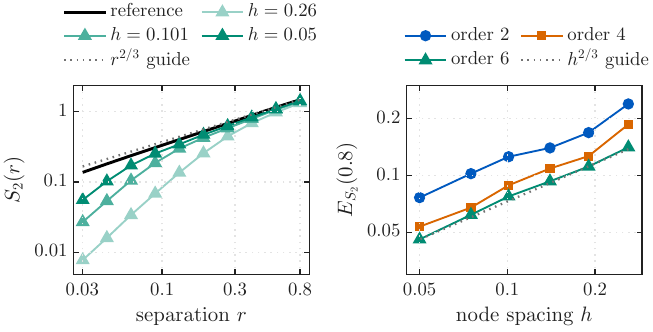}
\caption{Kolmogorov statistics of the local multiscale Mat\'ern output reconstruction. Left: structure functions for degree five on three of the six node sets, and the exact reference. Right: worst relative error $E_{S_2}(0.8)$ of~\eqref{eq:spatial-k41-error} of the reconstructed structure function against the node spacing $h$, with an $h^{2/3}$ guide.}
\label{fig:spatial-k41}
\end{figure}

\subsubsection{Output reconstruction in the fluid benchmarks}
\label{sec:spatial-ranks}
We now test the output reconstruction on the predictions of the operator benchmarks (Section~\ref{sec:protocol}). The data sets hold the velocity at more mesh points than the $m$ output points that we use as nodes, and the remaining points lie between the nodes. We reconstruct the predicted nodal velocities of the model with the smallest validation error at $N=10000$ with the global divergence-free Mat\'ern interpolant of our earlier work and with our local output reconstruction, and measure the error at these points (Table~\ref{tab:spatial-reconstruction-controls}, Appendix~\ref{app:spatial-supporting}). Neither output reconstruction is uniformly more accurate. The global one is more accurate for both cylinder flows, the backward-facing step, and Taylor--Green vortices (space), by factors between $1.06$ and $2.7$. The local one is more accurate for the other six, by factors between $1.04$ (turbulent species transport) and $10$ (buoyancy-driven cavity), although each local system has at most $62$ unknowns in 2D and $262$ in 3D, against up to $21000$ in the global system. For several flows, the errors between the nodes are much larger than those at the nodes, so they are set by the resolution of the $m$ nodes and not by the operator regression. For both turbulent flows, the multiscale Mat\'ern kernel has a lower error than the PHS, and validation selects $c_\sigma=16$ for both. Almost all stencils keep the full polynomial space, except in the species-transport mixer, whose geometry truncates the rank on more than half the stencils (Table~\ref{tab:spatial-ranks} in Appendix~\ref{app:local-solves}).

\input{sections/spatial_controls_table}

\subsubsection{Cost of the local output reconstruction}
\label{sec:spatial-cost}
We now compare the setup cost of our local output reconstruction with that of a global divergence-free output reconstruction, as in PPKM, on scattered nodes in the unit square and the unit cube. The global setup assembles the dense $dm\times dm$ Gram matrix of a divergence-free Mat\'ern kernel and factors it by Cholesky. The local setup selects the stencils, builds the polynomial bases and the local systems, factors them, and forms the evaluation weights of Section~\ref{sec:commuted-prediction}, with the degree-four output reconstruction of the operator benchmarks. Both run on the GPU in double precision, and we time them after compilation. Figure~\ref{fig:spatial-cost} shows the results.

\begin{figure}[pos=tbp]
\centering
\includegraphics{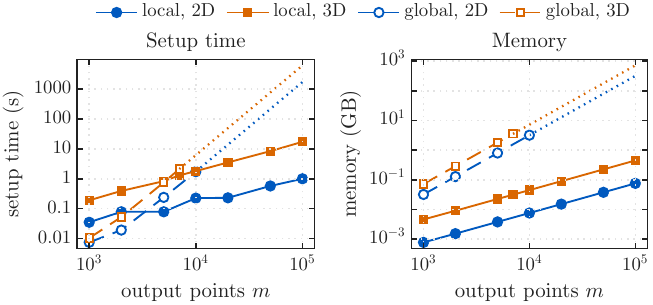}
\caption{Setup time and memory of our local output reconstruction (degree four, cubic PHS) and of a global divergence-free Mat\'ern output reconstruction (dense Gram matrix and Cholesky factor) against the number $m$ of output points, on scattered nodes in the unit square (2D) and cube (3D). Dotted: the global cost extrapolated as $(dm)^3$ in time and $(dm)^2$ in memory beyond the largest size that a single dense factorization fits on our GPU. All times are measured after compilation.}
\label{fig:spatial-cost}
\end{figure}

The local setup grows at most linearly in $m$, in both time and memory, as Table~\ref{tab:complexity} predicts. The global setup grows like $(dm)^3$ in time and $(dm)^2$ in memory, and a single dense factorization no longer fits on our GPU beyond $dm\approx2.5\times10^4$. It could be factored in blocks, at the same cubic cost. On small node sets, the dense global factorization is faster, since its work is done by a few large and efficient GPU kernels. The local setup becomes the faster one near $m=3\times10^3$ in 2D and $m=5\times10^3$ in 3D, and its memory is smaller at every size we tested, by a factor that grows linearly with $m$.

Two choices keep the local setup linear in $m$, and neither changes the output reconstruction beyond rounding. The polynomial bases are built by QR factorizations instead of singular value decompositions (Appendix~\ref{app:local-solves}), and non-periodic node sets of more than $4096$ nodes find their stencils by an exact search on a grid of cells, whose work grows linearly with $m$. Together, they make the setup of the output reconstruction a small part of the offline cost of every operator benchmark (Appendix~\ref{app:gpu-offline}).

%% file: sections/spatial_controls_table.tex
\begin{table}[pos=tbp]
\centering\small
\setlength{\tabcolsep}{4pt}
\caption{Test error between the nodes. We reconstruct the predicted nodal velocities of the 100 test inputs (model with the smallest validation error at $N=10000$) and measure the mean relative velocity error at the evaluation points, the mesh points of each data set that are not among its $m$ nodes. Global: the divergence-free Mat\'ern interpolant of~\citet{sharma2026fluids} on all nodes. Local: our output reconstruction with the cubic PHS, or with the multiscale Mat\'ern kernel ($c_\sigma=16$) where marked (Appendix~\ref{app:multiscale-spatial} compares it with the PHS). Bold: smaller error in each row.}
\label{tab:spatial-reconstruction-controls}
\begin{tabular}{lrrrr}
\toprule
Data set & Nodes $m$ & Evaluation points & Global & Local \\
\midrule
Forced isotropic turbulence & 7000 & 3000 & $1.88\times10^{-2}$ & $\mathbf{1.65\times10^{-2}}$ (multiscale) \\
Turbulent species transport & 7000 & 4388 & $1.52\times10^{-1}$ & $\mathbf{1.47\times10^{-1}}$ (multiscale) \\
Merging vortices & 500 & 6947 & $1.88\times10^{-1}$ & $\mathbf{1.35\times10^{-1}}$ \\
Lid-driven cavity & 1000 & 8566 & $6.15\times10^{-1}$ & $\mathbf{8.71\times10^{-2}}$ \\
Cylinder flow (no shedding) & 1000 & 8520 & $\mathbf{8.67\times10^{-2}}$ & $9.19\times10^{-2}$ \\
Cylinder flow (vortex shedding) & 1000 & 8520 & $\mathbf{9.65\times10^{-2}}$ & $1.12\times10^{-1}$ \\
Backward-facing step & 1000 & 6359 & $\mathbf{5.84\times10^{-3}}$ & $1.04\times10^{-2}$ \\
Buoyancy-driven cavity & 5000 & 5514 & $9.42\times10^{-2}$ & $\mathbf{9.31\times10^{-3}}$ \\
Taylor--Green vortices (space) & 500 & 6947 & $\mathbf{2.30\times10^{-4}}$ & $6.20\times10^{-4}$ \\
Taylor--Green vortices (spacetime) & 500 & 6947 & $1.99\times10^{-3}$ & $\mathbf{6.20\times10^{-4}}$ \\
\bottomrule
\end{tabular}
\end{table}

%% file: sections/amplification_table.tex
\begin{table}[pos=tbp]
\centering\small
\setlength{\tabcolsep}{4pt}
\caption{Amplification factor $\alpha_{\rm spatial}=\|B\|_F\|(\bA^\star)^{\mathsf T}\|_F/\|B(\bA^\star)^{\mathsf T}\|_F$ of~\eqref{eq:spatial-amplification}, which measures the cancellation among the spatial coefficients, for $N=1000$ training fields, for the global and for our local output reconstruction (stencils combined as in Section~\ref{sec:lagrange_accuracy}), and the largest ratio of the direct to the Lagrange error with global output reconstruction from Table~\ref{tab:lagrange-all}. Bold: smaller amplification factor in each row.}
\label{tab:amplification}
\begin{tabular}{lrrr}
\toprule
Data set & Global $\alpha_{\rm spatial}$ & Local $\alpha_{\rm spatial}$ & Largest gain (global) \\
\midrule
Forced isotropic turbulence & $3.7\times10^{9}$ & $\mathbf{396}$ & $1.2$ \\
Turbulent species transport & $2.9\times10^{9}$ & $\mathbf{9882}$ & $7.1$ \\
Merging vortices & $4.8\times10^{9}$ & $\mathbf{73}$ & $868.1$ \\
Lid-driven cavity & $5.8\times10^{9}$ & $\mathbf{13}$ & $34.5$ \\
Cylinder flow (no shedding) & $5.1\times10^{7}$ & $\mathbf{149}$ & $4.8$ \\
Cylinder flow (vortex shedding) & $5.2\times10^{7}$ & $\mathbf{149}$ & $2.0$ \\
Backward-facing step & $5.3\times10^{8}$ & $\mathbf{38}$ & $215.3$ \\
Buoyancy-driven cavity & $5.9\times10^{9}$ & $\mathbf{6}$ & $14.8$ \\
Taylor--Green vortices (space) & $1.6\times10^{7}$ & $\mathbf{73}$ & $6.1$ \\
Taylor--Green vortices (spacetime) & $1.9\times10^{2}$ & $\mathbf{145}$ & $1.0$ \\
\bottomrule
\end{tabular}
\end{table}

%% file: sections/lagrange_all_table.tex
\begin{table}[pos=tbp]
\centering\small
\setlength{\tabcolsep}{4pt}
\caption{Ratio of the test error of direct evaluation to that of Lagrange evaluation, for the same unaugmented operator models (the fixed-shape kernel of Section~\ref{sec:sparse_wendland} at the radii of the five prescribed densities and at its validated radius, and the fitted-shape kernel at its validated radius, for $N=500$, $1000$, $5000$, and $10000$) with global and local output reconstruction. Values above one favor Lagrange evaluation. Over all 560 pairs, the smallest ratio is $0.99987$.}
\label{tab:lagrange-all}
\begin{tabular}{lrrrrr}
\toprule
 & \multicolumn{3}{c}{Global output reconstruction} & \multicolumn{2}{c}{Local output reconstruction} \\
\cmidrule(lr){2-4}\cmidrule(lr){5-6}
Data set & Models & Median & Max & Models & Max $|$ratio$-1|$ \\
\midrule
Forced isotropic turbulence & 28 & 1.00 & 1.2 & 28 & $5.9\times10^{-11}$ \\
Turbulent species transport & 28 & 1.00 & 7.1 & 28 & $8.6\times10^{-7}$ \\
Merging vortices & 28 & 48.96 & 868.1 & 28 & $4.2\times10^{-8}$ \\
Lid-driven cavity & 28 & 5.07 & 34.5 & 28 & $1.4\times10^{-7}$ \\
Cylinder flow (no shedding) & 28 & 1.05 & 4.8 & 28 & $7.0\times10^{-8}$ \\
Cylinder flow (vortex shedding) & 28 & 1.01 & 2.0 & 28 & $2.4\times10^{-8}$ \\
Backward-facing step & 28 & 16.51 & 215.3 & 28 & $2.0\times10^{-8}$ \\
Buoyancy-driven cavity & 28 & 1.47 & 14.8 & 28 & $1.9\times10^{-8}$ \\
Taylor--Green vortices (space) & 28 & 1.05 & 6.1 & 28 & $1.7\times10^{-6}$ \\
Taylor--Green vortices (spacetime) & 28 & 1.00 & 1.0 & 28 & $2.4\times10^{-6}$ \\
\bottomrule
\end{tabular}
\end{table}

%% file: sections/conclusion.tex
\FloatBarrier
\section{Conclusions}
\label{sec:conclusion}

In this article, we presented a local and sparse kernel method for learning the solution operators of incompressible flows. It keeps the guarantee of our earlier property-preserving method~\citep{sharma2026fluids}, that every prediction is divergence free, and removes the two costs of that method that grow with the problem size. Local divergence-free interpolants replace the global output reconstruction, so the output can be refined without enlarging any linear system, and the local systems are independent and solved in parallel on the GPU. The output reconstruction is divergence free, periodic on periodic domains, and, with the multiscale Mat\'ern kernel, approaches Kolmogorov statistics under refinement at the separations that the nodes resolve. A compactly supported kernel that is positive semidefinite in any input dimension replaces the dense operator regression, so the kernel matrix is sparse. Its support radius, and with it the sparsity, is fitted to the data, and optional affine polynomials in the input, which the operator regression reproduces exactly, represent trends that a compactly supported kernel cannot. Operator Lagrange weights join the two, so each new input costs one sparse solve and the local coefficients of the training fields are never formed.

The local output reconstruction converges at close to the order of its polynomial degree on scattered nodes, is about as accurate between the nodes as the global output reconstruction (more accurate on six of ten benchmarks and less accurate on the other four) with systems that are smaller by about two orders of magnitude, is cheaper to set up than a global output reconstruction beyond a few thousand output points, with memory that grows linearly with their number, keeps the divergence and the periodic mismatch at rounding level, and, with the multiscale Mat\'ern kernel, approaches prescribed Kolmogorov statistics at the separations that the nodes resolve.

In the operator benchmarks, the test errors decrease roughly like $N^{-\tau}$ with $\tau$ between $0.4$ and $1.3$ in all but two degenerate cases, and the most accurate kernels are sparse on all benchmarks but one, with densities that mostly decrease as the training set grows. At the training sizes of our earlier study, these sparse models are at least as accurate as dense PPKM on every data set, and more accurate than the neural operators reported there by factors from $2$ to more than $10^6$. Over our range of training sizes, the storage of most models grows roughly linearly in $N$. The benchmarks with rank-one inputs, apart from Taylor--Green, are insensitive to the choice of kernel and augmentation. The benchmarks with multivariate inputs are not, and the best choice differs between them. Affine augmentation gives the largest gain for forced isotropic turbulence, a fitted kernel shape for turbulent species transport, and neither helps for merging vortices. Our recommendation is therefore to select these choices by validation and to report error, density, and storage together.

Lagrange evaluation removes the need to store the local coefficients of the training fields, which would take $n_s+n_P/d$ times (the stencil size plus the number of polynomials per velocity component) the storage of the training fields themselves. Its accuracy agrees with our rounding analysis. With our local output reconstruction, whose amplification factors are at most about $10^4$, it is as accurate as direct evaluation. With global output reconstructions, whose amplification factors reach $6\times10^9$, it is up to nearly three orders of magnitude more accurate.

Several extensions are natural. First, benchmarks with richer input families, and a comparison with dense operator regression and with neural operators on identical test sets and hardware, would test these conclusions beyond our ten benchmarks. Second, our local interpolants are divergence free on each stencil but are not coupled across stencils. Blending them with a partition of unity would give a continuous field, but the blending weights introduce divergence. The divergence-free partition-of-unity method of~\citet{drake2021partition} avoids this by blending potentials instead of velocities. Third, at a fixed density, the storage of the kernel matrix still grows quadratically with the number of training functions, which hierarchical or multilevel sparse factorizations could reduce. Finally, our rounding analysis covers the model without augmentation and evaluation through coefficients. Extending it to the affine correction and to evaluation before the output reconstruction would complete the analysis.

%% file: sections/ack.tex
\section*{Acknowledgements}
\label{sec:ack}
RS, ML, and VS were supported by the National Science Foundation under DMS 2505986, Cooperative Agreement 2421782, and the Simons Foundation grant MPS-AI-00010515 awarded to the NSF-Simons AI Institute for Cosmic Origins — CosmicAI, \url{https://www.cosmicai.org}. VS was additionally supported by the Air Force Office of Scientific Research (AFOSR) under LRIR award FA9550-25-1-0042. HO acknowledges support from the Air Force Office of Scientific Research under MURI award number FOA-AFRL-AFOSR-2023-0004 (Mathematics of Digital Twins), the Department of Energy under award number DE-SC0023163 (SEA-CROGS: Scalable, Efficient, and Accelerated Causal Reasoning Operators, Graphs and Spikes for Earth and Embedded Systems) and the DoD Vannevar Bush Faculty Fellowship Program under ONR award number N00014-18-1-2363.

%% file: sections/appendix.tex
\FloatBarrier
\appendix
\section{Details of the local output reconstruction}
\label{app:implementation}

\subsection{Ordering of vector coefficients}\label{app:coefficient-ordering}
The coefficient vectors of~\eqref{eq:local_dfphs_system} are ordered by component, with all values of one component before those of the next, but the kernel evaluation at a query point $y_q^\star$ is naturally ordered by stencil node,
\begin{align}
\bPhi_{j(q),q,\mathrm{pt}}^\star =
\begin{bmatrix}
    \Phi_{\rm DF}(y_q^\star,\tilde y_{j(q),1}) & \cdots & \Phi_{\rm DF}(y_q^\star,\tilde y_{j(q),n_s})
\end{bmatrix}
\in \mathbb{R}^{d \times d n_s}.
\label{eq:local_dfphs_kernel_evaluation}
\end{align}
Let $\Pi_{d,n_s}$ be the permutation with $\Pi_{d,n_s}\bb_{\Phi,\mathrm{comp}}=\bb_{\Phi,\mathrm{pt}}$, where $\bb_{\Phi,\mathrm{pt}}$ groups the $d$ components of each node. The evaluation matrix in~\eqref{eq:local_dfphs_evaluation} is then
\begin{align}
\bPhi_{j(q),q,\mathrm{comp}}^\star
&=\bPhi_{j(q),q,\mathrm{pt}}^\star\Pi_{d,n_s}.
\label{eq:local_dfphs_component_order}
\end{align}

\subsection{Polynomial basis and local solves on the GPU}
\label{app:local-solves}
We first remove the dependencies among the candidates that do not depend on the stencil, once, before sampling. A singular value decomposition of the fixed map from candidates to polynomial coefficients gives an orthonormal basis $\Xi$ of its row space, shared by all stencils of one dimension and degree, with
\begin{equation}
G_j=G_j\Xi\Xi^{\mathsf T},\qquad \Xi^{\mathsf T}\Xi=I,
\label{eq:gpu-poly-rowspace}
\end{equation}
in exact arithmetic. Thus $G_j\Xi$ has the nonzero singular values of $G_j$, and we factor it by thin QR, $G_j\Xi=\widetilde Q_j\widetilde R_j$. On most stencils, $\widetilde R_j$ is far from singular, and the polynomial basis is
\begin{equation}
\bP_j(y)=G_j(y)\Xi\widetilde R_j^{-1},
\label{eq:gpu-poly-changebasis}
\end{equation}
whose values at the stencil nodes are the orthonormal columns of $\widetilde Q_j$, as in~\eqref{eq:dfpoly_change_basis}. We use~\eqref{eq:gpu-poly-changebasis} when the eigenvalues of $\widetilde R_j^{\mathsf T}\widetilde R_j$ certify that its smallest singular value is at least $10^{-6}$ times its largest, far above the rank threshold of $10^{-12}$, since these eigenvalues carry absolute errors of the order of the unit roundoff times the largest one. On the remaining stencils, the column-pivoted QR factorization $\widetilde R_j\Pi_j=\widehat Q_j\widehat R_j$ determines the rank as in~\eqref{eq:dfpoly_cpqr}, and the basis is~\eqref{eq:dfpoly_change_basis} with $G_{j,I}$ replaced by the retained columns of $G_j\Xi\Pi_j$. Both factorizations are batched over the stencils, and the columns of different stencils may differ, since every system is padded to the same width (below). A singular value decomposition on every stencil would decide the rank as well, but its batched form is slow on the GPU for these matrices ($85\times85$ in 3D at degree four, about $10$ ms each). On the species-transport node set, where most stencils are rank deficient, both rank decisions agree on every stencil, and the output reconstructions between the nodes agree to $10^{-12}$.

To keep the batches of fixed size, every system carries all $20$ (2D) or $85$ (3D) polynomial columns at degree four. The columns of discarded modes are set to zero and given a unit diagonal, which decouples their coefficients without changing the interpolant. Every system therefore has the padded width $62$ (2D) or $262$ (3D), and our costs and storage include the padding.

If the LU solution of a local system is not finite, we use the pseudoinverse of that system, also for its evaluation weights. The resulting field is still divergence free, but it need not interpolate or reproduce the full polynomial space. The pseudoinverse was never needed. The LU solutions are finite on every node stencil of all ten operator benchmarks, including the rank-deficient stencils of Table~\ref{tab:spatial-ranks}, and in the manufactured tests of Section~\ref{sec:spatial-convergence} every stencil kept the full polynomial space.

\begin{table}[pos=tbp]
\centering\small
\caption{Retained divergence-free polynomial ranks of the local systems on the nodes of each benchmark, at degree four and with periodic images, from the construction above. The full ranks are 20 in 2D and 85 in 3D. Taylor--Green vortices (spacetime) has the stencils of Taylor--Green vortices (space).}
\label{tab:spatial-ranks}
\begin{tabular}{lrrr}
\toprule
Data set & Stencils & Retained rank & Truncated stencils \\
\midrule
Turbulent species transport & 7000 & 45--85 & 4338 \\
Forced isotropic turbulence & 7000 & 85 & 0 \\
Buoyancy-driven cavity & 5000 & 20 & 0 \\
Lid-driven cavity & 1000 & 20 & 0 \\
Cylinder flow (no shedding) & 1000 & 20 & 0 \\
Cylinder flow (vortex shedding) & 1000 & 20 & 0 \\
Backward-facing step & 1000 & 20 & 0 \\
Taylor--Green vortices (space) & 500 & 19--20 & 1 \\
Merging vortices & 500 & 19--20 & 1 \\
\bottomrule
\end{tabular}
\end{table}

\subsection{Evaluation away from the nodes}
\label{app:stencil-assignment}
The operator benchmarks evaluate predictions at the nodes. To evaluate at another point, our GPU implementation selects a stencil centered at that point and builds its basis there. The comparisons of Table~\ref{tab:spatial-reconstruction-controls} use this rule. The manufactured tests of Section~\ref{sec:spatial-convergence} instead use the stencil of the nearest node, and their results between the nodes refer to that rule.

\subsection{Periodic stencils}
\label{app:periodic-correction}
The Taylor--Green flows and the merging vortices have period $2\pi$ in each direction. Forced isotropic turbulence has period $2\pi$ on the box $[0,2\pi]^3$, which we supply explicitly, since the cell centers of its grid do not reach the faces of the box.

\subsection{Output reconstruction between the nodes}
\label{app:spatial-supporting}
\label{app:multiscale-spatial}
Table~\ref{tab:spatial-reconstruction-controls} reconstructs the test predictions at the nodes of the model with the smallest validation error among the four of Section~\ref{sec:protocol}, at $N=10000$, so that the operator regression is that of Section~\ref{sec:regression-results} and only the output reconstruction changes. The evaluation points of Table~\ref{tab:spatial-reconstruction-controls} are all mesh or grid points of the data set that are not among the $m$ output points. The local output reconstruction selects a stencil centered at each of them (Appendix~\ref{app:stencil-assignment}). The global output reconstruction is the divergence-free interpolant of~\citet{sharma2026fluids} with the scalar kernel $3M_{5/2}(\varepsilon r)$ on all nodes, with the shape parameters of that work ($\varepsilon=0.3$ for the lid-driven cavity, $0.0636$ for both cylinder flows, $0.06$ for the backward-facing step, $0.8$ for the buoyancy-driven cavity, $0.02$ for merging vortices and Taylor--Green vortices (space), and $1.5$ for Taylor--Green vortices (spacetime)), applied at each sampled time for Taylor--Green vortices (spacetime). For the periodic data sets, it uses one representative of each set of periodically equivalent nodes.

We select the scale factor $c_\sigma$ of the multiscale kernel on the $100$ validation fields, by reconstructing their exact nodal values and measuring the error between the nodes. For turbulent species transport, $c_\sigma=2$, $4$, $8$, and $16$ give errors of $0.166$, $0.158$, $0.150$, and $0.147$, and the PHS gives $0.169$. For forced isotropic turbulence, they give $0.0228$, $0.0206$, $0.0178$, and $0.0165$, and the PHS gives $0.0227$. Both select $c_\sigma=16$, which the operator benchmarks use. For turbulent species transport, the error between the nodes remains much larger than at the nodes, so an accurate prediction at the nodes does not guarantee one everywhere.

We also compare the second-order structure functions
\begin{align}
S_2(r)=\left\langle\|v(x)-v(y)\|_2^2\right\rangle_{\|x-y\|\approx r}
\label{eq:multiscale_structure_function}
\end{align}
of the output reconstructions and of the reference fields on the same pairs of points between the nodes ($256$ anchors with up to four partners in each of $12$ logarithmic distance bins, with periodic distances for forced isotropic turbulence), averaged over the $100$ test fields (Figure~\ref{fig:multiscale-structure-functions}). The relative discrepancy of the binned structure function is $1.7\%$ for the PHS and $2.6\%$ for the multiscale kernel for turbulent species transport, and $0.40\%$ and $0.20\%$ for forced isotropic turbulence. The multiscale kernel thus lowers the pointwise error between the nodes for both flows, but improves the binned increments only for forced isotropic turbulence. Since the reference fields come from Reynolds-averaged (RANS) and coarse large-eddy (LES) simulations, this comparison measures agreement with the reference, not with physical turbulence.

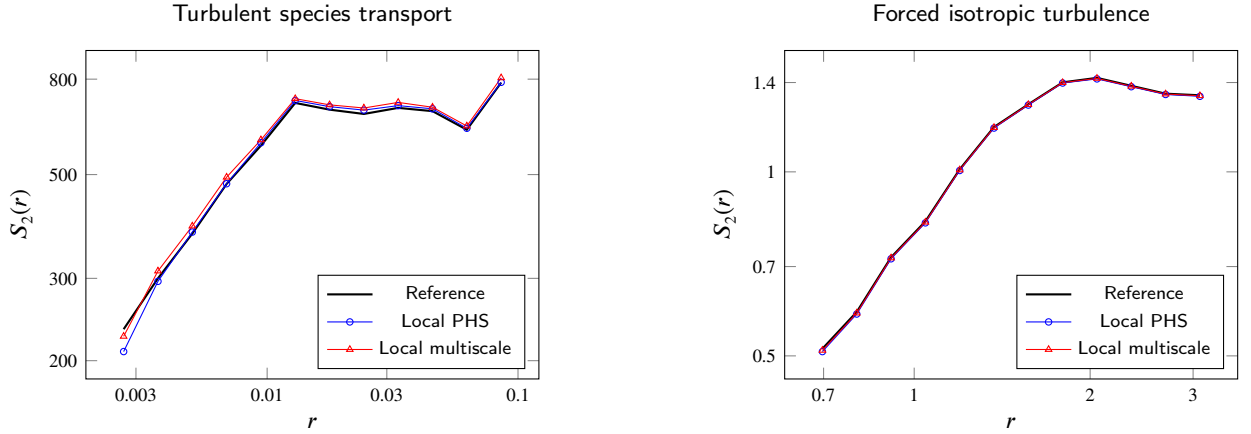
\begin{figure}[pos=tbp]
\centering
\input{sections/fig_structure_functions}
\caption{Second-order structure functions between the nodes of the reference fields and of the local PHS and multiscale Mat\'ern output reconstructions of the test predictions at $N=10000$.}
\label{fig:multiscale-structure-functions}
\end{figure}

\section{Details of the operator regression}
\label{app:operator-details}

\subsection{Normalization and principal directions}
\label{app:residual-gate}
At every $N$, we compute the normalization and the principal directions from the first $N$ training inputs alone. Each coordinate is divided by its standard deviation, or set to zero if that vanishes, and the numerical rank uses the threshold $\max\{10^{-10},2u_{\rm mach}\max(N,nd_a)\}$ relative to the largest singular value, with $u_{\rm mach}$ the unit roundoff. The four models at one $N$ share the normalization, the principal directions, and the gate~\eqref{eq:residual-gate}, and the two models with augmentation share the affine matrix.

\subsection{Fitting and validation}
\label{app:gpu-selection}
The fitted-shape kernel starts from the radius $\varrho_0$ that gives $25\%$ density and from the regularization $\theta=10^{-6}$, and takes $60$ steps of the Adam optimizer~\citep{kingma2015adam} with learning rate $0.03$ for each of $c_d=0$ and $c_d=1$. Its log-radius is confined to $[\log(\varrho_0/8),\log(4\varrho_0)]$ and its regularization to $\theta=e^{t_\theta}+10^{-12}$ with $t_\theta\in[\log(10^{-12}),\log(0.1)]$. The anisotropy uses $n_z-1$ free parameters $b_k\in[-2,2]$, completed by a zero, through $\tilde\upsilon_k=\log(2)\tanh(b_k)$ and $\upsilon_k=\tilde\upsilon_k-n_z^{-1}\sum_{k'=1}^{n_z}\tilde\upsilon_{k'}$. The residual scale lies in $[1/4,4]$. It is inactive for the seven benchmarks with $n_z=1$ and stays at one for turbulent species transport.

Validation evaluates the radii of the five density budgets of Section~\ref{sec:kernel-training}. With $\varrho_{\max}$ the largest of these, it also searches radii in $[\varrho_{\max}/8,\varrho_{\max}]$ from four starts (radii $\varrho_{\max}/2$ and $\varrho_{\max}/4$, regularizations $10^{-6}$ and $10^{-3}$), each with $35$ Adam steps on all $100$ validation fields. A fitted-shape model repeats this for both of its fitted shapes, so its search is twice as long. We keep the model with the smallest validation error among those with density at most $0.90$, and break ties by density.

\subsection{Work, storage, and timing}
\label{app:cost-details}
We assemble the kernel matrix by visiting all pairs of training inputs in blocks of bounded size and keeping those within the support. This keeps the memory bounded, but the work of the search is still quadratic in $N$, since compact support alone does not give a faster neighbor search. The likelihood and its gradient use the factors through
\begin{equation}
\Pi_U=L_K^{-\mathsf T}(I-Q_PQ_P^{\mathsf T})L_K^{-1},
\end{equation}
with the entries of the inverse needed for the trace derivatives computed only on the support. When the outputs are wide, we replace $E$ by the triangular factor of a thin QR factorization $E^{\mathsf T}=Q_ER_E$, since $EE^{\mathsf T}=R_E^{\mathsf T}R_E$.

The storage we report is the sum of the stored model arrays (including the training fields), the sparse matrix arrays, cuDSS's estimate of the memory of its factors, and the solve buffers. It is computed from the stored arrays instead of measured at its peak, and it excludes compiled code and temporary workspace. A kernel with $\mathcal N_K$ nonzeros in compressed sparse row form takes $12\mathcal N_K+4(N+1)$ bytes.

We report the setup of the output reconstruction, which all four models share, separately from the fitting and validation of each model. The warm prediction time of Figure~\ref{fig:gpu-warm-main} is the median of $11$ repetitions of predicting all $100$ test fields after the model is set up and compiled. The four models of a benchmark run in one process and share compiled code, so small differences between their times are not significant.

\subsection{Batches and checks}
\label{app:query-batches}
The $100$ test inputs are solved in four batches of $32$ right-hand sides, the last one padded by repeated inputs. Repeated solves replay CUDA graphs (recorded sequences of GPU kernel launches). After every factorization, we check the residual of the solves, the reproduction of affine functions, and the agreement of the batched and compiled predictions with reference predictions. Every reported model passed these checks.

\section{Proofs}
\label{app:proofs}

\subsection{Completeness of the polynomial candidates}
\label{app:polynomial-span}
In 2D, every divergence-free vector polynomial of degree at most $\varpi$ has a polynomial stream function of degree at most $\varpi+1$, so the rotated gradients of the scalar basis span the space. In 3D, split a divergence-free polynomial into homogeneous parts $v_j$ of degree $j$, each of which is divergence free. The identity
\begin{equation}
\nabla_\xi\times(v_j\times\xi)=(j+2)v_j
\end{equation}
shows that $v_j$ is the curl of the potential $(v_j\times\xi)/(j+2)$ of degree $j+1$, and expanding the components of this potential in the scalar basis gives a combination of the candidates~\eqref{eq:dfpoly_candidates}. For the dimension, the divergence maps the vector polynomials of degree at most $\varpi$ onto the scalar polynomials of degree at most $\varpi-1$ (integrate in one coordinate to find a preimage), and rank--nullity gives $d\binom{\varpi+d}{d}-\binom{\varpi+d-1}{d}$.

\subsection{Positive definiteness of the kernel}
\label{app:dimension-safe-proof}
\begin{proof}[Proof of Proposition~\ref{prop:dimension-safe-pd}]
The scaled principal coordinates $D_\parallel z$ lie in $\mathbb R^{n_z}$ with $n_z\leq3$, so the Wendland factor is a positive semidefinite kernel. The Mat\'ern factor is positive semidefinite on the residual space in every finite dimension. The Gram matrix of their product is the Schur product of two positive semidefinite matrices, and multiplying it on both sides by the diagonal matrix of gates keeps it positive semidefinite.

For strictness, group distinct standardized inputs by their distinct principal coordinates $z^{(1)},\ldots,z^{(J)}$, and let $W$ be the Wendland Gram matrix on these, which is positive definite. Let $\Psi_\perp$ be a feature map of the residual Mat\'ern kernel, and for coefficients $c_i$ let
\begin{equation}
f_\iota=\sum_{i:z_i=z^{(\iota)}}c_i \eta(\ba_i)\Psi_\perp(\zeta_i).
\end{equation}
The quadratic form of the product Gram matrix is $\sum_{\iota,\iota'}W_{\iota\iota'}\langle f_\iota,f_{\iota'}\rangle\geq\lambda_{\min}(W)\sum_\iota\|f_\iota\|^2$. Inputs in the same group differ in their residuals, so strict positive definiteness of the Mat\'ern kernel gives $f_\iota\ne0$ whenever a coefficient of the group is nonzero and the gates are positive. (If the residual space is trivial, distinct inputs lie in distinct groups.) Finally, $c^{\mathsf T}(K_{UU}+\theta I)c\geq\theta\|c\|_2^2>0$ for every $c\ne0$ and $\theta>0$.
\end{proof}
The proof does not assume that the inputs lie in a low-dimensional affine subspace. Note that standardization can make two inputs equal if they differ only in coordinates with zero training variance. The regularization keeps the matrix positive definite in that case.

\subsection{Variational form of the operator regression}
\label{app:operator-variational}
Let $f_{\mathcal H}\in\mathcal H$, and split $f_{\mathcal H}=f_\parallel+f_\perp$, where $f_\parallel$ lies in the span of the functions $\lambda(\cdot,\ba_i)\mathbf e_k$ and $f_\perp$ is orthogonal to it. By the reproducing property, $f_\perp(\ba_i)=0$ for every $i$, so $f_\perp$ leaves the data term of~\eqref{eq:operator-variational} unchanged and only increases the norm. A minimizer therefore has the form $f_{\mathcal H}=\sum_i\lambda(\cdot,\ba_i)\bc_i$, with $\|f_{\mathcal H}\|_{\mathcal H}^2=\operatorname{tr}(C^{\mathsf T}K_{UU}C)$, where $C$ has the rows $\bc_i^{\mathsf T}$. With $B$ the matrix with rows $\bb_i^{\mathsf T}$ and $P=P_U$, the objective becomes the convex quadratic
\begin{equation}
J(C,\boldsymbol\beta)=\|K_{UU}C+P\boldsymbol\beta-B\|_F^2+\theta\operatorname{tr}(C^{\mathsf T}K_{UU}C),
\end{equation}
with gradients
\begin{equation}
\nabla_CJ=2K_{UU}\bigl(K_{UU}C+P\boldsymbol\beta-B+\theta C\bigr),\qquad
\nabla_{\boldsymbol\beta}J=2P^{\mathsf T}\bigl(K_{UU}C+P\boldsymbol\beta-B\bigr).
\end{equation}
The conditions~\eqref{eq:operator-interpolant-conditions} state that $K_{UU}C+P\boldsymbol\beta-B=-\theta C$ and $P^{\mathsf T}C=0$. Both gradients therefore vanish, and every solution of~\eqref{eq:operator-interpolant-conditions} minimizes $J$, and hence~\eqref{eq:operator-variational}. Without augmentation, the same argument applies with $P$ and $\boldsymbol\beta$ removed.

\subsection{Block elimination and reproduction}
\label{app:augmented-lagrange-proof}
\begin{proof}[Proof of Proposition~\ref{prop:operator-reproduction}]
Write $K=K_{UU}^{\theta}=L_KL_K^{\mathsf T}$ and $P=P_U$. Since $P$ has full column rank, $L_K^{-1}P=Q_PR_P$ with $R_P$ invertible, and
\begin{equation}
\mathcal E_UP=R_P^{-1}Q_P^{\mathsf T}L_K^{-1}P=I,\qquad
K\mathcal E_U^{\mathsf T}=PR_P^{-1}R_P^{-\mathsf T}.
\end{equation}
Let $\ell_0=K^{-1}k_U(\ba)$ and $\tilde p=p(\ba)-P^{\mathsf T}\ell_0$. Then $\ell=\ell_0+\mathcal E_U^{\mathsf T}\tilde p$ satisfies
\begin{equation}
P^{\mathsf T}\ell=P^{\mathsf T}\ell_0+\tilde p=p(\ba),\qquad
K\ell=k_U(\ba)+PR_P^{-1}R_P^{-\mathsf T}\tilde p,
\end{equation}
so~\eqref{eq:augmented-lagrange-system} holds with $\mu=-R_P^{-1}R_P^{-\mathsf T}\tilde p$. If $(v,\omega)$ solves the homogeneous system, then $P^{\mathsf T}v=0$ and $v^{\mathsf T}Kv=0$, so $v=0$, and full column rank gives $\omega=0$. The solution is unique.

For targets in the space of affine functions, $B_j=P_UT_j$, we get $B_j^{\mathsf T}\ell=T_j^{\mathsf T}p(\ba)$, which is reproduction. The coefficients of~\eqref{eq:operator-augmented-primal} are
\begin{equation}
C_j=\Pi_U B_j,\qquad \beta_j=\mathcal E_UB_j,\qquad
k_U(\ba)^{\mathsf T}C_j+p(\ba)^{\mathsf T}\beta_j
=\ell(\ba)^{\mathsf T}B_j,
\end{equation}
where the last identity follows by transposing the augmented system. Direct and Lagrange evaluation thus agree in exact arithmetic. Finally, linearity of $\chi$ and $\mathcal D$ gives~\eqref{eq:operator-property-preservation}. For a nonzero common constraint $\mathcal D\chi(\bb_i)=f$, the affine model also gives $\mathcal D\chi(\sum_i\ell_i\bb_i)=f$, since its constant column enforces $\sum_i\ell_i=1$.
\end{proof}
Without augmentation, $\ell=(K_{UU}^{\theta})^{-1}k_U(\ba)$. If $\theta=0$, then $k_U(\ba_i)$ is the $i$th column of $K_{UU}$, so $\ell(\ba_i)=\mathbf e_i$ and the weights are cardinal. With $\theta>0$, the vector $k_U$ excludes the regularization, and the weights are no longer cardinal.

\subsection{Rounding errors of direct and Lagrange evaluation}
\label{app:lagrange-proof}
\begin{proof}[Proof of Theorem~\ref{prop:lagrange-error}]
The first residual equation gives $\widehat C=(K_{UU}^{\theta})^{-1}(B-R_{\rm dir})$, and multiplying by $K_{U_*U}$ and $(\bA^\star)^{\mathsf T}$ gives~\eqref{eq:direct-error-identity}. The second gives $\widehat L=(K_{U_*U}-R_{\rm Lag})(K_{UU}^{\theta})^{-1}$, and hence~\eqref{eq:lagrange-error-identity}.

For the rounding, let $Z=\operatorname{fl}(M_1M_2)$ with $(M_1,M_2)=(K_{U_*U},\widehat C)$ or $(\widehat L,B)$. The standard bound for matrix products~\citep[Chapter~3]{higham2002accuracy} gives $|Z-M_1M_2|\leq\gamma_N|M_1||M_2|$, and the product with $(\bA^\star)^{\mathsf T}$ satisfies $|\operatorname{fl}(Z(\bA^\star)^{\mathsf T})-Z(\bA^\star)^{\mathsf T}|\leq\gamma_{r+d}|Z||(\bA^\star)^{\mathsf T}|$, also when it is computed in component blocks (dot products of length at most $r$ followed by at most $d-1$ additions). Together,
\begin{equation}
\big|\operatorname{fl}(\operatorname{fl}(M_1M_2)(\bA^\star)^{\mathsf T})-M_1M_2(\bA^\star)^{\mathsf T}\big|
\leq \gamma_{\rm mult} |M_1||M_2||(\bA^\star)^{\mathsf T}|.
\label{eq:two-product-rounding}
\end{equation}
The triangle inequality and submultiplicativity of the Frobenius norm then give~\eqref{eq:direct-full-error-bound}--\eqref{eq:lagrange-full-error-bound}.
\end{proof}

\section{Details of the numerical experiments}
\label{app:additional-results}

\subsection{Benchmarks}
\label{app:benchmark-configurations}
The ten data sets are those of our previous work~\citep{sharma2026fluids}, where their simulations are described in full. Each figure shows the first simulation of a data set. The color fields interpolate the simulation samples linearly (Appendix~\ref{app:reference-visualization}).

\subsubsection{Flow past a cylinder}
\label{sec:results_cylinder}
\label{app:cylinder-no-shedding}
Two-dimensional flow passes a cylinder of radius $0.8$ centered at $(1.5,7)$ in the channel $\Omega_a=\Omega_v=[0,20]\times[0,14]$, with an inflow on the left, freestream conditions at the top and bottom, no slip on the cylinder, and zero pressure at the outlet. The constant initial velocity varies within the vortex-shedding regime in one data set and below the onset of vortex shedding in the other. We learn $\G:\bu(x,0)\rightarrow\bu(y,10)$ on $m=1000$ approximate Fekete points (well-spaced points selected from the mesh as in~\citet{sharma2026fluids}) of the $9520$ mesh points.

\begin{figure}[pos=tbp]
\centering
\includegraphics{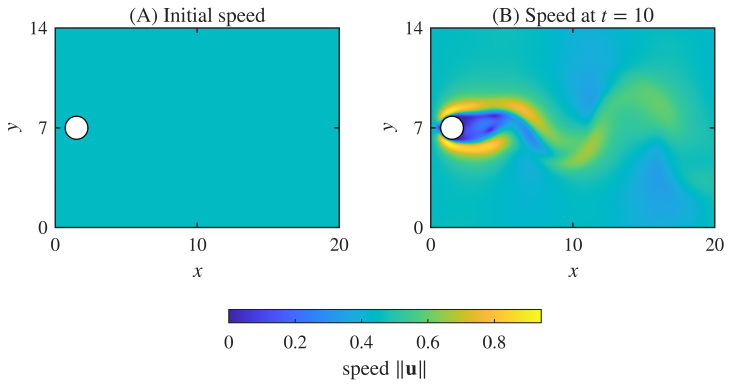}
\caption{Cylinder flow. (A) Initial speed. (B) Speed at $t=10$, on the same color scale.}
\label{fig:cylinder_problem}
\end{figure}

\subsubsection{Taylor--Green vortices}
\label{sec:results_taylor_green}
\label{app:taylor-green}
The analytical Taylor--Green vortices decay on the periodic domain $\Omega_a=\Omega_v=[0,2\pi]^2$, sampled at $500$ Fekete points. For the spatial data set, we learn $\G:\bu(x,0)\rightarrow\bu(y,1)$. For the spacetime data set, we learn the map from $\bu(x,0)$ to the velocity on $\Omega\times[0.7,1]$, with snapshots at $\mathcal T=\{0.7,0.8,0.9,1.0\}$ and cubic interpolation in time (Section~\ref{sec:spacetime}). Both use the amplitude normalization~\eqref{eq:tg-response-transform}.

\begin{figure}[pos=tbp]
\centering
\includegraphics{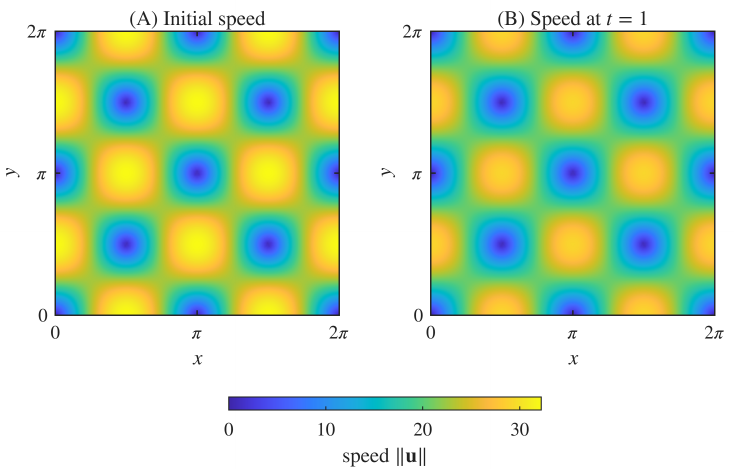}
\caption{Taylor--Green vortices. (A) Initial speed. (B) Speed at $t=1$, on the same color scale.}
\label{fig:taylor_green_problem}
\end{figure}

\subsubsection{Lid-driven cavity}
\label{app:lid-cavity}
In the cavity $\Omega_a=\Omega_v=[0,1]^2$, the speed of the lid is drawn from a truncated normal distribution, and the fluid starts at rest, so the initial velocity is nonzero only on the lid. We learn $\G:\bu(x,0)\rightarrow\bu(y,5)$ on $m=1000$ Fekete points selected from the $9566$ mesh points.
\begin{figure}[pos=tbp]
\centering
\includegraphics{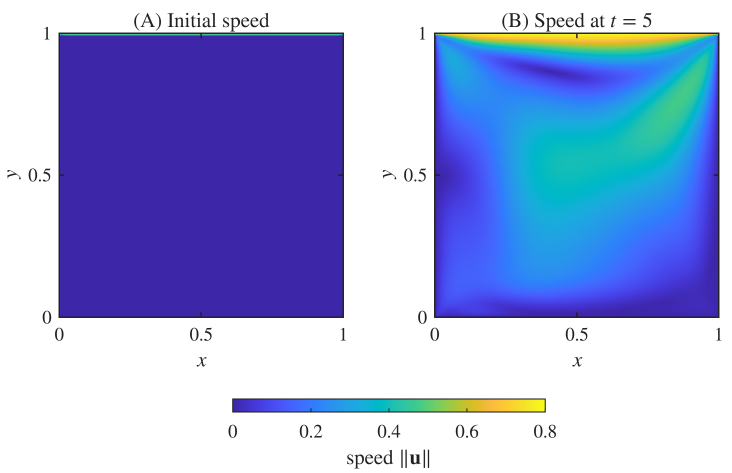}
\caption{Lid-driven cavity. (A) Initial speed. (B) Speed at $t=5$, on the same color scale.}
\label{fig:appendix-lid-cavity-problem}
\end{figure}

\subsubsection{Backward-facing step}
\label{app:backward-step}
A parabolic inflow, scaled by a factor drawn from a truncated normal distribution, drives the flow over a step into the channel $\Omega_a=\Omega_v=[0,15]\times[-0.5,0.5]$. We learn $\G:\bu(x,0)\rightarrow\bu(y,5)$ on $m=1000$ Fekete points selected from the $7359$ mesh points. The input has only $11$ coordinates, the samples of the inflow profile.
\begin{figure}[pos=tbp]
\centering
\includegraphics{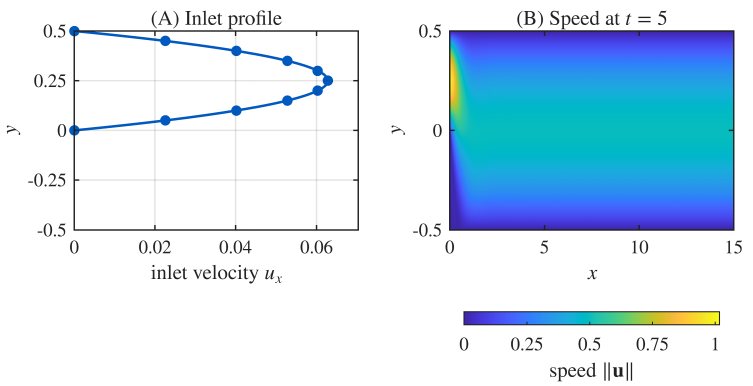}
\caption{Backward-facing step. (A) Inflow profile. (B) Speed at $t=5$, with the transverse direction expanded.}
\label{fig:appendix-backward-step-problem}
\end{figure}

\subsubsection{Buoyancy-driven cavity}
\label{app:buoyancy-cavity}
In the cavity $\Omega_a=\Omega_v=[0,1]^2$, an initial temperature field $T(x,0)$ drives the flow by buoyancy. We learn the map from $T(x,0)$ to the velocity at $t=5$ on $m=5000$ Fekete points selected from the $10514$ mesh points.
\begin{figure}[pos=tbp]
\centering
\includegraphics{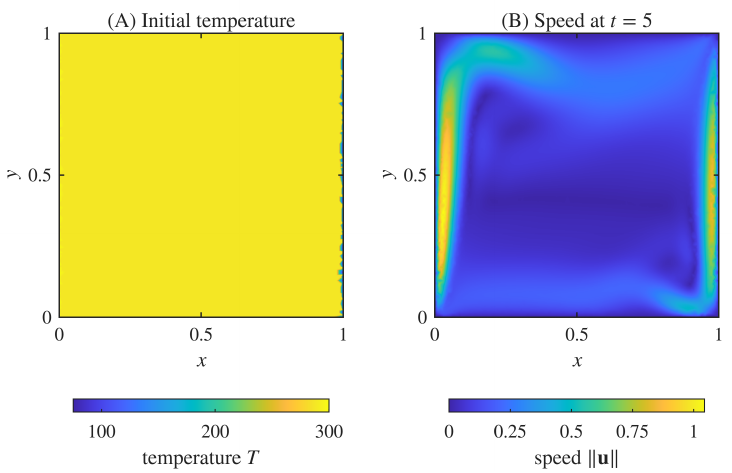}
\caption{Buoyancy-driven cavity. (A) Initial temperature. (B) Speed at $t=5$.}
\label{fig:appendix-buoyancy-problem}
\end{figure}

\subsubsection{Merging vortices}
\label{app:merging-vortices}
Two vortices merge on the periodic domain $\Omega_a=\Omega_v=[0,2\pi]^2$, and the width of their cores varies across the data set. We learn $\G:\omega(x,0)\rightarrow\bu(y,0.4)$ from the initial vorticity on $500$ Fekete points selected from the $7447$ mesh points.
\begin{figure}[pos=tbp]
\centering
\includegraphics{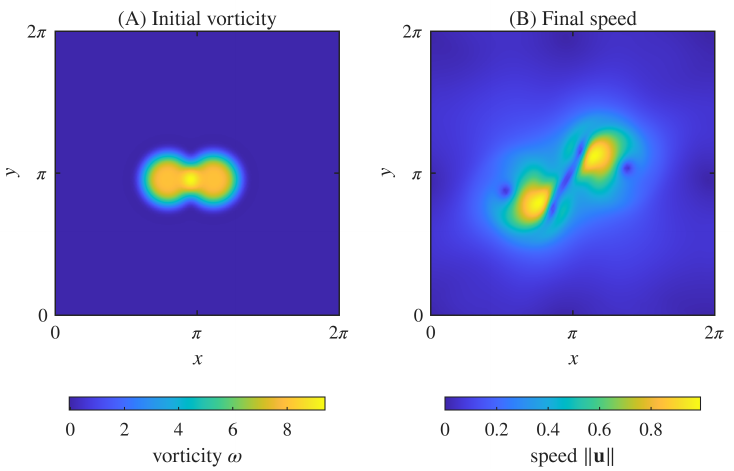}
\caption{Merging vortices. (A) Initial vorticity. (B) Speed at $t=0.4$.}
\label{fig:appendix-merging-vortices-problem}
\end{figure}

\subsubsection{Turbulent species transport}
\label{sec:results_species}
Air and methane mix turbulently in a static Kenics mixer, whose three blades twist along the $z$ axis, at Reynolds numbers between $10^5$ and $10^6$. We learn the map from the inlet velocity ($1388$ samples) to the velocity at $t=0.5$ on $7000$ Fekete points selected from $11388$ locations.

\begin{figure}[tbp]
    \centering
    \begin{subfigure}[t]{0.37\textwidth}
        \centering
        \caption{Inlet velocity}
        \includegraphics[width=\linewidth]{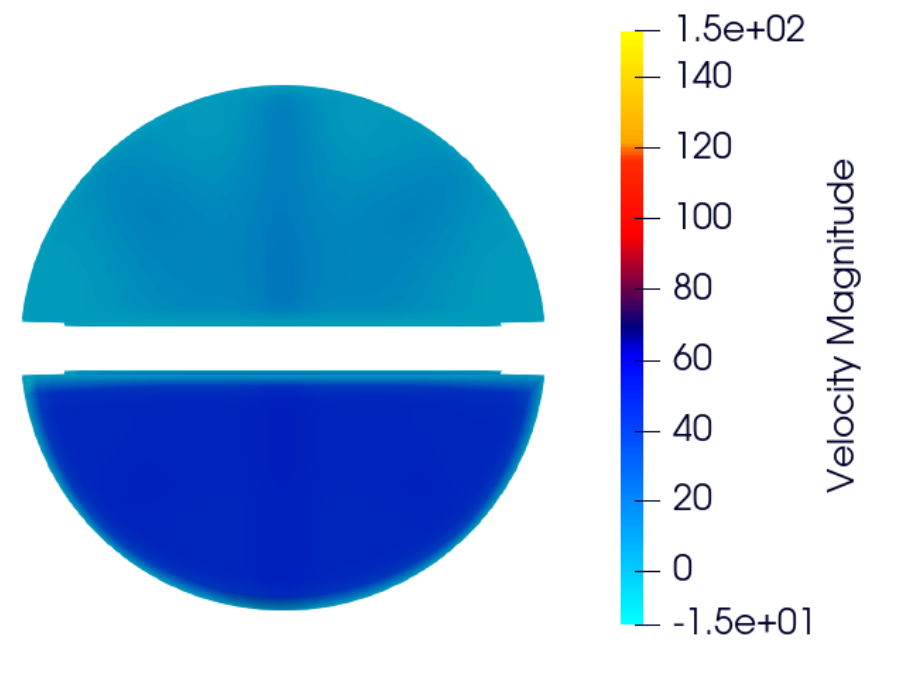}
        \label{fig:species_inlet}
    \end{subfigure}
    \hfill
    \begin{subfigure}[t]{0.65\textwidth}
        \centering
        \caption{Speed at $t=0.5$}
        \includegraphics[width=\linewidth]{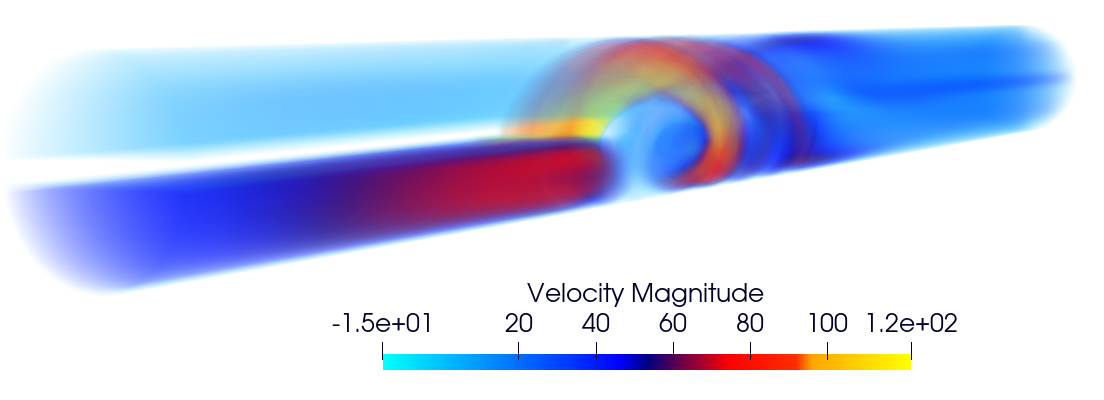}
        \label{fig:species_speed}
    \end{subfigure}
    \caption{Turbulent species transport. (A) Inlet velocity on the plane $z=0$. (B) Speed at $t=0.5$ on the whole domain.}
    \label{fig:species_problem}
\end{figure}

\subsubsection{Forced homogeneous isotropic turbulence}
\label{sec:results_forced_turbulence}
Forced isotropic turbulence evolves in the periodic box $\Omega_a=\Omega_v=[0,2\pi]^3$, simulated by large-eddy simulation with the Smagorinsky model on $32^3$ cells, with $\Delta t=0.002$, final time $T=0.5$, and Reynolds number $6\times10^6$. Every simulation starts from the divergence-free field
\begin{align}
u_x(x,0)&=\sin(2x_3)+\cos(2x_2)+\tfrac12\bigl(\sin(3x_3)+\cos(3x_2)\bigr)+\tfrac14\bigl(\sin(5x_3)+\cos(5x_2)\bigr),\notag\\
u_y(x,0)&=\sin(2x_1)+\cos(2x_3)+\tfrac12\bigl(\sin(3x_1)+\cos(3x_3)\bigr)+\tfrac14\bigl(\sin(5x_1)+\cos(5x_3)\bigr),\notag\\
u_z(x,0)&=\sin(2x_2)+\cos(2x_1)+\tfrac12\bigl(\sin(3x_2)+\cos(3x_1)\bigr)+\tfrac14\bigl(\sin(5x_2)+\cos(5x_1)\bigr),
\end{align}
and is driven by the steady divergence-free forcing
\begin{equation}
\mathbf f(x)=\begin{bmatrix}
c_1\bigl(\sin(q_1x_3)+\cos(q_2x_2)\bigr)\\
c_2\bigl(\sin(q_1x_1)+\cos(q_2x_3)\bigr)\\
c_3\bigl(\sin(q_1x_2)+\cos(q_2x_1)\bigr)
\end{bmatrix},
\end{equation}
whose amplitudes $c_1,c_2,c_3$ are drawn from a normal distribution with standard deviation $0.15$ truncated to $[-0.3,0.3]$, and whose distinct wavenumbers $q_1,q_2$ are drawn from $\{1,2,3\}$. We learn the map from the sampled forcing to the velocity at $t=0.5$ on $7000$ Fekete points of the grid.

\begin{figure}[tbp]
    \centering
    \begin{subfigure}[t]{0.46\textwidth}
        \centering
        \caption{Forcing}
        \includegraphics[width=\linewidth]{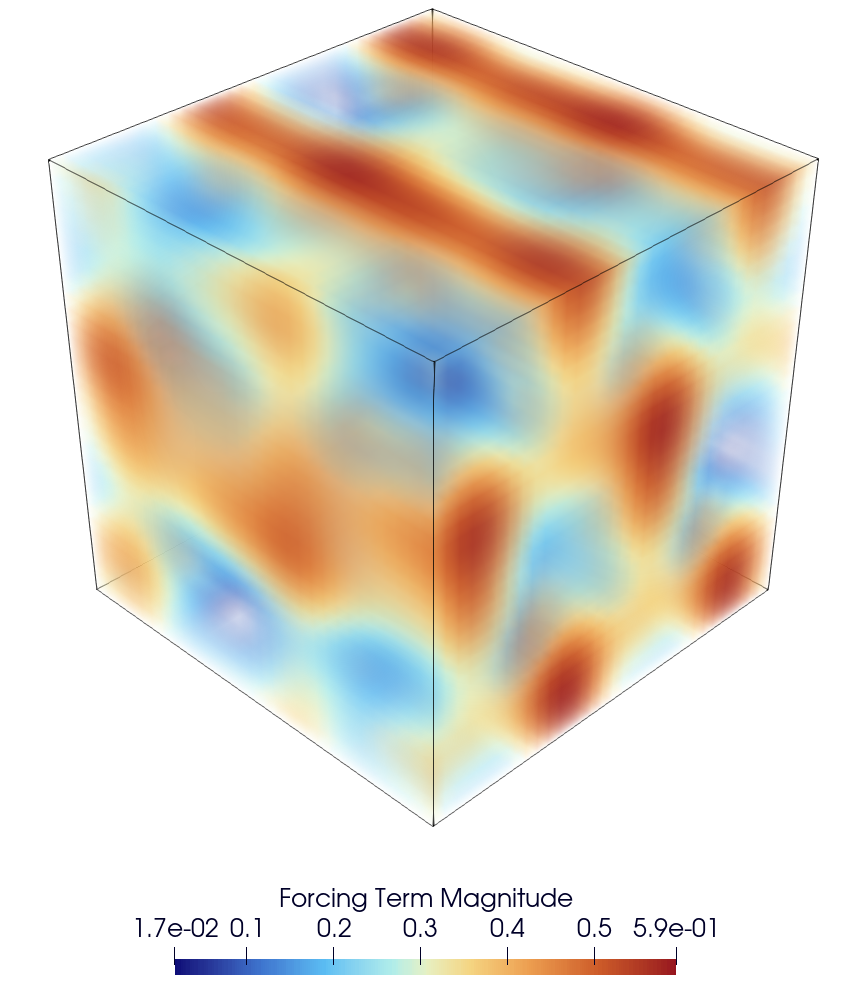}
        \label{fig:forced_turbulence_forcing}
    \end{subfigure}
    \hfill
    \begin{subfigure}[t]{0.46\textwidth}
        \centering
        \caption{Speed at $t=0.5$}
        \includegraphics[width=\linewidth]{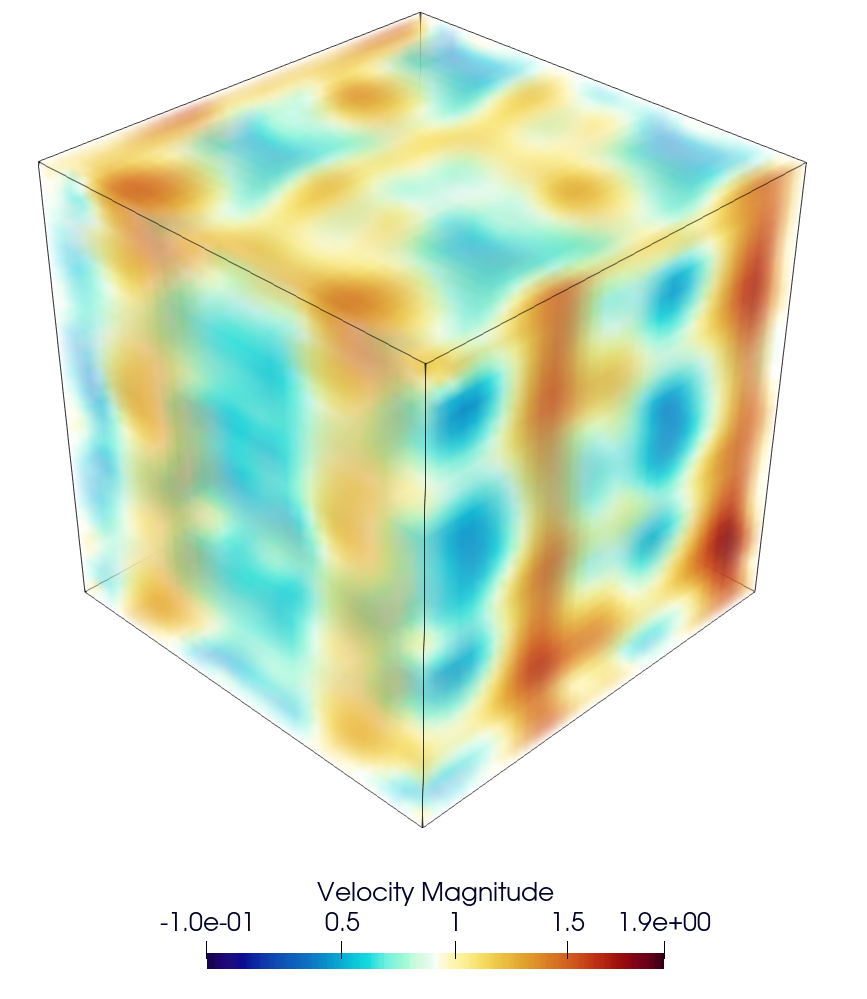}
        \label{fig:forced_turbulence_speed}
    \end{subfigure}
    \caption{Forced isotropic turbulence. (A) Forcing magnitude. (B) Speed at $t=0.5$.}
    \label{fig:forced_turbulence_problem}
\end{figure}

\input{sections/spatial_convergence_setup}

\subsection{Field figures}
\label{app:reference-visualization}
The field figures show the first simulation of each data set and are drawn by the MATLAB scripts in our repository. In two dimensions, the speed $\|\bu\|_2$ is interpolated linearly from the samples. Panels of the same quantity share a color scale, and the transverse direction of the backward-facing step is expanded.

\subsection{Direct and Lagrange evaluation}
\label{app:lagrange-numerics}
Table~\ref{tab:lagrange-all} evaluates seven unaugmented operator models for each data set and each $N\in\{500,1000,5000,10000\}$: the fixed-shape kernel at the radii of the five prescribed densities of Section~\ref{sec:kernel-training} (from $10\%$ to $90\%$, with $\theta=10^{-6}$) and at its validated radius and regularization, and the fitted-shape kernel at its validated parameters. We rebuild each model from its selected parameters and check that its Lagrange predictions at the nodes agree with those of Section~\ref{sec:regression-results} (largest relative difference $1.0\times10^{-10}$). Each model is then evaluated through its coefficients and through its Lagrange weights, with the same normalization, factorization (cuDSS with nested dissection), support, and regularization, and with the same spatial basis. With the local output reconstruction, direct evaluation solves the kernel system for the coefficients of all stencils, in blocks of stencils, and Lagrange evaluation applies the weights to the same coefficients. The global output reconstruction is that of Appendix~\ref{app:spatial-supporting}, on the nodes. The test errors are measured at the nodes on the $100$ test fields.

\input{sections/lagrange_matched_control_table}

Table~\ref{tab:lagrange-matched-control} follows the validated fixed-shape model of the lid-driven cavity with global output reconstruction as $N$ grows. The Frobenius norm of the spatial coefficients of the training fields grows from $1.6\times10^{9}$ to $7.1\times10^{9}$, which shows the cancellation that Theorem~\ref{prop:lagrange-error} warns about, and the ratio of the direct to the Lagrange error grows from $1.3$ to $5.1$. Direct evaluation applied to the sampled velocities instead of the coefficients is as accurate as Lagrange evaluation, so the gain arises in the output reconstruction, through the large coefficients.

\subsection{Software}
\label{app:software}
Our implementation and the scripts of this paper are part of KernelPack-JAX (\url{https://github.com/VarShankar/kernelpack-jax}). The operator models of Section~\ref{sec:regression-results} were fitted with \nolinkurl{scripts/run_fastol_gpu_campaign.py}. The comparisons of Section~\ref{sec:spatial-ranks}, Section~\ref{sec:lagrange_accuracy}, and Appendices~\ref{app:spatial-supporting} and~\ref{app:lagrange-numerics} rebuild these models with the released code, whose predictions agree with those of the original fits to $1.0\times10^{-10}$ (Appendix~\ref{app:lagrange-numerics}), and run on the same GPU with \nolinkurl{scripts/fastol_gpu_evaluation_study.py}. The directory \texttt{FastOLPaper} lists the scripts that produce every result and holds the scripts that build the tables and figures from them.

%% file: sections/fig_structure_functions.tex
\begin{tikzpicture}\begin{loglogaxis}[width=0.46\linewidth, height=0.36\linewidth, title={Turbulent species transport}, xlabel={$r$}, ylabel={$S_2(r)$}, legend pos=south east, log ticks with fixed point, xtick={0.003,0.01,0.03,0.1}, ytick={200,300,500,800}, minor tick num=0, legend style={font=\scriptsize}, tick label style={font=\scriptsize}, label style={font=\small}]
\addplot[black, thick] coordinates {(0.00267945,233.525) (0.00367148,300.044) (0.00503081,373.901) (0.0068934,477.463) (0.0094456,576.747) (0.0129427,711.577) (0.0177346,688.478) (0.0243006,674.043) (0.0332976,694.515) (0.0456257,683.447) (0.0625181,623.608) (0.0856646,786.494)};\addlegendentry{Reference}
\addplot[blue, mark=o, mark size=1.2pt] coordinates {(0.00267945,209.174) (0.00367148,295.604) (0.00503081,376.822) (0.0068934,478.038) (0.0094456,584.913) (0.0129427,720.365) (0.0177346,698.742) (0.0243006,686.879) (0.0332976,702.295) (0.0456257,691.376) (0.0625181,628.157) (0.0856646,787.672)};\addlegendentry{Local PHS}
\addplot[red, mark=triangle, mark size=1.4pt] coordinates {(0.00267945,225.466) (0.00367148,311.019) (0.00503081,388.105) (0.0068934,493.681) (0.0094456,593.289) (0.0129427,726.676) (0.0177346,704.944) (0.0243006,694.085) (0.0332976,713.454) (0.0456257,696.787) (0.0625181,635.143) (0.0856646,805.083)};\addlegendentry{Local multiscale}
\end{loglogaxis}\end{tikzpicture}\hfill\begin{tikzpicture}\begin{loglogaxis}[width=0.46\linewidth, height=0.36\linewidth, title={Forced isotropic turbulence}, xlabel={$r$}, ylabel={$S_2(r)$}, legend pos=south east, log ticks with fixed point, xtick={0.7,1,2,3}, ytick={0.5,0.7,1,1.4}, minor tick num=0, legend style={font=\scriptsize}, tick label style={font=\scriptsize}, label style={font=\small}]
\addplot[black, thick] coordinates {(0.696751,0.513738) (0.79759,0.590113) (0.913024,0.724874) (1.04516,0.829919) (1.19643,1.00962) (1.36959,1.18385) (1.5678,1.28983) (1.79471,1.40097) (2.05445,1.42484) (2.35179,1.38284) (2.69216,1.34243) (3.08179,1.33433)};\addlegendentry{Reference}
\addplot[blue, mark=o, mark size=1.2pt] coordinates {(0.696751,0.50815) (0.79759,0.58541) (0.913024,0.720852) (1.04516,0.825244) (1.19643,1.00516) (1.36959,1.17946) (1.5678,1.28609) (1.79471,1.39754) (2.05445,1.41939) (2.35179,1.37825) (2.69216,1.33845) (3.08179,1.3293)};\addlegendentry{Local PHS}
\addplot[red, mark=triangle, mark size=1.4pt] coordinates {(0.696751,0.510563) (0.79759,0.587428) (0.913024,0.723012) (1.04516,0.827274) (1.19643,1.00722) (1.36959,1.18167) (1.5678,1.28841) (1.79471,1.39948) (2.05445,1.42218) (2.35179,1.38054) (2.69216,1.34077) (3.08179,1.33206)};\addlegendentry{Local multiscale}
\end{loglogaxis}\end{tikzpicture}

%% file: sections/spatial_convergence_setup.tex
\subsection{Manufactured tests of the output reconstruction}
\label{app:manufactured-spatial}

\subsubsection{Domains and fields}
The convergence tests of Section~\ref{sec:spatial-convergence} use the domains $\Omega_d=\{F_d>0\}$ with
\begin{align}
F_2(x,y)&=1-x^2-1.4y^2-0.15x^4 \notag\\
&\quad-0.10x^2y^2-0.20y^4+0.10x^3,\\
F_3(x,y,z)&=1-x^2-1.3y^2-1.7z^2-0.10(x^4+y^4+z^4) \notag\\
&\quad-0.08x^2y^2-0.06y^2z^2+0.10xyz.
\end{align}
In 2D, the field is $v=(\partial_y\psi,-\partial_x\psi)$ with
\begin{align}
\psi=F_2^2\bigl[\sin(\pi x)\sin(2\pi y)+0.2\cos(3\pi x)\sin(\pi y)\bigr],
\end{align}
and in 3D, $v=\nabla\times A$ with
\begin{align}
A=F_3^2\bigl(\sin(\pi y)\sin(2\pi z),\sin(\pi z)\sin(2\pi x),
\sin(\pi x)\sin(2\pi y)\bigr)^{\mathsf T}.
\end{align}
Both fields are divergence free and vanish on the boundary.

The node generator of KernelPack-JAX~\citep{cattaneo2026kernelpack} produced four node sets in each domain, at spacings $0.103$, $0.0675$, $0.0445$, and $0.0295$ in 2D ($172$ to $1907$ nodes) and $0.22$, $0.16$, $0.118$, and $0.088$ in 3D ($197$ to $2570$ nodes). The errors use fixed sets of $4096$ (2D) and $32768$ (3D) uniformly random interior points. For degrees $\varpi=1,3,5$, the divergence-free polynomial spaces have dimensions $n_P=5,14,27$ in 2D and $11,50,133$ in 3D, and every stencil keeps the full space. At the finest level, the sixth-order errors are about $4.5\times10^{-6}$ in 2D and $1.6\times10^{-3}$ in 3D for both kernels.

\subsubsection{Divergence and periodicity}
On the evaluation points where the output reconstruction is differentiable (i.e., away from the centers of the cubic kernel), we measure the divergence relative to the size of the velocity gradient,
\begin{align}
D_{\rm rms}=
\frac{\bigl(|Q_D|^{-1}\sum_{y\in Q_D}|\nabla\!\cdot v_h(y)|^2\bigr)^{1/2}}
{\max\{1,\bigl(|Q_D|^{-1}\sum_{y\in Q_D}\|Dv_h(y)\|_F^2\bigr)^{1/2}\}},
\end{align}
with the derivative of each local expansion taken on its own stencil. Its largest value over all smooth tests is $1.5\times10^{-13}$.

The periodic tests on $[0,2\pi)^d$ use the potentials
\begin{align}
\psi&=\sin x\cos(2y)+0.2\cos(3x)\sin y,\qquad d=2,\\
A&=(\sin y\sin(2z),\sin z\sin(2x),\sin x\sin(2y))^{\mathsf T},\qquad d=3,
\end{align}
on $657$ (2D) and $1042$ (3D) scattered nodes. We compare the reconstructed velocities and Jacobians at points that differ by integer multiples of the period, on the faces, edges, and corners of the box and at $128$ interior points. The largest differences are $1.4\times10^{-12}$ for the velocities and $1.6\times10^{-12}$ for the Jacobians.

\subsubsection{The turbulent ensemble}
\label{app:k41-ensemble}
The statistical test of Section~\ref{sec:spatial-convergence} uses the solenoidal Fourier ensemble
\begin{align}
v(y)=\sum_{\mathbf k\in\mathcal K}\bigl[\mathbf c_{\mathbf k}\cos(\mathbf k\cdot y)+\mathbf s_{\mathbf k}\sin(\mathbf k\cdot y)\bigr],
\qquad
\mathbb E[\mathbf c_{\mathbf k}\mathbf c_{\mathbf k}^{\mathsf T}]=\mathbb E[\mathbf s_{\mathbf k}\mathbf s_{\mathbf k}^{\mathsf T}]
=w_{\mathbf k}\left(I-\frac{\mathbf k\mathbf k^{\mathsf T}}{\|\mathbf k\|_2^2}\right),
\end{align}
with independent Gaussian coefficients perpendicular to their wave vectors, $1514$ modes up to wavenumber $1024$, and weights $w_{\mathbf k}$ that integrate a $\|\mathbf k\|_2^{-5/3}$ energy spectrum over radial shells, normalized so that $2\sum_{\mathbf k}w_{\mathbf k}=1$. Its covariance is divergence free, and its structure function approximately follows the Kolmogorov law over the separations $0.03\leq r\leq0.8$, which we fixed before any output reconstruction. At smaller separations, the finite bandwidth makes the increments smooth. We restrict the ensemble to $\Omega_3$ and reconstruct it on six node sets with spacings $0.26$, $0.19$, $0.14$, $0.101$, $0.075$, and $0.05$ ($128$ to $13567$ nodes). The structure functions~\eqref{eq:spatial-k41-s2} use $32$ separations from $10^{-5}$ to $0.8$ in eight fixed directions and are computed exactly from the covariance, not from samples. For degree five, $E_{S_2}(0.8)$ falls from $14\%$ on the coarsest set to $4.6\%$ on the finest, while the worst error over the whole range remains $64\%$.

For the comparison of local and global output reconstruction, we use the node set with $1770$ nodes and one fixed five-scale Mat\'ern kernel for both, with the same polynomials. The global system uses the covariance anchored at a point $a$, $K(x,y)-K(x,a)-K(a,y)+K(a,a)$, which is positive definite and gives the same interpolant, since the constant polynomials absorb the difference. For degree five, the errors $E_{S_2}(0.8)$ are $9.9\%$ (local) and $10.6\%$ (global), the normalized mean square errors of the increments are $0.168$ and $0.166$, and the worst errors over the whole range are about $87\%$ for both.

%% file: sections/lagrange_matched_control_table.tex
\begin{table}[pos=!htbp]
\centering\small
\caption{Test error at the nodes (mean relative velocity error over the 100 test inputs) of direct and Lagrange evaluation of the validated unaugmented fixed-shape model of the lid-driven cavity with global output reconstruction, of direct evaluation applied to the sampled velocities instead of the coefficients, and the Frobenius norm of the spatial coefficients of the training fields. Bold: smallest test error in each row.}
\label{tab:lagrange-matched-control}
\begin{tabular}{rrrrrr}
\toprule
$N$ & Direct & Lagrange & Direct/Lagrange & Direct on samples & $\|B\|_F$ \\
\midrule
500 & $1.54\times10^{-5}$ & $\mathbf{1.20\times10^{-5}}$ & 1.28 & $\mathbf{1.20\times10^{-5}}$ & $1.58\times10^{9}$ \\
1000 & $8.39\times10^{-6}$ & $4.25\times10^{-6}$ & 1.97 & $\mathbf{4.24\times10^{-6}}$ & $2.24\times10^{9}$ \\
5000 & $3.19\times10^{-6}$ & $9.54\times10^{-7}$ & 3.34 & $\mathbf{9.32\times10^{-7}}$ & $5.00\times10^{9}$ \\
10000 & $3.42\times10^{-6}$ & $6.64\times10^{-7}$ & 5.14 & $\mathbf{6.42\times10^{-7}}$ & $7.08\times10^{9}$ \\
\bottomrule
\end{tabular}
\end{table}

%% file: sections/gpu_results_appendix.tex
\subsection{All operator benchmarks}
\label{app:gpu-results}

This section collects the results of all ten benchmarks for every training size. Figures~\ref{fig:gpu-error-additional} and~\ref{fig:gpu-warm-additional} complete Figures~\ref{fig:gpu-error-main} and~\ref{fig:gpu-warm-main}, Figures~\ref{fig:gpu-density-additional} and~\ref{fig:gpu-storage-additional} complete Figures~\ref{fig:gpu-density-main} and~\ref{fig:gpu-storage-main}, and Table~\ref{tab:gpu-decay-rates} gives the rates at which the errors decrease with $N$. Tables~\ref{tab:gpu-accuracy-n10000} and~\ref{tab:gpu-cost-n10000} collect the errors, storage, and prediction times at $N=10000$, and Tables~\ref{tab:gpu-all-forced-turb}--\ref{tab:gpu-all-buoyancy-cavity-flow} list all $160$ fitted models. Two entries show why we report storage separately from density. The fixed-shape and fitted-shape augmented species models have densities of $9.1\%$ and $21.7\%$ but $26$ and $35$ million factor entries, and the factor of the backward-facing step, at $33\%$ density, has $38$ million entries, three quarters of a dense one. All cost columns use the nested-dissection ordering of Section~\ref{sec:operator-factors}, measured by a replay of the selected models (\texttt{scripts/fastol\_gpu\_cost\_replay.py}). The test errors were computed with the approximate-minimum-degree ordering, and the predictions with the two orderings agree to $1.2\times10^{-10}$.

\input{sections/gpu_decay_rates}
\input{sections/gpu_accuracy_n10000}
\input{sections/gpu_cost_n10000}

\begin{figure}[pos=tbp]
\centering
\includegraphics{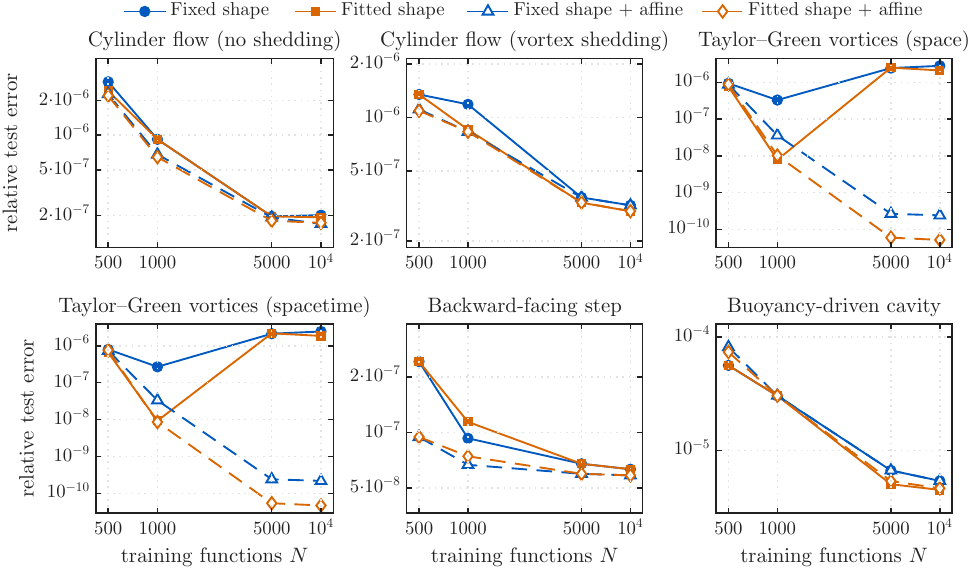}
\caption{Test error of the other six benchmarks. Dashed: with affine augmentation.}
\label{fig:gpu-error-additional}
\end{figure}

\begin{figure}[pos=tbp]
\centering
\includegraphics{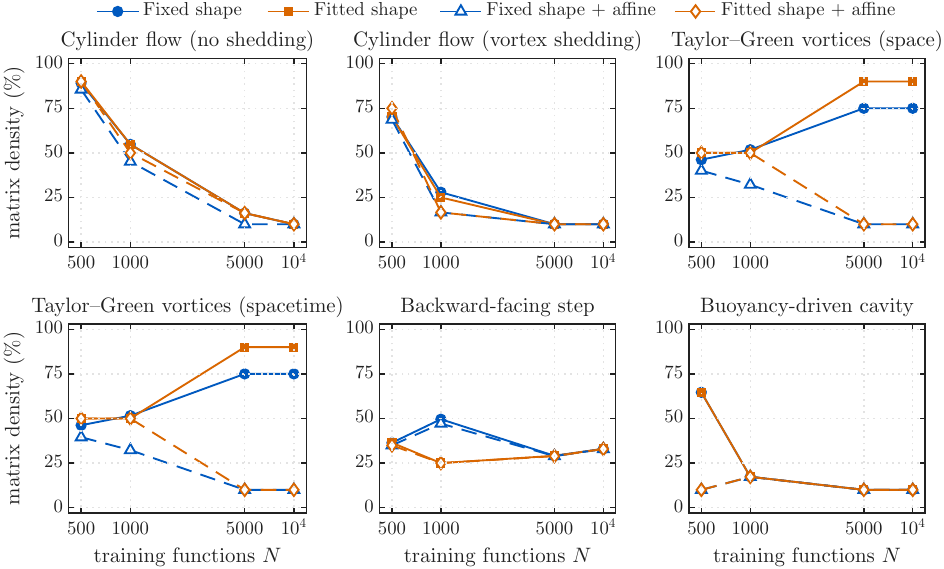}
\caption{Selected density of the other six benchmarks. With augmentation, the Taylor--Green models select about $10\%$ density at the two largest $N$, against $75\%$ or $90\%$ without it.}
\label{fig:gpu-density-additional}
\end{figure}

\begin{figure}[pos=tbp]
\centering
\includegraphics{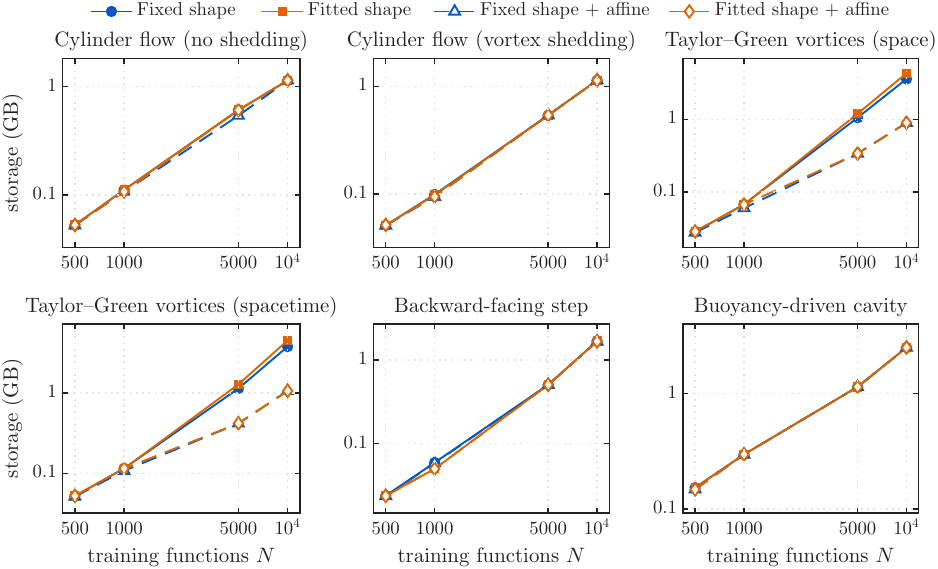}
\caption{Accounted storage of the other six benchmarks (Appendix~\ref{app:cost-details}).}
\label{fig:gpu-storage-additional}
\end{figure}

\begin{figure}[pos=tbp]
\centering
\includegraphics{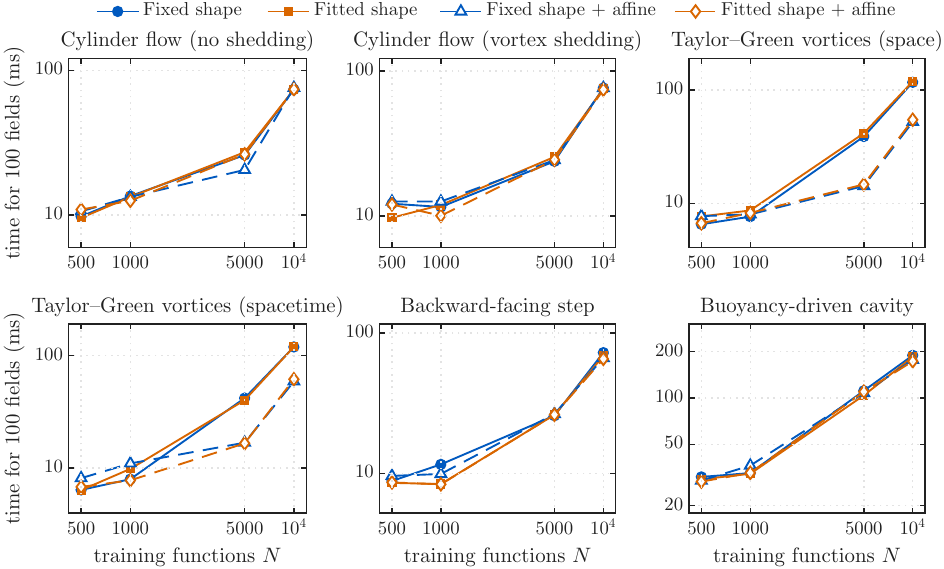}
\caption{Median warm time to predict all $100$ test fields for the other six benchmarks.}
\label{fig:gpu-warm-additional}
\end{figure}

\input{sections/gpu_all_models_tables}

\subsubsection{Offline cost}
\label{app:gpu-offline}
Table~\ref{tab:gpu-offline-n10000} gives the offline costs at $N=10000$. Fitting and validating a fitted-shape model takes $11$--$27$ minutes, against $2$--$7$ minutes for a fixed-shape model. Once the parameters are chosen, rebuilding the factorization takes about one to five seconds for most benchmarks and about $43$ seconds for the dense forced-turbulence models. The setup of the output reconstruction, which the four models share, takes $0.03$--$2.9$ seconds after compilation, and a first setup, which also compiles, takes $1.6$--$24$ seconds.

\input{sections/gpu_offline_n10000}
\FloatBarrier

%% file: sections/gpu_decay_rates.tex
\begin{table}[pos=!htbp]
\centering\small
\setlength{\tabcolsep}{4pt}
\caption{Least-squares exponent $\tau$ of the test error $\propto N^{-\tau}$ over $N\in\{500,1000,5000,10000\}$. Bold: fastest decay in each row.}
\label{tab:gpu-decay-rates}
\begin{tabular}{lrrrr}
\toprule
Data set & Fixed & Fitted & Fixed + affine & Fitted + affine \\
\midrule
Forced isotropic turbulence & 0.68 & 0.69 & \textbf{0.76} & 0.66 \\
Turbulent species transport & 0.85 & 1.09 & 1.03 & \textbf{1.26} \\
Merging vortices & \textbf{0.77} & 0.73 & 0.57 & 0.53 \\
Lid-driven cavity & 0.97 & \textbf{1.09} & 0.96 & 1.08 \\
Cylinder flow (no shedding) & \textbf{0.91} & 0.87 & 0.85 & 0.84 \\
Cylinder flow (vortex shedding) & \textbf{0.54} & 0.53 & 0.44 & 0.47 \\
Backward-facing step & 0.39 & \textbf{0.42} & 0.14 & 0.15 \\
Buoyancy-driven cavity & 0.82 & 0.91 & 0.92 & \textbf{0.96} \\
Taylor--Green vortices (space) & $-$0.56 & $-$1.06 & 2.81 & \textbf{3.24} \\
Taylor--Green vortices (spacetime) & $-$0.58 & $-$1.03 & 2.79 & \textbf{3.22} \\
\bottomrule
\end{tabular}
\end{table}

%% file: sections/gpu_accuracy_n10000.tex
\begin{table}[pos=!htbp]
\centering\small
\setlength{\tabcolsep}{4pt}
\caption{Test error (mean relative velocity error of the predictions for the 100 test inputs) at $N=10000$. The two augmented models use the same affine columns. Each model is fitted and validated separately, and all four use the same local output reconstruction. Both Taylor--Green cases use the amplitude normalization~\eqref{eq:tg-response-transform}. Bold: smallest test error in each row.}
\label{tab:gpu-accuracy-n10000}
\begin{tabular}{lrrrr}
\toprule
Data set & Fixed & Fitted & Fixed + affine & Fitted + affine \\
\midrule
Forced isotropic turbulence & $2.80\times10^{-3}$ & $1.62\times10^{-3}$ & $9.77\times10^{-5}$ & $\mathbf{7.14\times10^{-5}}$ \\
Turbulent species transport & $4.34\times10^{-4}$ & $6.86\times10^{-5}$ & $1.89\times10^{-4}$ & $\mathbf{3.69\times10^{-5}}$ \\
Merging vortices & $\mathbf{3.02\times10^{-7}}$ & $3.07\times10^{-7}$ & $3.39\times10^{-7}$ & $3.19\times10^{-7}$ \\
Lid-driven cavity & $6.42\times10^{-7}$ & $4.76\times10^{-7}$ & $6.42\times10^{-7}$ & $\mathbf{4.72\times10^{-7}}$ \\
Cylinder flow (no shedding) & $2.02\times10^{-7}$ & $1.94\times10^{-7}$ & $\mathbf{1.70\times10^{-7}}$ & $1.73\times10^{-7}$ \\
Cylinder flow (vortex shedding) & $3.19\times10^{-7}$ & $\mathbf{2.95\times10^{-7}}$ & $3.19\times10^{-7}$ & $2.96\times10^{-7}$ \\
Taylor--Green vortices (space) & $2.82\times10^{-6}$ & $2.11\times10^{-6}$ & $2.42\times10^{-10}$ & $\mathbf{5.22\times10^{-11}}$ \\
Taylor--Green vortices (spacetime) & $2.46\times10^{-6}$ & $1.86\times10^{-6}$ & $2.18\times10^{-10}$ & $\mathbf{4.70\times10^{-11}}$ \\
Backward-facing step & $6.33\times10^{-8}$ & $6.33\times10^{-8}$ & $\mathbf{5.86\times10^{-8}}$ & $\mathbf{5.86\times10^{-8}}$ \\
Buoyancy-driven cavity & $5.38\times10^{-6}$ & $\mathbf{4.47\times10^{-6}}$ & $5.38\times10^{-6}$ & $4.62\times10^{-6}$ \\
\bottomrule
\end{tabular}
\end{table}

%% file: sections/gpu_cost_n10000.tex
\begin{table}[pos=!htbp]
\centering\small
\setlength{\tabcolsep}{3pt}
\caption{Accounted storage in decimal GB / median warm prediction time in ms for 100 fields at $N=10000$. Storage includes the stored model arrays, the solve buffers, and cuDSS's estimate of the memory of its factors. It is not measured peak GPU memory. Times include operator Lagrange evaluation and local output reconstruction. Bold: smallest storage and smallest time in each row.}
\label{tab:gpu-cost-n10000}
\begin{tabular}{lrrrr}
\toprule
Data set & Fixed & Fitted & Fixed + affine & Fitted + affine \\
\midrule
Forced isotropic turbulence & \textbf{10.680} / 235.2 & \textbf{10.680} / 234.0 & 10.685 / 244.1 & 10.682 / \textbf{233.4} \\
Turbulent species transport & 2.662 / 126.6 & 3.130 / 126.0 & \textbf{2.630} / \textbf{123.3} & 3.137 / 137.4 \\
Merging vortices & 1.397 / 79.6 & 1.029 / 62.4 & 0.772 / \textbf{37.7} & \textbf{0.767} / 39.0 \\
Lid-driven cavity & \textbf{1.135} / 75.9 & \textbf{1.135} / 71.5 & 1.136 / \textbf{70.2} & 1.136 / 70.3 \\
Cylinder flow (no shedding) & \textbf{1.141} / 73.8 & \textbf{1.141} / \textbf{73.7} & 1.142 / 75.7 & 1.142 / 74.0 \\
Cylinder flow (vortex shedding) & \textbf{1.140} / 76.2 & \textbf{1.140} / 75.3 & 1.141 / 76.2 & 1.141 / \textbf{74.2} \\
Taylor--Green vortices (space) & 3.606 / 117.0 & 4.334 / 119.0 & \textbf{0.899} / \textbf{52.3} & \textbf{0.899} / 54.6 \\
Taylor--Green vortices (spacetime) & 3.766 / 119.8 & 4.494 / 120.1 & \textbf{1.059} / \textbf{59.0} & \textbf{1.059} / 61.7 \\
Backward-facing step & 1.679 / 72.4 & 1.679 / 68.8 & \textbf{1.657} / 66.3 & \textbf{1.657} / \textbf{65.3} \\
Buoyancy-driven cavity & \textbf{2.504} / 189.4 & \textbf{2.504} / 184.9 & 2.505 / 177.0 & 2.505 / \textbf{173.4} \\
\bottomrule
\end{tabular}
\end{table}

%% file: sections/gpu_all_models_tables.tex
\begin{table}[pos=!htbp]
\centering\small
\setlength{\tabcolsep}{3.5pt}
\caption{Forced isotropic turbulence: all models. Test error is the mean relative velocity error of the predictions for the 100 test inputs. Density is $d_K=\mathrm{nnz}(K_{UU})/N^2$ in percent. Factor is the number of entries of the cuDSS Cholesky factor in millions. GB is the accounted storage and ms the median warm time for all 100 fields. Bold: smallest test error at each $N$.}
\label{tab:gpu-all-forced-turb}
\begin{tabular}{rlrrrrr}
\toprule
$N$ & Operator model & Test error & Density (\%) & Factor ($10^6$) & GB & ms \\
\midrule
500 & Fixed shape & 0.0214 & 90.00 & 0.125 & 0.465 & 60.9 \\
500 & Fitted shape & 0.013 & 90.00 & 0.125 & 0.465 & 64.3 \\
500 & Fixed shape + affine & $9.37\times10^{-4}$ & 90.00 & 0.125 & 0.467 & 58.5 \\
500 & Fitted shape + affine & $\mathbf{4.90\times10^{-4}}$ & 90.00 & 0.126 & 0.467 & 56.5 \\
\addlinespace[3pt]
1000 & Fixed shape & 0.0157 & 90.00 & 0.990 & 0.750 & 64.8 \\
1000 & Fitted shape & 0.0139 & 90.00 & 0.991 & 0.750 & 70.2 \\
1000 & Fixed shape + affine & $6.05\times10^{-4}$ & 90.00 & 0.990 & 0.752 & 65.2 \\
1000 & Fitted shape + affine & $\mathbf{3.43\times10^{-4}}$ & 90.00 & 0.991 & 0.752 & 65.5 \\
\addlinespace[3pt]
5000 & Fixed shape & $5.16\times10^{-3}$ & 90.00 & 24.722 & 4.390 & 126.8 \\
5000 & Fitted shape & $4.69\times10^{-3}$ & 90.00 & 24.748 & 4.390 & 127.2 \\
5000 & Fixed shape + affine & $1.67\times10^{-4}$ & 90.00 & 24.722 & 4.394 & 126.9 \\
5000 & Fitted shape + affine & $\mathbf{1.10\times10^{-4}}$ & 90.00 & 24.751 & 4.394 & 126.1 \\
\addlinespace[3pt]
10000 & Fixed shape & $2.80\times10^{-3}$ & 90.00 & 98.899 & 10.680 & 235.2 \\
10000 & Fitted shape & $1.62\times10^{-3}$ & 90.00 & 98.984 & 10.680 & 234.0 \\
10000 & Fixed shape + affine & $9.77\times10^{-5}$ & 90.00 & 98.899 & 10.685 & 244.1 \\
10000 & Fitted shape + affine & $\mathbf{7.14\times10^{-5}}$ & 90.00 & 98.651 & 10.682 & 233.4 \\
\bottomrule
\end{tabular}
\end{table}

\begin{table}[pos=!htbp]
\centering\small
\setlength{\tabcolsep}{3.5pt}
\caption{Turbulent species transport: all models. Test error is the mean relative velocity error of the predictions for the 100 test inputs. Density is $d_K=\mathrm{nnz}(K_{UU})/N^2$ in percent. Factor is the number of entries of the cuDSS Cholesky factor in millions. GB is the accounted storage and ms the median warm time for all 100 fields. Bold: smallest test error at each $N$.}
\label{tab:gpu-all-species-transport}
\begin{tabular}{rlrrrrr}
\toprule
$N$ & Operator model & Test error & Density (\%) & Factor ($10^6$) & GB & ms \\
\midrule
500 & Fixed shape & $4.39\times10^{-3}$ & 63.94 & 0.116 & 0.226 & 54.3 \\
500 & Fitted shape & $2.50\times10^{-3}$ & 90.00 & 0.129 & 0.228 & 52.8 \\
500 & Fixed shape + affine & $3.35\times10^{-3}$ & 40.98 & 0.097 & 0.224 & 53.1 \\
500 & Fitted shape + affine & $\mathbf{2.48\times10^{-3}}$ & 44.60 & 0.096 & 0.224 & 54.0 \\
\addlinespace[3pt]
1000 & Fixed shape & $1.25\times10^{-3}$ & 89.71 & 1.000 & 0.278 & 57.1 \\
1000 & Fitted shape & $9.53\times10^{-4}$ & 85.62 & 0.962 & 0.276 & 57.8 \\
1000 & Fixed shape + affine & $1.49\times10^{-3}$ & 22.92 & 0.313 & 0.250 & 55.6 \\
1000 & Fitted shape + affine & $\mathbf{8.75\times10^{-4}}$ & 75.00 & 0.988 & 0.273 & 59.1 \\
\addlinespace[3pt]
5000 & Fixed shape & $2.02\times10^{-4}$ & 69.06 & 24.743 & 1.858 & 98.2 \\
5000 & Fitted shape & $2.95\times10^{-4}$ & 90.00 & 24.744 & 2.036 & 101.7 \\
5000 & Fixed shape + affine & $\mathbf{1.96\times10^{-4}}$ & 69.56 & 24.689 & 1.862 & 98.5 \\
5000 & Fitted shape + affine & $2.55\times10^{-4}$ & 66.71 & 24.746 & 1.838 & 98.3 \\
\addlinespace[3pt]
10000 & Fixed shape & $4.34\times10^{-4}$ & 10.40 & 24.431 & 2.662 & 126.6 \\
10000 & Fitted shape & $6.86\times10^{-5}$ & 21.95 & 33.613 & 3.130 & 126.0 \\
10000 & Fixed shape + affine & $1.89\times10^{-4}$ & 9.12 & 25.740 & 2.630 & 123.3 \\
10000 & Fitted shape + affine & $\mathbf{3.69\times10^{-5}}$ & 21.71 & 35.385 & 3.137 & 137.4 \\
\bottomrule
\end{tabular}
\end{table}

\begin{table}[pos=!htbp]
\centering\small
\setlength{\tabcolsep}{3.5pt}
\caption{Merging vortices: all models. Test error is the mean relative velocity error of the predictions for the 100 test inputs. Density is $d_K=\mathrm{nnz}(K_{UU})/N^2$ in percent. Factor is the number of entries of the cuDSS Cholesky factor in millions. GB is the accounted storage and ms the median warm time for all 100 fields. Bold: smallest test error at each $N$.}
\label{tab:gpu-all-merge-vortices-easier}
\begin{tabular}{rlrrrrr}
\toprule
$N$ & Operator model & Test error & Density (\%) & Factor ($10^6$) & GB & ms \\
\midrule
500 & Fixed shape & $2.85\times10^{-6}$ & 75.06 & 0.104 & 0.025 & 6.0 \\
500 & Fitted shape & $2.51\times10^{-6}$ & 75.00 & 0.106 & 0.025 & 6.8 \\
500 & Fixed shape + affine & $1.94\times10^{-6}$ & 37.36 & 0.065 & 0.021 & 7.9 \\
500 & Fitted shape + affine & $\mathbf{1.66\times10^{-6}}$ & 33.64 & 0.060 & 0.021 & 6.9 \\
\addlinespace[3pt]
1000 & Fixed shape & $1.41\times10^{-6}$ & 63.81 & 0.386 & 0.060 & 7.3 \\
1000 & Fitted shape & $1.45\times10^{-6}$ & 60.31 & 0.387 & 0.058 & 8.6 \\
1000 & Fixed shape + affine & $8.60\times10^{-7}$ & 29.73 & 0.225 & 0.047 & 9.1 \\
1000 & Fitted shape + affine & $\mathbf{8.46\times10^{-7}}$ & 29.04 & 0.217 & 0.047 & 7.4 \\
\addlinespace[3pt]
5000 & Fixed shape & $\mathbf{3.54\times10^{-7}}$ & 24.58 & 6.404 & 0.433 & 24.9 \\
5000 & Fitted shape & $3.91\times10^{-7}$ & 25.00 & 6.356 & 0.436 & 25.2 \\
5000 & Fixed shape + affine & $3.66\times10^{-7}$ & 24.58 & 6.404 & 0.434 & 24.9 \\
5000 & Fitted shape + affine & $4.08\times10^{-7}$ & 18.33 & 4.892 & 0.369 & 21.7 \\
\addlinespace[3pt]
10000 & Fixed shape & $\mathbf{3.02\times10^{-7}}$ & 24.67 & 26.353 & 1.397 & 79.6 \\
10000 & Fitted shape & $3.07\times10^{-7}$ & 15.85 & 18.009 & 1.029 & 62.4 \\
10000 & Fixed shape + affine & $3.39\times10^{-7}$ & 10.00 & 10.703 & 0.772 & 37.7 \\
10000 & Fitted shape + affine & $3.19\times10^{-7}$ & 10.00 & 10.047 & 0.767 & 39.0 \\
\bottomrule
\end{tabular}
\end{table}

\begin{table}[pos=!htbp]
\centering\small
\setlength{\tabcolsep}{3.5pt}
\caption{Lid-driven cavity: all models. Test error is the mean relative velocity error of the predictions for the 100 test inputs. Density is $d_K=\mathrm{nnz}(K_{UU})/N^2$ in percent. Factor is the number of entries of the cuDSS Cholesky factor in millions. GB is the accounted storage and ms the median warm time for all 100 fields. Bold: smallest test error at each $N$.}
\label{tab:gpu-all-lid-cavity-flow}
\begin{tabular}{rlrrrrr}
\toprule
$N$ & Operator model & Test error & Density (\%) & Factor ($10^6$) & GB & ms \\
\midrule
500 & Fixed shape & $1.20\times10^{-5}$ & 22.85 & 0.053 & 0.046 & 9.2 \\
500 & Fitted shape & $1.20\times10^{-5}$ & 22.85 & 0.053 & 0.046 & 10.8 \\
500 & Fixed shape + affine & $\mathbf{1.15\times10^{-5}}$ & 21.49 & 0.048 & 0.046 & 10.4 \\
500 & Fitted shape + affine & $\mathbf{1.15\times10^{-5}}$ & 25.00 & 0.055 & 0.047 & 9.2 \\
\addlinespace[3pt]
1000 & Fixed shape & $4.24\times10^{-6}$ & 19.99 & 0.190 & 0.096 & 10.6 \\
1000 & Fitted shape & $4.11\times10^{-6}$ & 25.00 & 0.223 & 0.098 & 11.4 \\
1000 & Fixed shape + affine & $4.13\times10^{-6}$ & 20.00 & 0.189 & 0.096 & 10.8 \\
1000 & Fitted shape + affine & $\mathbf{4.10\times10^{-6}}$ & 25.00 & 0.223 & 0.098 & 11.1 \\
\addlinespace[3pt]
5000 & Fixed shape & $9.32\times10^{-7}$ & 10.00 & 2.600 & 0.538 & 34.4 \\
5000 & Fitted shape & $\mathbf{6.73\times10^{-7}}$ & 10.00 & 2.600 & 0.538 & 34.6 \\
5000 & Fixed shape + affine & $9.32\times10^{-7}$ & 10.00 & 2.600 & 0.539 & 34.2 \\
5000 & Fitted shape + affine & $6.80\times10^{-7}$ & 10.00 & 2.600 & 0.539 & 34.7 \\
\addlinespace[3pt]
10000 & Fixed shape & $6.42\times10^{-7}$ & 10.00 & 11.198 & 1.135 & 75.9 \\
10000 & Fitted shape & $4.76\times10^{-7}$ & 10.00 & 11.198 & 1.135 & 71.5 \\
10000 & Fixed shape + affine & $6.42\times10^{-7}$ & 10.00 & 11.198 & 1.136 & 70.2 \\
10000 & Fitted shape + affine & $\mathbf{4.72\times10^{-7}}$ & 10.00 & 11.198 & 1.136 & 70.3 \\
\bottomrule
\end{tabular}
\end{table}

\begin{table}[pos=!htbp]
\centering\small
\setlength{\tabcolsep}{3.5pt}
\caption{Cylinder flow (no shedding): all models. Test error is the mean relative velocity error of the predictions for the 100 test inputs. Density is $d_K=\mathrm{nnz}(K_{UU})/N^2$ in percent. Factor is the number of entries of the cuDSS Cholesky factor in millions. GB is the accounted storage and ms the median warm time for all 100 fields. Bold: smallest test error at each $N$.}
\label{tab:gpu-all-flow-cylinder-laminar}
\begin{tabular}{rlrrrrr}
\toprule
$N$ & Operator model & Test error & Density (\%) & Factor ($10^6$) & GB & ms \\
\midrule
500 & Fixed shape & $2.91\times10^{-6}$ & 89.11 & 0.129 & 0.053 & 9.8 \\
500 & Fitted shape & $2.42\times10^{-6}$ & 90.00 & 0.128 & 0.053 & 9.6 \\
500 & Fixed shape + affine & $2.27\times10^{-6}$ & 85.52 & 0.128 & 0.052 & 10.6 \\
500 & Fitted shape + affine & $\mathbf{2.20\times10^{-6}}$ & 90.00 & 0.128 & 0.053 & 10.9 \\
\addlinespace[3pt]
1000 & Fixed shape & $9.21\times10^{-7}$ & 54.74 & 0.674 & 0.111 & 13.5 \\
1000 & Fitted shape & $9.21\times10^{-7}$ & 54.74 & 0.674 & 0.111 & 13.3 \\
1000 & Fixed shape + affine & $6.80\times10^{-7}$ & 45.10 & 0.581 & 0.107 & 13.4 \\
1000 & Fitted shape + affine & $\mathbf{6.45\times10^{-7}}$ & 50.00 & 0.394 & 0.107 & 12.6 \\
\addlinespace[3pt]
5000 & Fixed shape & $1.97\times10^{-7}$ & 16.22 & 4.555 & 0.607 & 26.0 \\
5000 & Fitted shape & $1.97\times10^{-7}$ & 16.22 & 4.555 & 0.607 & 27.2 \\
5000 & Fixed shape + affine & $1.93\times10^{-7}$ & 10.00 & 2.726 & 0.540 & 20.5 \\
5000 & Fitted shape + affine & $\mathbf{1.81\times10^{-7}}$ & 16.22 & 4.555 & 0.608 & 26.3 \\
\addlinespace[3pt]
10000 & Fixed shape & $2.02\times10^{-7}$ & 10.00 & 11.981 & 1.141 & 73.8 \\
10000 & Fitted shape & $1.94\times10^{-7}$ & 10.00 & 11.981 & 1.141 & 73.7 \\
10000 & Fixed shape + affine & $\mathbf{1.70\times10^{-7}}$ & 10.00 & 11.981 & 1.142 & 75.7 \\
10000 & Fitted shape + affine & $1.73\times10^{-7}$ & 10.00 & 11.981 & 1.142 & 74.0 \\
\bottomrule
\end{tabular}
\end{table}

\begin{table}[pos=!htbp]
\centering\small
\setlength{\tabcolsep}{3.5pt}
\caption{Cylinder flow (vortex shedding): all models. Test error is the mean relative velocity error of the predictions for the 100 test inputs. Density is $d_K=\mathrm{nnz}(K_{UU})/N^2$ in percent. Factor is the number of entries of the cuDSS Cholesky factor in millions. GB is the accounted storage and ms the median warm time for all 100 fields. Bold: smallest test error at each $N$.}
\label{tab:gpu-all-flow-cylinder-shedding}
\begin{tabular}{rlrrrrr}
\toprule
$N$ & Operator model & Test error & Density (\%) & Factor ($10^6$) & GB & ms \\
\midrule
500 & Fixed shape & $1.35\times10^{-6}$ & 70.21 & 0.122 & 0.051 & 12.2 \\
500 & Fitted shape & $1.35\times10^{-6}$ & 70.21 & 0.122 & 0.051 & 9.7 \\
500 & Fixed shape + affine & $1.11\times10^{-6}$ & 68.61 & 0.122 & 0.051 & 12.6 \\
500 & Fitted shape + affine & $\mathbf{1.09\times10^{-6}}$ & 75.00 & 0.125 & 0.051 & 12.0 \\
\addlinespace[3pt]
1000 & Fixed shape & $1.19\times10^{-6}$ & 27.92 & 0.262 & 0.099 & 11.6 \\
1000 & Fitted shape & $8.54\times10^{-7}$ & 25.00 & 0.231 & 0.098 & 11.9 \\
1000 & Fixed shape + affine & $\mathbf{8.31\times10^{-7}}$ & 16.61 & 0.164 & 0.094 & 12.6 \\
1000 & Fitted shape + affine & $\mathbf{8.31\times10^{-7}}$ & 16.61 & 0.164 & 0.094 & 10.1 \\
\addlinespace[3pt]
5000 & Fixed shape & $3.53\times10^{-7}$ & 10.00 & 2.710 & 0.539 & 23.8 \\
5000 & Fitted shape & $\mathbf{3.30\times10^{-7}}$ & 10.00 & 2.710 & 0.539 & 25.6 \\
5000 & Fixed shape + affine & $3.53\times10^{-7}$ & 10.00 & 2.710 & 0.540 & 24.3 \\
5000 & Fitted shape + affine & $3.31\times10^{-7}$ & 10.00 & 2.710 & 0.540 & 24.4 \\
\addlinespace[3pt]
10000 & Fixed shape & $3.19\times10^{-7}$ & 10.00 & 11.865 & 1.140 & 76.2 \\
10000 & Fitted shape & $\mathbf{2.95\times10^{-7}}$ & 10.00 & 11.865 & 1.140 & 75.3 \\
10000 & Fixed shape + affine & $3.19\times10^{-7}$ & 10.00 & 11.865 & 1.141 & 76.2 \\
10000 & Fitted shape + affine & $2.96\times10^{-7}$ & 10.00 & 11.865 & 1.141 & 74.2 \\
\bottomrule
\end{tabular}
\end{table}

\begin{table}[pos=!htbp]
\centering\small
\setlength{\tabcolsep}{3.5pt}
\caption{Taylor--Green vortices (space): all models. Test error is the mean relative velocity error of the predictions for the 100 test inputs. Density is $d_K=\mathrm{nnz}(K_{UU})/N^2$ in percent. Factor is the number of entries of the cuDSS Cholesky factor in millions. GB is the accounted storage and ms the median warm time for all 100 fields. Bold: smallest test error at each $N$.}
\label{tab:gpu-all-taylor-green}
\begin{tabular}{rlrrrrr}
\toprule
$N$ & Operator model & Test error & Density (\%) & Factor ($10^6$) & GB & ms \\
\midrule
500 & Fixed shape & $9.41\times10^{-7}$ & 46.17 & 0.094 & 0.028 & 6.5 \\
500 & Fitted shape & $\mathbf{7.72\times10^{-7}}$ & 50.00 & 0.097 & 0.029 & 7.7 \\
500 & Fixed shape + affine & $8.76\times10^{-7}$ & 40.07 & 0.083 & 0.028 & 7.7 \\
500 & Fitted shape + affine & $8.83\times10^{-7}$ & 50.00 & 0.097 & 0.029 & 6.7 \\
\addlinespace[3pt]
1000 & Fixed shape & $3.31\times10^{-7}$ & 51.76 & 0.422 & 0.068 & 7.6 \\
1000 & Fitted shape & $\mathbf{7.86\times10^{-9}}$ & 50.00 & 0.414 & 0.067 & 8.6 \\
1000 & Fixed shape + affine & $3.63\times10^{-8}$ & 32.16 & 0.288 & 0.060 & 8.0 \\
1000 & Fitted shape + affine & $1.03\times10^{-8}$ & 50.00 & 0.414 & 0.067 & 8.2 \\
\addlinespace[3pt]
5000 & Fixed shape & $2.45\times10^{-6}$ & 75.00 & 22.923 & 1.056 & 38.9 \\
5000 & Fitted shape & $2.51\times10^{-6}$ & 90.00 & 24.889 & 1.199 & 41.5 \\
5000 & Fixed shape + affine & $2.67\times10^{-10}$ & 10.00 & 2.644 & 0.339 & 14.2 \\
5000 & Fitted shape + affine & $\mathbf{6.00\times10^{-11}}$ & 10.00 & 2.644 & 0.339 & 14.7 \\
\addlinespace[3pt]
10000 & Fixed shape & $2.82\times10^{-6}$ & 75.00 & 72.677 & 3.606 & 117.0 \\
10000 & Fitted shape & $2.11\times10^{-6}$ & 90.00 & 99.800 & 4.334 & 119.0 \\
10000 & Fixed shape + affine & $2.42\times10^{-10}$ & 10.00 & 11.728 & 0.899 & 52.3 \\
10000 & Fitted shape + affine & $\mathbf{5.22\times10^{-11}}$ & 10.00 & 11.728 & 0.899 & 54.6 \\
\bottomrule
\end{tabular}
\end{table}

\begin{table}[pos=!htbp]
\centering\small
\setlength{\tabcolsep}{3.5pt}
\caption{Taylor--Green vortices (spacetime): all models. Test error is the mean relative velocity error of the predictions for the 100 test inputs. Density is $d_K=\mathrm{nnz}(K_{UU})/N^2$ in percent. Factor is the number of entries of the cuDSS Cholesky factor in millions. GB is the accounted storage and ms the median warm time for all 100 fields. Bold: smallest test error at each $N$.}
\label{tab:gpu-all-taylor-green-time}
\begin{tabular}{rlrrrrr}
\toprule
$N$ & Operator model & Test error & Density (\%) & Factor ($10^6$) & GB & ms \\
\midrule
500 & Fixed shape & $8.01\times10^{-7}$ & 46.23 & 0.093 & 0.052 & 6.4 \\
500 & Fitted shape & $\mathbf{6.58\times10^{-7}}$ & 50.00 & 0.097 & 0.053 & 6.3 \\
500 & Fixed shape + affine & $7.38\times10^{-7}$ & 39.51 & 0.082 & 0.052 & 8.1 \\
500 & Fitted shape + affine & $7.65\times10^{-7}$ & 50.00 & 0.097 & 0.053 & 6.8 \\
\addlinespace[3pt]
1000 & Fixed shape & $2.71\times10^{-7}$ & 51.58 & 0.419 & 0.116 & 7.9 \\
1000 & Fitted shape & $9.30\times10^{-9}$ & 50.00 & 0.414 & 0.115 & 9.8 \\
1000 & Fixed shape + affine & $3.34\times10^{-8}$ & 32.27 & 0.284 & 0.108 & 11.0 \\
1000 & Fitted shape + affine & $\mathbf{8.61\times10^{-9}}$ & 50.00 & 0.414 & 0.115 & 7.7 \\
\addlinespace[3pt]
5000 & Fixed shape & $2.14\times10^{-6}$ & 75.00 & 22.923 & 1.136 & 41.9 \\
5000 & Fitted shape & $2.19\times10^{-6}$ & 90.00 & 24.889 & 1.279 & 40.2 \\
5000 & Fixed shape + affine & $2.41\times10^{-10}$ & 10.00 & 2.644 & 0.419 & 16.8 \\
5000 & Fitted shape + affine & $\mathbf{5.40\times10^{-11}}$ & 10.00 & 2.644 & 0.419 & 16.6 \\
\addlinespace[3pt]
10000 & Fixed shape & $2.46\times10^{-6}$ & 75.00 & 72.677 & 3.766 & 119.8 \\
10000 & Fitted shape & $1.86\times10^{-6}$ & 90.00 & 99.800 & 4.494 & 120.1 \\
10000 & Fixed shape + affine & $2.18\times10^{-10}$ & 10.00 & 11.728 & 1.059 & 59.0 \\
10000 & Fitted shape + affine & $\mathbf{4.70\times10^{-11}}$ & 10.00 & 11.728 & 1.059 & 61.7 \\
\bottomrule
\end{tabular}
\end{table}

\begin{table}[pos=!htbp]
\centering\small
\setlength{\tabcolsep}{3.5pt}
\caption{Backward-facing step: all models. Test error is the mean relative velocity error of the predictions for the 100 test inputs. Density is $d_K=\mathrm{nnz}(K_{UU})/N^2$ in percent. Factor is the number of entries of the cuDSS Cholesky factor in millions. GB is the accounted storage and ms the median warm time for all 100 fields. Bold: smallest test error at each $N$.}
\label{tab:gpu-all-backward-facing-step}
\begin{tabular}{rlrrrrr}
\toprule
$N$ & Operator model & Test error & Density (\%) & Factor ($10^6$) & GB & ms \\
\midrule
500 & Fixed shape & $2.43\times10^{-7}$ & 36.41 & 0.073 & 0.024 & 8.8 \\
500 & Fitted shape & $2.43\times10^{-7}$ & 36.41 & 0.073 & 0.024 & 8.6 \\
500 & Fixed shape + affine & $\mathbf{9.46\times10^{-8}}$ & 34.94 & 0.068 & 0.024 & 9.6 \\
500 & Fitted shape + affine & $\mathbf{9.46\times10^{-8}}$ & 34.94 & 0.068 & 0.024 & 8.6 \\
\addlinespace[3pt]
1000 & Fixed shape & $9.31\times10^{-8}$ & 49.59 & 0.350 & 0.059 & 11.6 \\
1000 & Fitted shape & $1.14\times10^{-7}$ & 25.00 & 0.236 & 0.050 & 8.4 \\
1000 & Fixed shape + affine & $\mathbf{6.68\times10^{-8}}$ & 47.15 & 0.342 & 0.058 & 9.9 \\
1000 & Fitted shape + affine & $7.44\times10^{-8}$ & 25.00 & 0.236 & 0.050 & 8.4 \\
\addlinespace[3pt]
5000 & Fixed shape & $6.77\times10^{-8}$ & 29.08 & 7.711 & 0.504 & 25.8 \\
5000 & Fitted shape & $6.77\times10^{-8}$ & 29.08 & 7.711 & 0.504 & 26.1 \\
5000 & Fixed shape + affine & $\mathbf{5.98\times10^{-8}}$ & 28.90 & 7.785 & 0.503 & 26.5 \\
5000 & Fitted shape + affine & $\mathbf{5.98\times10^{-8}}$ & 28.90 & 7.785 & 0.503 & 26.2 \\
\addlinespace[3pt]
10000 & Fixed shape & $6.33\times10^{-8}$ & 32.98 & 40.506 & 1.679 & 72.4 \\
10000 & Fitted shape & $6.33\times10^{-8}$ & 32.98 & 40.506 & 1.679 & 68.8 \\
10000 & Fixed shape + affine & $\mathbf{5.86\times10^{-8}}$ & 32.88 & 38.228 & 1.657 & 66.3 \\
10000 & Fitted shape + affine & $\mathbf{5.86\times10^{-8}}$ & 32.88 & 38.228 & 1.657 & 65.3 \\
\bottomrule
\end{tabular}
\end{table}

\begin{table}[pos=!htbp]
\centering\small
\setlength{\tabcolsep}{3.5pt}
\caption{Buoyancy-driven cavity: all models. Test error is the mean relative velocity error of the predictions for the 100 test inputs. Density is $d_K=\mathrm{nnz}(K_{UU})/N^2$ in percent. Factor is the number of entries of the cuDSS Cholesky factor in millions. GB is the accounted storage and ms the median warm time for all 100 fields. Bold: smallest test error at each $N$.}
\label{tab:gpu-all-buoyancy-cavity-flow}
\begin{tabular}{rlrrrrr}
\toprule
$N$ & Operator model & Test error & Density (\%) & Factor ($10^6$) & GB & ms \\
\midrule
500 & Fixed shape & $\mathbf{5.63\times10^{-5}}$ & 64.68 & 0.118 & 0.154 & 30.7 \\
500 & Fitted shape & $\mathbf{5.63\times10^{-5}}$ & 64.68 & 0.118 & 0.154 & 29.6 \\
500 & Fixed shape + affine & $8.18\times10^{-5}$ & 10.00 & 0.024 & 0.148 & 29.1 \\
500 & Fitted shape + affine & $7.39\times10^{-5}$ & 10.00 & 0.024 & 0.148 & 28.6 \\
\addlinespace[3pt]
1000 & Fixed shape & $\mathbf{3.03\times10^{-5}}$ & 17.33 & 0.170 & 0.298 & 32.5 \\
1000 & Fitted shape & $\mathbf{3.03\times10^{-5}}$ & 17.33 & 0.170 & 0.298 & 32.2 \\
1000 & Fixed shape + affine & $3.04\times10^{-5}$ & 17.27 & 0.166 & 0.298 & 36.4 \\
1000 & Fitted shape + affine & $3.04\times10^{-5}$ & 17.27 & 0.166 & 0.298 & 32.6 \\
\addlinespace[3pt]
5000 & Fixed shape & $6.65\times10^{-6}$ & 10.00 & 2.838 & 1.143 & 111.5 \\
5000 & Fitted shape & $\mathbf{5.02\times10^{-6}}$ & 10.00 & 2.838 & 1.143 & 104.1 \\
5000 & Fixed shape + affine & $6.65\times10^{-6}$ & 10.00 & 2.838 & 1.144 & 107.7 \\
5000 & Fitted shape + affine & $5.35\times10^{-6}$ & 10.00 & 2.838 & 1.144 & 109.9 \\
\addlinespace[3pt]
10000 & Fixed shape & $5.38\times10^{-6}$ & 10.00 & 11.947 & 2.504 & 189.4 \\
10000 & Fitted shape & $\mathbf{4.47\times10^{-6}}$ & 10.00 & 11.947 & 2.504 & 184.9 \\
10000 & Fixed shape + affine & $5.38\times10^{-6}$ & 10.00 & 11.947 & 2.505 & 177.0 \\
10000 & Fitted shape + affine & $4.62\times10^{-6}$ & 10.00 & 11.947 & 2.505 & 173.4 \\
\bottomrule
\end{tabular}
\end{table}

%% file: sections/gpu_offline_n10000.tex
\begin{table}[pos=tbp]
\centering\small
\setlength{\tabcolsep}{4pt}
\caption{Offline and initialization costs at $N=10000$. Selection includes kernel fitting, where applicable, and validation. Rebuild assembles and factors the selected operator. Initialization includes two synchronized 100-field prediction passes, compilation, and CUDA graph preparation. The four models share the setup of the output reconstruction and may reuse compiled code. The dash marks the selection time of the fixed-shape backward-facing-step model, which contains an external pause. Its raw time is in the released data. Bold: smallest time in each column for each data set.}
\label{tab:gpu-offline-n10000}
\begin{tabular}{llrrr}
\toprule
Data set & Operator model & Selection (min) & Rebuild (s) & Initialization (s) \\
\midrule
Forced isotropic turbulence & Fixed shape & \textbf{3.83} & 43.27 & 3.71 \\
 & Fitted shape & 26.32 & \textbf{41.63} & 3.28 \\
 & Fixed shape + affine & 6.45 & 43.70 & 4.15 \\
 & Fitted shape + affine & 27.01 & 42.77 & \textbf{3.05} \\
\addlinespace[3pt]
Turbulent species transport & Fixed shape & \textbf{3.24} & \textbf{1.48} & 4.17 \\
 & Fitted shape & 25.89 & 4.64 & \textbf{3.70} \\
 & Fixed shape + affine & 5.79 & 1.76 & 6.05 \\
 & Fitted shape + affine & 24.01 & 4.93 & 3.79 \\
\addlinespace[3pt]
Merging vortices & Fixed shape & \textbf{4.81} & 0.91 & 3.35 \\
 & Fitted shape & 15.26 & 2.92 & \textbf{3.20} \\
 & Fixed shape + affine & 4.94 & 1.12 & 4.94 \\
 & Fitted shape + affine & 11.70 & \textbf{0.81} & 3.26 \\
\addlinespace[3pt]
Lid-driven cavity & Fixed shape & \textbf{2.12} & 0.86 & 2.37 \\
 & Fitted shape & 12.97 & 0.87 & \textbf{2.20} \\
 & Fixed shape + affine & 3.06 & 0.91 & 3.54 \\
 & Fitted shape + affine & 10.90 & \textbf{0.85} & 2.26 \\
\addlinespace[3pt]
Cylinder flow (no shedding) & Fixed shape & \textbf{2.32} & 0.87 & 2.58 \\
 & Fitted shape & 12.77 & \textbf{0.86} & 2.54 \\
 & Fixed shape + affine & 4.28 & 0.92 & 3.77 \\
 & Fitted shape + affine & 14.53 & 0.90 & \textbf{2.50} \\
\addlinespace[3pt]
Cylinder flow (vortex shedding) & Fixed shape & \textbf{2.12} & \textbf{0.85} & 2.47 \\
 & Fitted shape & 13.58 & \textbf{0.85} & 2.61 \\
 & Fixed shape + affine & 3.07 & 0.91 & 3.43 \\
 & Fitted shape + affine & 12.86 & 0.88 & \textbf{2.37} \\
\addlinespace[3pt]
Taylor--Green vortices (space) & Fixed shape & \textbf{3.41} & 3.18 & 3.61 \\
 & Fitted shape & 12.80 & 5.54 & \textbf{3.43} \\
 & Fixed shape + affine & 5.32 & 1.25 & 5.47 \\
 & Fitted shape + affine & 15.30 & \textbf{0.68} & 3.54 \\
\addlinespace[3pt]
Taylor--Green vortices (spacetime) & Fixed shape & \textbf{3.49} & 3.30 & 3.87 \\
 & Fitted shape & 16.10 & 3.60 & \textbf{3.60} \\
 & Fixed shape + affine & 4.82 & 0.88 & 5.22 \\
 & Fitted shape + affine & 15.47 & \textbf{0.69} & 3.64 \\
\addlinespace[3pt]
Backward-facing step & Fixed shape & --- & 0.95 & 1.48 \\
 & Fitted shape & 14.23 & 0.99 & \textbf{1.31} \\
 & Fixed shape + affine & \textbf{3.25} & 3.27 & 2.53 \\
 & Fitted shape + affine & 13.42 & \textbf{0.94} & 1.36 \\
\addlinespace[3pt]
Buoyancy-driven cavity & Fixed shape & \textbf{2.11} & 1.65 & 2.72 \\
 & Fitted shape & 17.67 & \textbf{1.48} & \textbf{2.53} \\
 & Fixed shape + affine & 3.16 & 1.64 & 4.06 \\
 & Fitted shape + affine & 17.00 & 1.51 & 2.63 \\
\addlinespace[3pt]
\bottomrule
\end{tabular}
\end{table}